\documentclass[11pt]{article}

\usepackage{natbib}
\usepackage{amsmath}
\usepackage{amssymb}
\usepackage{amsthm}
\usepackage{graphicx}
\usepackage{url}
\usepackage{fullpage}
\usepackage[colorlinks,
            linkcolor=blue,
            citecolor=blue,
            urlcolor=magenta,
            linktocpage,
            plainpages=false]{hyperref}
\usepackage{float}
\usepackage{dsfont}
\usepackage{algorithm}
\usepackage{algpseudocode}
\newenvironment{keywords}
{\bgroup\leftskip 27pt\rightskip 27pt \small\noindent{\bfseries
Keywords:} \ignorespaces}%
{\par\egroup\vskip 0.25ex}

\newcommand{\bs}[1]{\boldsymbol{\mathbf{#1}}}
\newcommand{\bbE}{\mathbb{E}}
\newcommand{\bbK}{\mathbb{K}}
\newcommand{\bbN}{\mathbb{N}}
\newcommand{\bbR}{\mathbb{R}}
\newcommand{\bbS}{\mathbb{S}}
\newcommand{\rmd}{\mathrm{d}}

\newcommand{\fracpartial}[2]{\frac{\partial #1}{\partial #2}}
\newcommand{\diag}{\mathop{\mathrm{diag}}}
\newcommand{\tr}{\mathop{\mathrm{tr}}}
\newcommand{\cov}{\mathop{\mathrm{cov}}}
\newcommand{\argmin}{\mathop{\mathrm{argmin}}}

\newcommand{\opt}{\mathrm{opt}}
\newcommand{\sgd}{\mathrm{SGD}}

\newcommand{\defeq}{\stackrel{\text{def}}{=}}

\newcommand{\cB}{\mathcal{B}}

\newcommand{\cI}{\mathcal{I}}

\newcommand{\cO}{\mathcal{O}}
\newcommand{\cQ}{\mathcal{Q}}
\newcommand{\cT}{\mathcal{T}}
\newcommand{\cX}{\mathcal{X}}

\newcommand*\samethanks[1][\value{footnote}]{\footnotemark[#1]}

\theoremstyle{definition}
\newtheorem{definition}{Definition}

\theoremstyle{plain}
\newtheorem{theorem}{Theorem}

\theoremstyle{plain}
\newtheorem{assumption}{Assumption}

\theoremstyle{plain}
\newtheorem{lemma}{Lemma}

\theoremstyle{remark}
\newtheorem*{remark}{Remark}

\theoremstyle{plain}
\newtheorem{example}{Example}

\title{Optimal Algorithms in Linear Regression under Covariate Shift:\\ On the Importance of Precondition}
\author{Yuanshi Liu$^{\dagger}$\thanks{Equal Contribution.}  \quad Haihan Zhang$^{\dagger}$\samethanks \quad Qian Chen$^{\ddagger}$ \quad Cong Fang$^{\dagger}$ \\
\\
        \small $^{\dagger}$Peking University
        \\\\
        \small $^{\ddagger}$Huazhong University of Science and Technology
}
\begin{document}
\date{}
\maketitle

\begin{abstract}%
A common pursuit in modern statistical learning is to attain satisfactory generalization out of the source data distribution (OOD). 
%Recent years have witnessed remarkable performances of the model trained by vanilla algorithms, such as stochastic gradient descent (SGD) despite distribution shifts. However, theoretically, the realization of OOD  generalization 
In theory, the challenge remains unsolved even under the canonical setting of covariate shift for the linear model. This paper studies the foundational (high-dimensional) linear regression where the ground truth variables are confined to an ellipse-shape constraint and addresses two fundamental questions in this regime: (i) given the target covariate matrix, what is the min-max \emph{optimal} algorithm under covariate shift? (ii) for what kinds of target classes, the commonly-used SGD-type algorithms achieve optimality? Our analysis starts with establishing a tight lower generalization bound via a Bayesian Cramer-Rao inequality. For (i), we prove that  the optimal estimator can be simply a certain linear transformation of the best estimator for the source distribution. Given the source and target matrices, we show that the transformation can be efficiently computed via a convex program. The min-max optimal analysis for SGD leverages the idea that we recognize both the accumulated  updates of the applied algorithms and
 the ideal transformation as preconditions on the learning variables. We provide sufficient conditions when SGD with its acceleration variants attain optimality. 
 %As a byproduct, we exemplify when momentum provably benefits  CoS generalization and when the learning curves exhibit emergent phenomena that decrease slowly at first and drop rapidly after hitting a critical sample size.
\end{abstract}


\section{Introduction}

Covariate shift is a fundamental challenge in machine learning, particularly relevant in the modern paradigm where data collection and application scenarios are separated \citep{shimodaira2000improving,sugiyama2012machine}. This underlying mismatch causes problems including out-of-distribution (OOD) generalization \citep{NEURIPS2022_43119db5,9710159}, posing significant challenges in the performance of large language models (LLMs) and fundamental understanding of the recent modern artificial intelligence progress.
Despite its ubiquity, the theoretical understanding of the covariate shift problem remains underexplored. The current research seeks a theoretical study of covariate shift from generalization by exploring the interplay between two key aspects: (1) the statistical bottleneck and (2) the potential implicit effects of prevalent computational methods.

Towards the statistical bottleneck, we ask the question: 
\begin{center}
 \textit{What is the information-theoretic optimal estimator under the covariate shift problem, and what the associated minimax risk?}
\end{center}
\noindent This question is intrinsic to statistical learning yet remains blurred in the covariate shift problem.

Previous analyses have provided valuable insights into optimality. However, gaps persist: existing minimax results are either limited to asymptotic settings (e.g., the pioneering work of \citet{shimodaira2000improving}, which examines the asymptotic optimality of weighted kernel regression) or focus on optimality under a coarse characterization (see e.g., the seminal work \citet{kpotufe2021marginal}, which studies optimality under a given transfer-exponent $\gamma$ that quantifies the local singularity of the target distribution relative to the source).

We focus on the problem of high-dimensional linear regression, which corresponds to the overparameterized paradigm prevalent in modern large-scale machine learning.
As one of the most fundamental statistical models, high-dimensional linear regression is experiencing a rise due to its interpretability of learning phenomena, such as benign overfitting \citep{bartlett2020benign}, and has garnered considerable attention in the covariate shift literature \citep{lei2021near,zhang2022class}.

Concretely, we establish the first information-theoretic minimax optimality result for the high-dimension linear regression problem. Our approach begins with the derivation of an information-theoretic minimax lower bound, which is established under two constraints: (1) the canonical $\|\cdot\|_2$ isotropic ball constraint, and (2) the anisotropic ellipsoid constraint prevalent in the reproducing kernel Hilbert space (RKHS) literature \citep{caponnetto2007optimal}. Technically, we overcome the challenge of establishing the tight minimax lower bound through two novel steps: (1) a geometric framework employing a family of priors supported on all ellipsoid-inscribed parallelohedra; (2) a Bayesian matrix lower bound for estimator covariance under a given prior, serving as a multivariate generalized version of Bayesian Cramer-Rao inequality \citep{bj/1186078362}. Our minimax lower bound emerges as the supremum of Bayesian risk over all priors.

Conversely, the upper bound is achieved by an estimator with a special preconditioning. Its realization incorporates an optimization preprocessing step over the linear preconditioner class to identify the one with minimal risk. 
Surprisingly, our result demonstrates the equivalence of two extrema through duality analysis, thereby establishing the minimax rate. 
There are two notable points about this minimax optimality: first, our result is instance-optimal – it establishes the exact minimax rate for any given source-target pair; second, the preconditioning class is a broad class spanning classical methods (least squares, ridge regression, and even gradient descent implicitly), enabling a unified interpretation of diverse methods through one upper bound.

Beyond characterizing optimal algorithms under covariate shifts, we examine how commonly employed modern machine learning optimization methods shape generalization in the presence of such shifts?
The second question we seek to address is: 
\begin{center}
 \textit{How do the stochastic gradient based methods—the backbone of large-scale training—perform in the covariate shift problem, without preknowing the knowledge of the source domain? } 
\end{center}
\noindent We answer the question by quantifying generalization efficiency and identifying its fundamental limits. We compare the risk with prior optimality criteria, while the risk provides insights into real-world phenomena. 
The analysis is based on a framework of accelerated stochastic gradient descent (ASGD) \citep{jain2018accelerating} that utilizes geometrically decaying step sizes \citep{ge2019step}, which is close to practical implementation. 
This framework of methods can be regarded as special cases of the previously discussed linear preconditioner. We establish an instance-based generalization bound for this framework. Leveraging the previous result, we propose the optimal conditions for ASGD. The proposed optimality criteria include diverse canonical assumptions, including information-divergence constraints and bounded likelihood-ratio conditions as special cases, demonstrating that ASGD can achieve near-optimal performance under high tolerances of shift magnitude and spatial rotation intensity. As a byproduct, we also exemplify when momentum provably benefits generalization in the covariate shift problem, and when the learning curves exhibit emergent phenomena that decrease slowly at first and drop rapidly after hitting a critical sample size.

\section{Related Work}

\textbf{Covariate Shift.} 
There is a vast theoretical literature on the covariate shift problem (See e.g. \citet{ben2010theory, germain2013pac, cortes2010learning, cortes2019adaptation} and the review in \citet{sugiyama2012machine,kouw2019review}). 
In the following, we primarily review the literature related to the optimality analysis of the covariate shift problem.
Pioneering work \citep{shimodaira2000improving} studies weighted maximum likelihood estimation in the covariate shift problem under the misspecified setting. Their optimality result is only under the asymptotic setting. 
More recently, a thread of research considers the optimality of covariate shift problem in the kernel regression setup \citep{ma2023optimally, pathak2022new}. Their results analyze the truncated-reweighted kernel ridge regression. However, their corresponding minimax rate result relies on the likelihood ratio bound and \citet{ge2023maximum} analyze the maximum likelihood optimality under the well-specified and low-dimensional setting. In their setup, vanilla maximum likelihood attains the optimal rate.
Another celebrated work by \citet{kpotufe2021marginal} delves into the nonparametric classification problem and propose a novel characterization of transfer exponent, a local regularity measurement, and show a nearest neighbor based method attain the minimax rate. One advantage of our result established is that we can provide an instance-based optimality in a refined geometric characterization.
When confining the scope to the high-dimensional linear regression, several studies have investigated the minimax rates under general distribution shifts; however, the entanglement of different types of distribution shift leads to suboptimal or inapplicable results when specifically applied to the covariate shift problem \citep{zhang2022class, mousavi2020minimax}. Another attempt of minimax optimality from \citet{lei2021near} derives the essentially same estimator as ours, but their optimality guarantees only compromise to the linear class, which is relatively obvious. General information-theoretic optimality has remained a notable gap and resolving this constitutes one of the major technical advancements central to our contribution.

\noindent\textbf{Stochastic Gradient Methods in Linear Regression.}
Recent advances in stochastic gradient descent (SGD) for linear regression have yielded extensive theoretical analyses that bridge gaps between empirical practice and formal optimization frameworks. These studies validate practical techniques—such as the exponential decay step size schedule, widely adopted in implementations—by establishing their minimax-optimal risk guarantees \citep{ge2019step, pan2021eigencurve}, insights previously unattainable within conventional black-box optimization paradigms. Another notable example involves the accelerated gradient method which retains the efficacy even in regimes with substantial gradient noise given certain noise scheme, as demonstrated by \citet{jain2018accelerating} and subsequent studies \citep{varre2022accelerated}.
In the context of high-dimensional linear regression, a line of research, starting from the concept of benign overfitting under overparametrized model \citet{bartlett2020benign}, studies the provable generalization power of stochastic gradient based optimization \citep{zou2021benign, wu2022last, lirisk,Zhang2024OptimalityAcceleratedSGD}. 
When extending this analysis to scenarios involving covariate shift, recent work by \citet{wu2022power} examines both the capabilities and limitations of stochastic gradient methods under such distributional mismatches.
Our results provide the upper bound for SGD with a geometrically decaying stepsize schedule and Nesterov momentum acceleration, which requires a more complex analysis of the iterative process. 
Another related research line concerns the nonparametric regression problem. \citet{Dieuleveut2016NonparametricStochasticApproximation} study the stochastic gradient method in the setup of reproducing kernel hilbert space, and \citet{cai2006prediction} consider the functional linear regression.
A parallel line centers on SGD on nonparametric regression, exploring the optimality conditions of SGD in the context of reproducing kernel \citep{Dieuleveut2016NonparametricStochasticApproximation}.

\section{Problem Formulation and Minimax Lower Bound}
\textbf{Notations.} We denote the spectral norm, Frobenius norm and nuclear norm of a matrix $\bs A$ by $\left\Vert\bs A\right\Vert$, $\left\Vert\bs A\right\Vert_F$ and $\left\Vert\bs A\right\Vert_*$, respectively. Define the elliptical norm of vector $\bs x$ under positive definite matrix $\bs M$ as $\|\bs x\|_{\bs M}^2 = \bs x^{\top}\bs M\bs x$.
We use $\bs O$ to denote the matrix with all entries equal to zero.
For positive integer $n$, let $[n]=\left\{1,2,\ldots,n\right\}$. The diagonal matrix with sequence $\left\{a_i\right\}_{i=1}^d$ as its diagonal entries is denoted by $\diag\left\{a_i\right\}_{i=1}^d$.

\subsection{Problem Setup}
\subsubsection{Linear Regression under Covariate Shifts}

The regression problem using covariate variable $\bs x\in\mathbb{R}^d$ to predict the response variable $y\in\mathbb{R}$. In the covariate shift problem, there are two distinct data domains on covariate space and the response space: a source domain $\mathcal{S}$ and a target domain $\mathcal{T}$. Let $P_{\bs x \times y}$ denote the joint distribution of $(\mathbf{x},y)$ over domain $\mathcal{S}$  and $Q_{\bs x \times y}$ denote the joint distribution of $(\mathbf{x},y)$ over domain $\mathcal{T}$. 

We assume access to $n$ i.i.d. samples $\left \{ \left ( \mathbf{x}_i,y_i  \right )  \right \} _{i=1}^n$ drawn from $P_{\bs x \times y}$, while the predictor’s performance is evaluated under the generalization risk on the target distribution $Q_{\bs x \times y}$. 
Covariate shift refers to the problems where the covariate marginal distribution $P_{\mathbf{x} }$ in source domain differs from the covariate marginal distribution $Q_{\mathbf{x}}$ in the target domain, while the conditional distribution remains unchanged in both domains.

Denote the covariance of the source and target distribution as
$\bs S=\mathbb{E} _{P_{\mathbf{x} }}\left [ \mathbf{x}  \mathbf{x}^{\top} \right ]$ and $\bs T=\mathbb{E} _{Q_{\mathbf{x} }}\left [ \mathbf{x}  \mathbf{x}^{\top} \right ]$. The eigenvalue decomposition of $\bs{S}$ is given by 
\begin{equation}
  \bs{S}=  \mathbf{U} \mathrm{diag} \left \{ \lambda _1,\dots,\lambda _d \right \} \mathbf{U}^{\top },
\end{equation}
where $\lambda_1\ge \cdots\ge \lambda_d$ are the eigenvalues 
of $\bs{S}$ in non-increasing order. For simplicity, we assume that $\mathbf{U}$ is the standard orthonormal basis in $\mathbb{R}^d$.

In the high-dimensional linear regression problem, we aim to find a regressor $\bs w$ based on the observed data, such that $\bs w^{\top} \bs x$ provides predicts for $y$ on fresh data $(\bs x, y)$.
For any regressor $\mathbf{w}\in \mathbb{R}^d$, the source risk $\mathcal{E} _{\mathcal{S}} \left ( \mathbf{w}  \right )$ and target risk $\mathcal{E} _{\mathcal{T}} \left ( \mathbf{w}  \right )$ are defined as:
\begin{equation}
  \mathcal{E} _{\mathcal{S}} \left ( \mathbf{w}  \right ) =\frac{1}{2} \mathbb{E} _{P_{\bs x \times y}}\left ( y-\left \langle \mathbf{w},\mathbf{x }   \right \rangle  \right )^2 ,\  \   \mathcal{E} _{\mathcal{T}} \left ( \mathbf{w}  \right ) =\frac{1}{2} \mathbb{E} _{Q_{\bs x \times y}}\left ( y-\left \langle \mathbf{w},\mathbf{x }   \right \rangle  \right )^2. 
\end{equation}
We make the following assumption regarding the source and the target risk.
\begin{assumption}\label{asp:minimizer}
    The source risk $\mathcal{E} _{\mathcal{S}}$ and the target risk $\mathcal{E} _{\mathcal{T}}$ admit one same minimizer $\bs w^*$.
\end{assumption}
One particular scenario where Assumption~\ref{asp:minimizer} holds is that the linear model is well-specified, meaning $\mathbb{E}[y|\bs x]$ is a linear function of $\bs x$. And well-specification is a common assumption when considering high-dimensional linear regression problem \citep{bartlett2020benign}. This assumption is related to the powerful expressive capacity of the high-dimensional overparametrized model, which can include a regressor capable of interpolating the true conditional expectation. 

In covariate shift problems, the performance of regressor $\mathbf{w}$ is evaluated by the excess risk on the target distribution 
 $Q_{\bs x \times y}$:
\begin{equation}
     \mathcal{R} _{\mathcal{T}} \left ( \mathbf{w}  \right ) =\frac{1}{2} \left( \mathcal{E} _{\mathcal{T}} \left ( \mathbf{w}  \right )-\min_{\mathbf{w} } \mathcal{E} _{\mathcal{T}} \left ( \mathbf{w}  \right )\right ) =\frac{1}{2}\left \| \mathbf{w}-\mathbf{w}^*   \right \| _{\bs T}^2.
\end{equation}

Our analyses consider the minimax optimality under the candidate set $W$ for $\mathbf{w}^*$ being an ellipsoid, specifically
\begin{equation}\label{eq: constraint}
    W=\left\{\bs w^*\in\bbR^d:\left\Vert\bs w^*\right\Vert_{\bs M}^2\leq 1\right\},
\end{equation}
where $\bs M\in \mathbb{R}^{d\times d}$ is a positive definite matrix. This ground truth candidate set $W$ encompasses the common unit ball space $\left\{\bs w^*:\left \| \mathbf{w}^*  \right \|_2^2 \le1\right\}$ as special case, and is analogous to the interpolation space $[\mathcal{H}]^s$ in the context of the reproducing kernel Hilbert space (RKHS) regression \citep{caponnetto2007optimal}.

\subsubsection{Assumptions}

We then introduce several commonly used assumptions required in our analysis.

\begin{assumption}[Finite Initial Excess Risk]~\label{ass-finite-risk}
    We assume that there exists $c>0$ such that for any $\mathbf{w}^*\in W$, $\left \| \mathbf{w}^* \right \|_{\bs S}^2\le c$.
\end{assumption}
\begin{remark}
    This assumptionis mild and  implies that for any ground truth parameter $\mathbf{w}^* \in W$, the excess risk of $\mathbf{w}_0 = \mathbf{0}$ under the source distribution $P_{\bs x \times y}$ is finite. In other words, the problem is well defined on the source domain. In particular, when $\mathbb{R}^d$ is treated as a special case of RKHS, this assumption corresponds to the commonly used source condition~\citep{caponnetto2007optimal,Dieuleveut2016NonparametricStochasticApproximation,fischer2020sobolev}.  Furthermore, this assumption leads to the bound $ \left \| \bs M^{-1/2}\bs S\bs M^{-1/2} \right \| \le c$.
\end{remark}

\begin{assumption}[Bound of Fourth Moment in the Source Domain]\label{assumption:fourth-moment}
    There exists a constant $\psi\geq 1$, such that for every PSD matrix A, we have
    \begin{equation}
        \bbE_{P_{\bs x}}\left[\bs x\bs x^\top\bs A\bs x\bs x^\top\right]\leq\psi\tr\left(\bs S\bs A\right)\bs S.
    \end{equation}
\end{assumption}
\begin{remark}
Assumption~\ref{assumption:fourth-moment} is assumed to hold for data distributions with bounded kurtosis for the projection of the covariate $\mathbf{x}$ onto any vector $z \in \mathbb{R}^d$. Specifically, there exists a constant $c > 0$ such that for any $z \in \mathbb{R}^d$, the following inequality holds: $z\in\mathbb R^d$, $\mathbb{E}_{P_{\bs x}} \left \langle \mathbf{z} ,\mathbf{x}  \right \rangle ^4\le c\left \langle \mathbf{z},\bs S \mathbf{z}  \right \rangle^2 $. 
For instance, if $\bs{S}^{-\frac{1}{2}}\mathbf{x}$ follows a Gaussian distribution, Assumption~\ref{assumption:fourth-moment} holds with $\psi = 3$. This assumption is widely adopted in works related to linear regression~\citep{bartlett2020benign,tsigler2023benign,zou2021benign,wu2022last,wu2022power}. Indeed, we impose this condition to handle the case where $\left \| \mathbf{w}^* \right \|  _2=\infty $. If $\left \| \mathbf{w}^* \right \|  _2$ is finite, all of our conclusions hold under a weaker assumption $\bbE_{\mathbb{P}_{\mathbf{x} }^\mathcal{S}}\left[\left \| \bs x \right \|^2 \bs x\bs x^\top\right]\leq\psi\bs S$.
\end{remark}

\begin{assumption}[Bound of Noise]\label{assumption:noise}
    Denote $\epsilon=y-\bs x^\top\bs w^*$ as the response noise. We assume that $\bbE\left[\epsilon^2\bs x\bs x^\top\right]\preceq\sigma^2\bs S$.
\end{assumption}
% \begin{assumption}[Bound of Noise]\label{assumption:noise}
%     We assume that $\bbE\left[\left(y-\bs x^\top\bs w^*\right)^2\bs x\bs x^\top\right]\preceq\sigma^2\bs S$.
% \end{assumption}
This assumption holds when $y-\bs x^\top\bs w^*$ is sub-Gaussian conditioned on $\mathbf{x}$.
It does not require $\mathbb{E}[\epsilon|\mathbf{x}]=0$, thus encompassing a broader class of problems. This assumption is also adopted by several works for linear regression~\citep{hsu2012random,Dieuleveut2016NonparametricStochasticApproximation,dieuleveut2017harder,jain2018accelerating,zou2021benign,lirisk}. In this paper, we regard both $\sigma^2$ and $c>0$ in Assumption~\ref{ass-finite-risk} as fixed constants.

\subsection{Minimax Lower Bound}\label{sec: main lower}

In this section, we derive the information-theoretic minimax lower bound for the class of problems $\mathcal{P}(W, \mathbf{S}, \mathbf{T})$, which encompasses the problems under consideration, as defined in Definition~\ref{ass-minimax}.
\begin{definition}\label{ass-minimax}
    The problem class $\mathcal{P }(W,\bs S,\bs T)$ consists of all independently joint distributions $P_{\bs x \times y}\times Q_{\bs x \times y}$ such that the corresponding the best estimator $\mathbf{w}^*\in W$, with $\bs S=\mathbb{E} _{P_{\mathbf{x} }}\left [ \mathbf{x}  \mathbf{x}^{\top} \right ]$ and $\bs T=\mathbb{E} _{Q_{\mathbf{x} }}\left [ \mathbf{x}  \mathbf{x}^{\top} \right ]$. And all the distributions in $\mathcal{P }(W,\bs S,\bs T)$ satisfy Assumption~\ref{asp:minimizer},\ref{ass-finite-risk} and~\ref{assumption:noise}.
\end{definition}

\begin{remark}[Prior knowledge of $\bs S$ and $\bs T$]
Note that problem class $\mathcal{P}(W,\bs S,\bs T)$ takes the source covariance matrix $\bs S$ and target covariance matrix $\bs T$ as class parameters, implying that its minimax optimality assumes prior knowledge of these matrices. This parallels the fact that that many classical methods leverage the prior knowledge from the source and the target, such as the importance weighting method \citep{cortes2010learning} which requires two probability densities. On the other hand, it aligns with semi-supervised learning paradigm, where the covariance matrix $\bs T$ and $\bs S$ can be estimated from the large amount of unlabeled data, as discussed in the previous literature \citep{lei2021near}.
\end{remark}

\begin{theorem}[Lower Bound]\label{thm:lower-bound}
Let $W=\left\{\bs w^*\in\bbR^d:\left\Vert\bs w^*\right\Vert_{\bs M}^2\leq 1\right\}$ as in \eqref{eq: constraint} and $\bs S$, $\bs T$ are positive semi-definite matrices. 
Given probability $\tilde{P}\in\mathcal{P}(W,\bs S,\bs T)$, samples $\{(\bs x_i, y_i)\}_{i=1}^n$ are drawn from the source distribution of $\tilde{P}$.
We consider any random estimator $\hat{\bs w}\left(\{(\bs x_i, y_i)\}_{i=1}^n, \xi  \right)$, where $\xi\sim P_{\xi }$ encodes the randomness of $\hat{\bs w}$. 
Then the minimax excess risk on the target domain over $\mathcal{P}(W,\bs S,\bs T)$ can be bounded from below by:
    \begin{equation}\label{eq: mainlower}
        \inf_{\hat{\bs w}}\sup_{\tilde{P} \in \mathcal{P}(W,\bs S,\bs T)}\bbE_{\tilde{P}^{\otimes n}\times P_{\xi }}\left\Vert\hat{\bs w}-\bs w^*\right\Vert_{\bs T}^2\geq\sup_{\substack{\bs F\succeq\bs O\\ \left\Vert\bs F\right\Vert_*\leq 1/\pi^2}}\left\langle\bs T',\left(\bs F^{-1}+\frac{n\bs S'}{\sigma^2}\right)^{-1}\right\rangle,
    \end{equation}
    where $\bs S'=\bs M^{-1/2}\bs S\bs M^{-1/2}$ and $\bs T'=\bs M^{-1/2}\bs T\bs M^{-1/2}$.
\end{theorem}
Theorem~\ref{thm:lower-bound} provides the \emph{first} algorithm-independent, worst-case lower bound over problem class $\mathcal{P }(W,\bs S,\bs T)$ for any instance of $\bs S$ and $\bs T$. It characterizes the information limit for linear regression under covariate shift for any $n$ and $d$.

\iffalse
An additional benefit of characterizing the problem class under the refined problem class incorporating $\mathbf{S}$ and $\mathbf{T}$ is that: it enables the derivation of an instance-specific lower bound for any $\mathbf{S}$ and $\mathbf{T}$. Consequently, for any given pair of $\mathbf{S}$ and $\mathbf{T}$, if the minimax rate in Theorem~\ref{thm:lower-bound} is achieved, there exists an estimator that is optimal for this particular instance.
\fi

\begin{remark}
We establish a Bayesian matrix lower bound for the estimator covariance under a given prior as shown in Lemma~\ref{lemma:van-trees}, which serves as a multivariate generalization of Bayesian Cramer-Rao inequality \citep{bj/1186078362}. The minimax lower bound is given by the supremum of the Bayesian risk over a sequence of constructed candidate sets of prior distributions.
\end{remark}

\section{Optimal Algorithm}\label{sec: main optimal alg}
In this section, we analyze a class of estimators, establishing that one member of this class attains the minimax optimality. 

Concretely, we consider the estimator class of
\begin{equation}
    \hat{\bs w}_{\bs A}=\frac{1}{n}\bs M^{-1/2}\bs A\bs M^{1/2}\bs S^{-1}\sum_{i=1}^n\bs x_i y_i,\label{eq:def-w-opt}
\end{equation}
where $\bs A\in\mathbb{R}^{d\times d}$ is a preconditioner for the estimation and different choices of $\bs A$ lead to a broad class of estimators. The estimator has prior knowledge on the covariance matrices $\bs S$ and $\bs T$, which is consistent with our minimax lower bound. This prior knowledge also aligns with classical methods or can be derived from a semi-supervised learning framework, as discussed in Section~\ref{sec: main lower}.

% Concretely, with preknowing covariance $\bs S$ and $\bs T$ as stated in Section~\ref{sec: main lower}, 

To understand the effect of this preconditioned estimator, consider the case when $\bs A = \bs I$ where the estimator simplifies to $\hat{\bs w}_{\bs A}=\frac{1}{n}\bs S^{-1}\sum_{i=1}^n\bs x_i y_i$ and represents the unbiased least square estimator. In contrast, when $\bs A\prec \bs I$, there is a deliberate trade-off between bias accuracy and reduced variance. Specifically, under the covariate shift scenario, determining the optimal balance between these factors becomes complex and necessitates careful selection, owing to the intricate interaction between regularization and the distributional shift.

We demonstrate that the optimal preconditioner $\bs A$ is determined by minimizing the following the objective function 
\begin{equation}\label{eq: A choice}
    \bs A=\argmin_{\bs A\in\bbR^{d\times d}}\left\Vert(\bs I-\bs A)^\top\bs T'(\bs I-\bs A)\right\Vert+\frac{2\sigma^2+2\psi\left\Vert\bs S'\right\Vert}{n}\left\langle\bs T',\bs A\left(\bs S'\right)^{-1}\bs A^\top\right\rangle,
\end{equation}
where $\bs S'=\bs M^{-1/2}\bs S\bs M^{-1/2}$ and $\bs T'=\bs M^{-1/2}\bs T\bs M^{-1/2}$.
%\congfang{The optimization problem 8 is a convex program.. So the estimator can be approximiated solved efficiently...}
The optimization problem \eqref{eq: A choice} is a convex program in  $\mathbf{A}$, ensuring the existence of a unique global minimizer. This can be efficiently solved using proximal gradient methods or semi-definite programming, even in high-dimensional settings~\citep{boyd2004convex}.
Subsequently, the optimal estimator $\hat{\bs w}^\opt$ is defined according to \eqref{eq:def-w-opt}. This estimator leads to the derivation of the following minimax upper bound.
\begin{theorem}[Upper Bound]\label{thm:optimal-upper-bound}
   Let $W=\left\{\bs w^*\in\bbR^d:\left\Vert\bs w^*\right\Vert_{\bs M}^2\leq 1\right\}$. For any positive semi-definite matrix $\bs S$ and $\bs T$, and  $\tilde{P}\in\mathcal{P}(W,\bs S,\bs T)$. Suppose we get samples $\{(\bs x_i, y_i)\}_{i=1}^n$ drawn from the source distribution of $\tilde{P}$. The excess risk of the optimal estimator $\hat{\bs w}^\opt$ defined in \eqref{eq:def-w-opt} on the target distribution of $\tilde{P}$ can be bounded from the above by:
    \begin{equation}\label{eq: main minimaxupper}
    \begin{aligned}
       &\sup_{\tilde{P} \in \mathcal{P}(W,\bs S,\bs T)}\bbE_{\tilde{P}^{\otimes n}}\left\Vert\hat{\bs w}^\opt-\bs w^*\right\Vert_{\bs T}^2\\ \leq&\min_{\bs A\in\bbR^{d\times d}}\left\Vert(\bs I-\bs A)^\top\bs T'(\bs I-\bs A)\right\Vert+\frac{2\sigma^2+2\psi\left\Vert\bs S'\right\Vert}{n}\left\langle\bs T',\bs A\left(\bs S'\right)^{-1}\bs A^\top\right\rangle,
        \end{aligned}
    \end{equation}
    where $\bs S'=\bs M^{-1/2}\bs S\bs M^{-1/2}$ and $\bs T'=\bs M^{-1/2}\bs T\bs M^{-1/2}$.
\end{theorem}
There exists a equivalence between the objective in \eqref{eq: A choice} and the objective in the minimax upper bound presented in \eqref{eq: main minimaxupper}. As demonstrated in the proof of Theorem~\ref{thm:optimal-upper-bound}, this objective corresponds to an excess risk upper bound for the estimator $\hat{\bs w}_{\bs A}$. Consequently, the estimator $\hat{\bs w}_{\bs A}$ is obtained by minimizing this excess risk upper bound over the class of estimators described in \eqref{eq:def-w-opt}.

Furthermore, Theorem~\ref{thm:lower-upper-match} demonstrates that, $\hat{\bs w}_{\bs A}$ not only minimizes the excess risk over the estimator class defined in \eqref{eq:def-w-opt}, but also attains the information-theoretic optimality.
\begin{theorem}[Matching Bounds]\label{thm:lower-upper-match}
For any positive definite matrix $\bs S$, $\bs M$ and positive semi-definite matrix $\bs T$, the following equation holds:
    \begin{equation}
        \sup_{\substack{\bs F\succeq\bs O\\ \left\Vert\bs F\right\Vert_*\leq 1/\pi^2}}\left\langle\bs T',\left(\bs F^{-1}+\frac{n\bs S'}{\sigma^2}\right)^{-1}\right\rangle=\min_{\bs A\in\bbR^{d\times d}}\frac{1}{\pi^2}\left\Vert(\bs I-\bs A)^\top\bs T'(\bs I-\bs A)\right\Vert+\frac{\sigma^2}{n}\left\langle\bs T',\bs A\left(\bs S'\right)^{-1}\bs A^\top\right\rangle,\label{eq:lower-upper-match}\notag
    \end{equation}
    where $\bs S'=\bs M^{-1/2}\bs S\bs M^{-1/2}$ and $\bs T'=\bs M^{-1/2}\bs T\bs M^{-1/2}$. 
    %\congfang{Therefore, if we treat $\sigma^2$ and $\left\Vert\bs S'\right\Vert$ as constant, the lower bound and the upper bound of the optimal algorithm are match up to constant factors.}
\end{theorem}

\section{Optimality of ASGD}\label{sec:asgd-optimality}
%\congfang{You need one paragraph to tell story For example, T is not known like LLM training on source using SGD type and test on other domain}
Modern large-scale pretrained models, particularly LLMs, are typically trained using SGD-type methods on source distributions. These models exhibit remarkable adaptability when transferred to downstream tasks with distinct target distributions. 
%Although information from the target distribution may be insufficient, SGD-type methods achieve implicit regularization through preconditioned hyperparameters such as step size and momentum, leading to strong generalization performance on unseen target domains.
In this section, we go further from characterizing the optimal algorithms under covariate shifts and consider the performance of the stochastic gradient descent and its accelerated  variants, the commonly used modern machine learning optimization methods. We provide sufficient conditions
when SGD with its acceleration variants attain optimality. Specifically, established excess risk in Section~\ref{sec: asgd upper} demonstrates that: these methods result in a specific ``greedily" preconditioned estimator, a particular case within the broader class of preconditioned estimators as in \eqref{eq:def-w-opt}. Therefore, we can leverage the previous optimality results to analyze their optimality conditions in Section~\ref{sec: sub asgd optiamlity}.
Our results demonstrate that, even in the absence of prior knowledge on $\mathbf{S}$ and $\mathbf{T}$, ASGD can achieve the optimality established in Theorem~\ref{thm:lower-bound} within a broad class of covariate shift problems.

\subsection{ASGD Upper Bound}\label{sec: asgd upper}

Our analysis is based on an integrated framework of geometrically decaying stepsizes \citep{ge2019step,wu2022last} and the standard accelerated stochastic gradient descent for linear regression \citep{jain2018accelerating,lirisk}, as presented in Algorithm~\ref{alg:asgd}. (When $\gamma=\delta$, Algorithm~\ref{alg:asgd} reduces to the vanilla stochastic gradient descent method with geometrically decaying step size.) 
Theorem~\ref{thm:asgd-upper-bound} establishes an excess risk upper bound of ASGD on the target distribution.

\begin{algorithm}[t]
    \caption{Accelerated Stochastic Gradient Descent (ASGD) with exponentially decaying step size}\label{alg:asgd}
    \begin{algorithmic}
    \Require {Initial weight $\bs w_0$, initial step size $\delta$, $\gamma$, momentum parameters $\alpha$, $\beta$, Dataset $\mathcal{D}=\left\{\left(\bs x_i,y_i\right)\right\}_{i=1}^n$}
    \State $\bs v_0\gets\bs w_0, t\gets 1$
    \For{$\ell=1,2,\ldots,\log_2 n$}
        \State $\delta_{(\ell)}\gets\delta_0/4^{\ell-1}$, $\gamma_{(\ell)}\gets\gamma_0/4^{\ell-1}$
        \For{$t=1,2,\ldots,\frac{n}{\log_2 n}$}
            \State $\bs u_{i-1}\gets\alpha\bs w_{i-1}+(1-\alpha)\bs v_{i-1}$
            \State $\bs g_i\gets\left(\bs x_i^\top\bs u_{i-1}-y_i\right)\bs x_i$
            \State $\bs w_i\gets\bs u_{i-1}-\delta_{(\ell)}\bs g_i$
            \State $\bs v_i\gets\beta \bs u_{i-1}+(1-\beta)\bs v_{i-1}-\gamma_{(\ell)}\bs g_i$
            \State $i\gets i+1$
        \EndFor
    \EndFor
    \Ensure $\bs w_n$

    \end{algorithmic}
\end{algorithm}
\begin{theorem}[Upper Bound of Accelerated SGD]\label{thm:asgd-upper-bound}
Let $W=\left\{\bs w^*\in\bbR^d:\left\Vert\bs w^*\right\Vert_{\bs M}^2\leq 1\right\}$. For any positive semi-definite matrix $\bs S$ and $\bs T$, and  $\tilde{P}\in\mathcal{P}(W,\bs S,\bs T)$. Suppose we get samples $\{(\bs x_i, y_i)\}_{i=1}^n$ drawn from the source distribution of $\tilde{P}$. 
 When $n\geq16$, we choose the hyperparameters according to Section~\ref{sec:param-choice}. Denote the output of Algorithm~\ref{alg:asgd} as $\bs w_n^\sgd$, the excess risk of $\bs w_n^\sgd$ on the target distribution of $\tilde{P}$ can be bounded from the above by
    \begin{equation}\label{eq: asgdupper}
    \begin{aligned}        
        \bbE_{\tilde{P}^{\otimes n}}\left\Vert\bs w_n^\sgd-\bs w^*\right\Vert_{\bs T}^2\leq&\underbrace{\left(\sigma^2+2c\right)\cdot\left[\sum_{i=1}^{k^*}{\frac{2t_{ii}}{K\lambda_i}}+\frac{128}{15}K\left(\gamma+\delta\right)^2\sum_{i=k^*+1}^{d}\lambda_i t_{ii}\right]}_\text{Effective Variance}\\
        &+\underbrace{\frac{\left\Vert\bs T'_{0:k^*}\right\Vert}{8n^2(\log_2 n)^4}+4\left\Vert\bs T'_{k^*:\infty}\right\Vert}_\text{Effective Bias}, \\
    \end{aligned}
    \end{equation}
    where $k^*=\max\left\{k:\lambda_k>\frac{32\ln n}{(\gamma+\delta)K}\right\}$, often referred to as the effective dimension~\citep{bartlett2020benign, zou2021benign}, $K=\frac{n}{\log_2 n}$, $\bs S'=\bs M^{-1/2}\bs S\bs M^{-1/2}$, $\bs T_{0:k^*}'=\bs M^{-1/2}\bs T_{0:k^*}\bs M^{-1/2}$ and $\bs T_{k^*:\infty}'=\bs M^{-1/2}\bs T_{k^*:\infty}\bs M^{-1/2}$.
\end{theorem}

\begin{remark} The upper bound \eqref{eq: asgdupper} decomposes into effective bias and effective variance. The effective variance arises from the randomness of the samples, while the effective bias corresponds to the excess risk of the result at the population level optimization.
\end{remark}
 One viewpoint on considering the convergence of SGD is that it solves along different eigen directions of the source covariance, with varying convergence behavior depending on the corresponding eigenvalue sizes. To illustrate this, the excess risk in \eqref{eq: asgdupper} appears as a summation over each eigenvalue dimension in the source eigenspace. ASGD demonstrates distinct behaviors in two subspaces separated by effective dimension $k^*$. In the larger eigenvalue space indexed by $k \leq k^*$, the population convergence of ASGD approaches $\mathbf{w}^*$ well, with statistical error being the dominant factor. In the remaining subspace, ASGD’s estimation remains close to zero, meaning the excess risk bound in this space depends heavily on prior information. Increasing the momentum can enlarge $k^*$, which enables the optimization of more challenging $\mathbf{w}^*$.

\subsection{Optimality under the Diagonal Dominant Condition}\label{sec: sub asgd optiamlity}

The eigendimension-wise dynamics of ASGD can be interpreted as greedily prioritizing estimation along the eigendirections of the source covariance matrix $\bs S$ in descending order of its eigenvalues. This behavior serves as a specialized preconditioner, represented by $\bs A = \bs U\diag\{1,\ldots,1,0,\ldots,0\}$ as in \eqref{eq:def-w-opt} ), which selectively focuses on the large eigenspaces of $\bs S$. Given this connection between ASGD and the preconditioned estimator, a natural question arises: for which target distribution classes can ASGD achieve optimality? To characterize the rate-optimality of ASGD with respect to the sample size $n$, we adopt the following power-law anisotropic covariance structures and best estimator constraint assumptions:
\begin{assumption}\label{assumption:poly-decay-source}
We make the following assumption for analyzing the optimality of ASGD:
\begin{itemize}
    \item We assume $\lambda_i\eqsim i^{-a}$ with $a>1$.
    \item We assume $\bs M$ commutes with $\bs S$, and its diagonal entries $m_i\eqsim\lambda_i^{1-s}$.
\end{itemize}
\end{assumption}
These assumptions 
align with the canonical capacity condition and source condition in RKHS framework~\citep{caponnetto2007optimal} and widely used in studying the optimality in linear regression \cite{Dieuleveut2016NonparametricStochasticApproximation,10.1214/17-EJS1258,JMLR:v25:23-0383}. Under these assumptions, the excess risk of ASGD is shown to decay polynomially with the sample size $n$ in the case of no covariate shift.
 To investigate the optimality of ASGD within a specific class of target distributions, previous studies have primarily focused on the bounded likelihood ratio target distribution class, which essentially corresponds to a shift to a distribution similar in shape to the source distribution. To extend this function class and provide a comprehensive analysis of how power-law anisotropic structures influence statistical optimality, we consider the following $r$-smooth target distribution class:

\begin{definition}\label{def:r-smooth-class}
    Let $r\in\bbR$. A class $\cQ$ of covariate marginal distributions $Q_{\mathbf{x}}$ in the target domain is called a $r$-smooth class if $\cQ$ satisfies the following conditions:
    \begin{enumerate}
        \item There exists a constant $C>0$, such that for any  $Q_{\mathbf{x}}\in\cQ$,  $\bs S^{-\frac{1}{2}}\bbE_{\bs x\sim Q_{\mathbf{x}}}\left[\bs x\bs x^\top\right]\bs S^{-\frac{1}{2}}\preceq C\bs S^{r}$.
        \item There exist a constant $c>0$ and  $Q_{\mathbf{x}_0}\in\cQ$, such that $\bs S^{-\frac{1}{2}}\bbE_{\bs x\sim Q_{\mathbf{x}_0}}\left[\bs x\bs x^\top\right]\bs S^{-\frac{1}{2}}\succeq c\bs S^{r}$.
    \end{enumerate}
\end{definition}
The $r$-smooth distribution class includes the following two commonly encountered distribution classes.

\begin{example}[Density Ratio Bounded Class]
    Let $B>0$ be a constant and $\cQ=\left\{Q:\rmd Q_{\bs x }/\rmd P_{\bs x }\leq B\right\}$ be the distribution class with density ratio bounded by $B$. Then, $\cQ$ is $0$-smooth class. 
\end{example}

\begin{example}[Gaussian $D_\mathrm{KL}$ Bounded Class]\label{example:gaussian}
    Suppose the source distribution $P_{\mathbf{x}}$ is Gaussian. Let $\epsilon>0$ and $\cQ$ be the Gaussian distribution class with bounded Kullback–Leibler divergence, namely $\cQ=\left\{Q_{\bs x }:D_\mathrm{KL}(Q_{\bs x }\Vert P_{\bs x })<\epsilon, Q_{\bs x } \text{ is Gaussian}\right\}$. Then we have $\cQ$ is a $0$-smooth class. 
\end{example}

\begin{remark}
    In Definition~\ref{def:r-smooth-class}, the mismatch between the source and target distributions is measured through the similarity matrix $\bs S^{-1/2}\bbE_{\bs x\sim Q_{\mathbf{x}}}\left[\bs x\bs x^\top\right]\bs S^{-1/2}$. Intuitively, this definition measures the alignment between the source and target distributions, and the magnitude is characterized by the source and smoothness coefficient $r$. This allows us to analyze a broader class of distributions than previously considered for $r=0$, particularly for $r<0$, which corresponds to inherently more difficult problems.
\end{remark}
We then examine the optimality region of ASGD for $r$-smooth class $\cQ$ under Assumption~\ref{def:r-smooth-class}.
\begin{theorem}\label{thm:asgd-poly-class-opt}
% gaiwei gei Q,M,S under ass? with proper para SGD opt reg? rate? ASgD further reg rate?
    Suppose Assumptions~\ref{assumption:poly-decay} holds. Let $\cQ$ be a $r$-smooth class with $r>\max\{\frac{1}{a}-2,-s\}$. With properly chosen parameters, SGD achieves the following minimax excess rate:
    \begin{equation}
    \begin{aligned}
        &\sup_{Q_{\mathbf{x}}\in\cQ,\ \tilde{P} \in \mathcal{P}(W,\bs S,\bs T_{Q_{\mathbf{x}}})}\bbE_{\tilde{P}^{\otimes n}}\left\Vert\bs w_n^\sgd-\bs w^*\right\Vert_{\bs T_{Q_{\mathbf{x}}}}^2 \\
        \eqsim&\begin{cases}
            (1+\sigma^2)\ln n/n,&r>1/a; \\
            (1+\sigma^2)(\ln n)^3/n,&r=1/a; \\
            (1+\sigma^2)(1/n)^\frac{(r+s)a}{sa+1}(\ln n)^\frac{3(1+r)a-3}{a},&1/a>r>\max\{1/a-2,-s\}. \\
            % (1+\sigma^2)(1/n)^\frac{sa-\kappa}{sa+1}(\ln n)^{\frac{3(a-\kappa-1)}{a}},&-1<\kappa<\min\{2a-1,sa\}. \\
        \end{cases}        
    \end{aligned}
    \end{equation}
    For the region $s\geq 1-\frac{1}{a}$, vanilla SGD achieves 
 optimality up to logarithmic factors; for the broader region $1-\frac{1}{a}>s>\frac{(a-1)^2}{a(2a-1)}$, ASGD achieves optimality up to logarithmic factors.
\end{theorem}
This theorem characterizes the optimality region of ASGD for the $r$-smooth class $\mathcal{Q}$, demonstrating that the inclusion of momentum leads to a broader optimal region. This is because a smaller $s$ in Assumption~\ref{thm:asgd-poly-class-opt} leads to a more challenging initial bias. SGD, constrained by its optimization capacity, can only achieve optimality within region $s\geq 1-\frac{1}{a}$. However, by adding momentum, the effective dimension is increased, enhancing its ability to optimize the bias, thereby achieving optimality over a broader region $1-\frac{1}{a}>s>\frac{(a-1)^2}{a(2a-1)}$. 
This theorem demonstrates that, even in the absence of prior knowledge about $\mathbf{S}$ and $\mathbf{T}$, ASGD can achieve optimality across a wide class of covariate shift problems with appropriately tuned step size and momentum.

\subsection{Further Discussions}
\subsubsection{Optimality Beyond the Diagonal Dominant Condition}
%In Section~\ref{sec: sub asgd optiamlity}, we 
While the previous subsection
discusses the optimality of ASGD within the $r$-smooth distribution class, this condition is not necessary for ASGD to achieve optimality. For any given source distribution and target distribution, the optimality gap of ASGD fundamentally reduces to comparing its achievable upper bound in~\eqref{eq: asgdupper} against the information-theoretic lower bound for the covariate shift problem~\eqref{eq: mainlower}. 
We demonstrate a more general target distribution optimality condition.
\begin{assumption}\label{assumption:poly-decay}
    Let $a,s,r,\nu\in\bbR$ be constants. Given the target distribution $Q$, let $\bs T=\bbE_{\bs x\sim Q}\bs x\bs x^\top$. We make the following assumptions on the source and target covariance matrix:
    \begin{enumerate}
        \item Let $t_{ii}$ denote the $i$-th diagonal element of $\bs T$. We assume $\lambda_i\eqsim i^{-a}$ and $t_{ii}\eqsim i^{-(1+r)a}$.
        \item We assume that $\left\Vert\bs M^{-1/2}\bs T_{i:\infty}\bs M^{-1/2}\right\Vert\eqsim i^{-(r+s)a+\nu}$. 
    \end{enumerate}
\end{assumption}
    Since $\max_{k>i} t_{kk}/m_{kk}\leq\left\Vert\bs M^{-1/2}\bs T_{i:\infty}\bs M^{-1/2}\right\Vert\leq\tr\left(\bs M^{-1/2}\bs T_{i:\infty}\bs M^{-1/2}\right)$, we have $0\leq\nu\leq1$. The case where $\bs T$ is diagonal corresponds to $\nu=0$, and the case where $T$ has a rank-$1$ structure corresponds to $\nu=0$. For instance, we provide the optimality analysis of ASGD when $T$ has a rank-$1$ structure $\bs T=\mathbf{w}\mathbf{w}^{\top}$, where the excess risk on the target distribution is typically interpreted as the prediction error on $\mathbf{w}$.
\begin{theorem}[Rank-$1$ Structure]\label{thm:rank-one}
     Consider $\lambda_i\eqsim i^{-a}$ and the target covariance matrix $\bs T=\mathbf{w}\mathbf{w}^{\top}$ where $\mathbf{w}\in \mathbb{R}^d$ and $\bs w_i\eqsim i^{-(1+r)a/2}$.  For region $s\geq 1$, vanilla SGD achieves optimality up to logarithmic factors; for region $1>s>\frac{a}{2a-1}$, ASGD achieves optimality up to logarithmic factors.
\end{theorem}
    This theorem characterizes the optimality region of ASGD when the target covariance matrix $\bs T$ has a rank-$1$ structure.
    Once again, it highlights that the inclusion of momentum leads to a broader optimal region.  Due to space limit, a more general optimality analysis for distributions satisfying Assumption~\ref{assumption:poly-decay} is shown in Appendix~\ref{sec:general-version-poly-decay}.

\subsubsection{When is Emergence Possible?}
When scaling up the training of large language models, models may suddenly perform much better on downstream tasks after hitting a critical sample size—an amazing phenomenon often referred to as emergence~\citep{wei2022emergent}. In covariate shift, emergence is possible when downstream tasks 
require high-quality estimations for parameters at  local regions, even though source excess risk decays regularly. When this happens, the excess risk on the target distribution will decrease slowly until the sample size guarantees a good estimation for the regions.
We provide a simple example to illustrate the emergent behavior.
 
\begin{example}\label{example:emergence}
    We suppose $\bs S=\diag\{i^{-a}\}_{i=1}^d$ and $\bs M=\bs I$. Let $d_0\in[d]$, we consider target covariance matrix $\bs T=\diag\left\{\left(\max\{i,d_0\}\right)^{-(1+r)a}\right\}_{i=1}^d$, where $-1<r<1/a$. With properly chosen hyperparameters, SGD achieves optimality up to logarithmic factors with the following upper bound:
    \begin{equation}
        \bbE_{\tilde{P}^{\otimes n}}\left\Vert\bs w_n^\sgd-\bs w^*\right\Vert_{\bs T}^2\leq\begin{cases}
            \tilde{\mathcal{O}}\left(d_0^{-(1+r)a}\right),&n\lesssim d_0^{a+1}; \\
            \tilde{\mathcal{O}}\left(n^{-\frac{(1+r)a}{1+a}}\right),&n\gtrsim d_0^{a+1}.
        \end{cases}
    \end{equation}
\end{example}
The target task requires precise estimation of the first $d_0$ parameters, while only the first $\mathcal{O}\left(n^\frac{1}{a+1}\right)$ parameters can be estimated precisely given $n$ source samples. SGD returns nearly zero for the remaining parameters, resulting in bias in estimation. When $n\lesssim d_0^{a+1}$, SGD yields biased estimation at least for the $d_0$-th parameter, indicating no decrease in excess risk. Once sample size exceeds $d_0^{a+1}$, SGD exploits information from samples to eliminate the bias of the first $d_0$ parameters, thereby rapidly reducing risk. In a word, the emergent phenomenon arises from the transition to a regime with sufficient samples, where SGD succeeds in estimating the important parameters of the target task.

\section{Conclusion}
This work establishes the minimax information-theoretic optimality for high-dimensional linear regression under covariate shift. Further, we  analyze the convergence behavior of stochastic gradient descent (SGD) and its accelerated variant (ASGD) within this optimality framework. 
This work focuses on characterizing covariate shift problem through a theoretical lens, whereas an interesting yet challenging extension of this research lies in translating these theoretical guarantees into novel practical methodologies, such as partially preconditioned methods which reduce the computation. We leave this study as a future work.

\newpage
% Acknowledgments---Will not appear in anonymized version
% \acks{We thank a bunch of people and funding agency.}

\newpage

\bibliographystyle{plainnat}
\bibliography{refs}

\newpage

\appendix

\section{Proof of Theorem~\ref{thm:lower-bound}}
This section provides the proof of the lower bound. For any $\bs w\in W$, we construct the probability distribution $P_{\bs w}$ of $(\bs x,y)$ such that
\begin{equation}
    \bs x\sim\mathcal{N}(\bs 0,\bs S), \quad y=\bs x^\top\bs w+\epsilon,
\end{equation}
where $\epsilon\sim\mathcal{N}(\bs 0,\sigma^2)$ and $\epsilon$ and $\bs x$ are independent. $P_{\bs w}$ satisfies Assumptions~\ref{assumption:fourth-moment} and~\ref{assumption:noise}. Let $\mathcal{G}(W,\bs S,\bs T)=\{P_{\bs w}:\bs w\in W\}$ denotes the Gaussian problem class, then we have $\mathcal{G}(W,\bs S,\bs T)\subseteq\mathcal{P}(W,\bs S,\bs T)$.

The first step is to reduce the minimax risk to Bayesian risk and show that the randomness of the estimator $\hat{\bs w}$ does not help to achieve better performance. We denote an estimator which only depends on samples $\{(\bs x_i,y_i)\}_{i=1}^n$ as $\hat{\bs w}^\mathrm{det}$. We have the following lemma.
\begin{lemma}\label{lemma:random-estimator}
    Suppose $\pi$ is any probability distribution supported on $W$, we have
    \begin{equation}        
        \inf_{\hat{\bs w}}\sup_{\tilde{P} \in \mathcal{P}(W,\bs S,\bs T)}\bbE_{\tilde{P}^{\otimes n}\times P_{\xi }}\left\Vert\hat{\bs w}-\bs w^*\right\Vert_{\bs T}^2\geq\inf_{\hat{\bs w}^\mathrm{det}}\bbE_{\bs w^*\sim\pi}\bbE_{P_{\bs w^*}^{\otimes n}}\left\Vert\hat{\bs w}-\bs w^*\right\Vert_{\bs T}^2.
    \end{equation}
\end{lemma}
\begin{proof}
    From Yao's minimax principle~\citep{yao1977probabilistic}, we have
    \begin{equation}
    \begin{aligned}        
        \inf_{\hat{\bs w}}\sup_{\tilde{P} \in \mathcal{P}(W,\bs S,\bs T)}\bbE_{\tilde{P}^{\otimes n}\times P_\xi}\left\Vert\hat{\bs w}-\bs w^*\right\Vert_{\bs T}^2\geq&\inf_{\hat{\bs w}}\sup_{P_{\bs w^*} \in \mathcal{G}(W,\bs S,\bs T)}\bbE_{P_{\bs w^*}^{\otimes n}\times P_{\xi }}\left\Vert\hat{\bs w}-\bs w^*\right\Vert_{\bs T}^2 \\
        \geq&\inf_{\hat{\bs w}}\bbE_{\bs w^*\sim\pi}\bbE_{P_{\bs w^*}^{\otimes n}\times P_{\xi }}\left\Vert\hat{\bs w}-\bs w^*\right\Vert_{\bs T}^2 \\
        \geq&\inf_{\xi}\inf_{\hat{\bs w}}\bbE_{\bs w^*\sim\pi}\bbE_{P_{\bs w^*}^{\otimes n}}\left\Vert\hat{\bs w}(\cdot,\xi)-\bs w^*\right\Vert_{\bs T}^2 \\
        \geq&\inf_{\hat{\bs w}^\mathrm{det}}\bbE_{\bs w^*\sim\pi}\bbE_{P_{\bs w^*}^{\otimes n}}\left\Vert\hat{\bs w}^\mathrm{det}-\bs w^*\right\Vert_{\bs T}^2.
    \end{aligned}
    \end{equation}
\end{proof}

We prove a multivariate generalization of Bayesian Cramer-Rao inequality in~\citet{bj/1186078362}.
\begin{lemma}\label{lemma:van-trees}
    We denote the density function of $P_{\bs w}^{\otimes n}$ as $f_{\bs w}$. Given data $X=\left\{(\bs x_i,y_i)\right\}_{i=1}^n\sim P_{\bs w}^{\otimes n}$, let $\hat{\bs w}^\mathrm{det}=\hat{\bs w}^\mathrm{det}(X)$ be an estimator of $\bs w$. The Fisher information matrix of $P_{\bs w}^{\otimes n}$ be defined as
    \begin{equation}\label{eq:information-theta}
        \mathcal{I}(\bs w)=\int_{\cX}{\left(\nabla_{\bs w}\ln f_{\bs w}(x)\right)\left(\nabla_{\bs w}\ln f_{\bs w}(x)\right)^\top f_{\bs w}(x)\rmd x}.
    \end{equation} 
    Consider a prior probability measure $\pi$ with density function $\pi(\bs w)$ that is supported on a compact set $W\subseteq\bbR^d$ and $\pi(\bs w)=0$ on the boundary of $W$. We define the information matrix of $\pi$ as
    \begin{equation}\label{eq:information-lambda}
        \mathcal{I}(\pi)=\int_{\bbR^d}{\left(\nabla\ln\pi(\bs w)\right)\left(\nabla\ln\pi(\bs w)\right)^\top\pi(\bs w)\rmd\bs w}.
    \end{equation}
    Then we have
    \begin{equation}
        \bbE_{\bs w\sim\pi}\bbE_{X\sim P_{\bs w}^{\otimes n}}\left(\hat{\bs w}-\bs w\right)\left(\hat{\bs w}-\bs w\right)^\top\succeq\left(\bbE_{\bs w\sim\pi}\mathcal{I}(\bs w)+\mathcal{I}(\pi)\right)^{-1}.
    \end{equation}
\end{lemma}
\begin{proof}
    We begin by defining two random variables as
    \begin{equation}
        \bs\xi=\hat{\bs w}^\mathrm{det}(X)-\bs w, \quad\bs\eta=\nabla_{\bs w}\ln\left(f_{\bs w}(X)\pi(\bs w)\right).
    \end{equation}    
    We denote $\bbE_{\bs w\sim\pi}\bbE_{X\sim P_{\bs w}^{\otimes n}}$ by $\bbE$ for simplicity. For any vector $\bs u, \bs v\in\bbR^d$, by Cauchy-Schwarz inequality, we have
    \begin{equation}
        \bbE\left(\bs u^\top\bs\xi\bs\xi^\top\bs u\right)\bbE\left(\bs v^\top\bs\eta\bs\eta^\top\bs v\right)\geq\left[\bbE\left(\bs u^\top\bs\xi\right)\left(\bs v^\top\bs\eta\right)\right]^2.
    \end{equation}
    We will show that
    \begin{equation}
        \bbE\bs\eta\bs\eta^{\top}=\bbE\mathcal{I}(\bs w)+\mathcal{I}(\pi), \quad\bbE\bs\xi\bs\eta^{\top}=\bs I.\label{eq:van-trees-calculation}
    \end{equation}
    Note that once we have established \eqref{eq:van-trees-calculation}, we have
    \begin{equation}
        \left[\bs u^\top\bbE\left(\hat{\bs w}^\mathrm{det}-\bs w\right)\left(\hat{\bs w}^\mathrm{det}-\bs w\right)^\top\bs u\right]\left[\bs v^\top\left(\bbE\mathcal{I}(\bs\theta)+\mathcal{I}(\lambda)\right)\bs v\right]\geq\left(\bs u^\top\bs v\right)^2.
    \end{equation}
    Let $\bs v=\left(\bbE\mathcal{I}(\bs\theta)+\mathcal{I}(\lambda)\right)^{-1}\bs u$, we get
    \begin{equation}
        \bs u^\top\bbE\left(\hat{\bs w}^\mathrm{det}-\bs w\right)\left(\hat{\bs w}^\mathrm{det}-\bs w\right)^\top\bs u\geq\bs u^\top\left(\bbE\mathcal{I}(\bs\theta)+\mathcal{I}(\lambda)\right)^{-1}\bs u.
    \end{equation}
    Since $\bs u$ is arbitrary, we get the desired result.
    
    Now, we prove \eqref{eq:van-trees-calculation} by direct calculation. Consider the $ij$-th entry of $\bbE\bs\eta\bs\eta^{\top}$, which is
    \begin{equation}
    \begin{aligned}        
        \bbE\bs\eta_i\bs\eta_j=&\bbE\fracpartial{\ln\left(f_{\bs w}(X)\pi(\bs w)\right)}{\bs w_i}\fracpartial{\ln\left(f_{\bs w}(X)\pi(\bs w)\right)}{\bs w_j} \\
        =&\bbE\left(\fracpartial{\ln f_{\bs w}(X)}{\bs w_i}+\fracpartial{\ln \pi(\bs w)}{\bs w_i}\right)\left(\fracpartial{\ln f_{\bs w}(X)}{\bs w_j}+\fracpartial{\ln \pi(\bs w)}{\bs w_j}\right) \\
        =&\bbE\fracpartial{\ln f_{\bs w}(X)}{\bs w_i}\fracpartial{\ln f_{\bs w}(X)}{\bs w_j}+\bbE\fracpartial{\ln \pi(\bs w)}{\bs w_i}\fracpartial{\ln \pi(\bs w)}{\bs w_j} \\
        &+\bbE\fracpartial{\ln f_{\bs w}(X)}{\bs w_i}\fracpartial{\ln \pi(\bs w)}{\bs w_j}+\bbE\fracpartial{\ln f_{\bs w}(X)}{\bs w_j}\fracpartial{\ln \pi(\bs w)}{\bs w_i} \\
        \stackrel{a}{=}&\bbE\mathcal{I}_{ij}(\bs w)+\mathcal{I}_{ij}(\pi)+\bbE\fracpartial{\ln f_{\bs w}(X)}{\bs w_i}\fracpartial{\ln \pi(\bs w)}{\bs w_j}+\bbE\fracpartial{\ln f_{\bs w}(X)}{\bs w_j}\fracpartial{\ln \pi(\bs w)}{\bs w_i},
    \end{aligned}
    \end{equation}
    where $\stackrel{a}{=}$ uses the definition of $\mathcal{I}(\bs w)$ and $\mathcal{I}(\pi)$. We need to show that
    \begin{equation}
        \bbE\fracpartial{\ln f_{\bs w}(X)}{\bs w_i}\fracpartial{\ln \pi(\bs w)}{\bs w_j}=\bbE\fracpartial{\ln f_{\bs w}(X)}{\bs w_j}\fracpartial{\ln \pi(\bs w)}{\bs w_i}=0.
    \end{equation}
    For simplicity, let $\mathcal{X}=\left(\bbR^{d}\times\bbR\right)^{n}$ be the range of $X$, then we have
    \begin{equation}
    \begin{aligned}        
        \bbE\fracpartial{\ln f_{\bs w}(X)}{\bs w_i}\fracpartial{\ln \pi(\bs w)}{\bs w_j}=&\int_{\cX\times\bbR^d}{\fracpartial{\ln f_{\bs w}(x)}{\bs w_i}\fracpartial{\ln \pi(\bs w)}{\bs w_j}f_{\bs w}(x)\pi(\bs w)\rmd x\rmd\bs w} \\
        =&\int_{\cX\times\bbR^d}{\fracpartial{f_{\bs w}(x)}{\bs w_i}\fracpartial{\pi(\bs w)}{\bs w_j}\rmd x\rmd\bs w} \\
        \stackrel{a}{=}&\int_{\bbR^d}{\left(\fracpartial{}{\bs w_i}\int_{\cX}{f_{\bs w}(x)\rmd x}\right)\fracpartial{\pi(\bs w)}{\bs w_j}\rmd\bs w} \\
        \stackrel{b}{=}&0,
    \end{aligned}
    \end{equation}
    where $\stackrel{a}{=}$ exchanges $\int_{\cX}$ and $\fracpartial{}{\bs w_i}$, and $\stackrel{b}{=}$ uses $\int_{\cX}{f_{\bs w}(x)\rmd x}\equiv 1$ and the derivative of a constant is $0$. Thus, $\bbE\bs\eta\bs\eta^{\top}=\bbE\mathcal{I}(\bs w)+\mathcal{I}(\pi)$.

    Consider the $ij$-th entry of $\bbE\bs\xi\bs\eta^{\top}$, which is
    \begin{equation}
    \begin{aligned}
        \bbE\bs\xi_i\bs\eta_j=&\bbE\left(\hat{\bs w}_i^\mathrm{det}(X)-\bs w_i\right)\fracpartial{\ln\left(f_{\bs w}(X)\pi(\bs w)\right)}{\bs w_j} \\
        =&\int_{\cX\times\bbR^d}{\left(\hat{\bs w}_i^\mathrm{det}(x)-\bs w_i\right)\fracpartial{\ln\left(f_{\bs w}(x)\pi(\bs w)\right)}{\bs w_j}f_{\bs w}(x)\pi(\bs w)\rmd x\rmd\bs w} \\
        =&\int_{\cX\times\bbR^d}{\left(\hat{\bs w}_i^\mathrm{det}(x)-\bs w_i\right)\fracpartial{\left(f_{\bs w}(x)\pi(\bs w)\right)}{\bs w_j}\rmd x\rmd\bs w} \\
        \stackrel{a}{=}&\int_{\cX\times\bbR^d}{\fracpartial{\left[\left(\hat{\bs w}^\mathrm{det}_i(x)-\bs w_i\right)f_{\bs w}(x)\pi(\bs w)\right]}{\bs w_j}\rmd x\rmd\bs w} \\
        &-\int_{\cX\times\bbR^d}{\fracpartial{\left(\hat{\bs w}^\mathrm{det}_i(x)-\bs w_i\right)}{\bs w_j}f_{\bs w}(x)\pi(\bs w)\rmd x\rmd\bs w} \\
        \stackrel{b}{=}&\int_{\cX\times\bbR^{d-1}}{\left[\left(\hat{\bs w}^\mathrm{det}_i(x)-\bs w_i\right)f_{\bs w}(x)\pi(\bs w)\right]\bigr\vert_{\bs w_j=-\infty}^{\bs w_j=+\infty}\rmd x\prod_{k\neq j}{\rmd \bs w_k}}-\bbE\fracpartial{\left(-\bs w_i\right)}{\bs w_j} \\
        \stackrel{c}{=}&\delta_{ij},
    \end{aligned}
    \end{equation}
    where $\stackrel{a}{=}$ uses integration by parts, $\stackrel{b}{=}$ integrates with respect to $\bs w_j$, and $\stackrel{c}{=}$ is from the fact that $W$ is compact, so $\lambda(\bs w)=0$ when $\bs w_j$ is sufficiently large, and $\delta_{ij}$ denotes the kronecker delta, which equals to the $ij$-th entry of identity matrix $\bs I$. Therefore, $\bbE\bs\eta\bs\eta^\top=\bs I$. This completes the proof of \eqref{eq:van-trees-calculation}.
\end{proof}

The above lemma provides a Bayesian Cramer-Rao inequality, which enables us to derive the lower bound in Theorem~\ref{thm:lower-bound}.

\begin{proof}[Proof of Theorem~\ref{thm:lower-bound}]
    We apply Lemma~\ref{lemma:van-trees}. In our case, let data $X=\left\{(\bs x_i,y_i)\right\}_{i=1}^n\sim P_{\bs w^*}^{\otimes n}$. By direct calculation, we have
    \begin{equation}
        \cI(\bs w^*)=\frac{n\bs S}{\sigma^2}.
    \end{equation}
    Thus, given any prior distribution $\pi$ with support included in $W=\left\{\bs w^*\in\bbR^d:\left\Vert\bs w^*\right\Vert_{\bs M}^2\leq 1\right\}$, by Lemma~\ref{lemma:random-estimator} we have
    \begin{equation}
        \inf_{\hat{\bs w}}\sup_{\tilde{P} \in \mathcal{P}(W,\bs S,\bs T)}\bbE_{\tilde{P}^{\otimes n}\times P_\xi}\left\Vert\hat{\bs w}-\bs w^*\right\Vert_{\bs T}^2\geq\left\langle\bs T,\left(\mathcal{I}(\pi)+\frac{n\bs S}{\sigma^2}\right)^{-1}\right\rangle.
    \end{equation}
    
    The rest of the proof is to construct the prior distribution. To build intuition, we first consider the case $\bs M=\bs I$. We construct the prior distribution $\pi$ as follows. Given any orthogonal matrix $\bs U$ and vector $\bs g$ with $\left\Vert\bs g\right\Vert\leq 1$, we define the prior density $\pi(\bs w;\bs U,\bs g)$, whose support is included in unit ball, as follows:
    \begin{equation}
        \pi(\bs w;\bs U,\bs g)=\prod_{i=1}^{d}\cos^2\left(\frac{\pi(\bs U^\top\bs w)_i}{2\bs g_i}\right)\mathds{1}_{\left\vert(\bs U^\top\bs w)_i\right\vert\leq\left\vert\bs g_i\right\vert},\label{eq:def-prior}
    \end{equation}
    where $\mathds{1}$ is the indicator function. Note that  $\pi(\bs w;\bs U,\bs g)$ has support is included in unit ball. Direct calculation shows that the information of $\pi$ is
    \begin{equation}
        \cI(\pi(\cdot;\bs U,\bs g))=\pi^2\bs U\diag\left\{\frac{1}{\bs g_1^2},\frac{1}{\bs g_2^2},\ldots,\frac{1}{\bs g_d^2}\right\}\bs U^\top.\label{eq:prior-ball-information-matrix}
    \end{equation}

    For a general positive definite matrix $\bs M$, we define a prior as follows:
    \begin{equation}
        \pi(\bs w;\bs U,\bs g,\bs M)=\left(\det\bs M^{1/2}\right)\pi(\bs M^{1/2}\bs w;\bs U;\bs g).
    \end{equation}
    Geometrically speaking, $\pi(\bs w;\bs U,\bs g,\bs M)$ is obtained by scaling $\pi(\bs w;\bs U,\bs g)$ along the eigenvector of $\bs M$, such that unit circle is transformed into the ellipse $\bs x^\top\bs M\bs x=1$, and then normalize it by the factor $\det\bs M^{1/2}$. Note that the support of $\pi(\bs w;\bs U,\bs g,\bs M)$ is included in $W=\left\{\bs w^*\in\bbR^d:\left\Vert\bs w^*\right\Vert_{\bs M}^2\leq 1\right\}$. Then, we calculate the information matrix of $\pi(\bs w;\bs U,\bs g,\bs M)$. Let $\bs s(\bs w)=\nabla\ln\pi(\bs w;\bs U;\bs g)$, we have $\nabla\ln\pi(\bs w;\bs U,\bs g,\bs M)=\bs M^{1/2}\bs s(\bs M^{1/2}\bs w)$. Therefore,
    \begin{equation}
    \begin{aligned}        
        &\mathcal{I}(\pi(\cdot;\bs U,\bs g,\bs M)) \\
        =&\int_{\bbR^d}{\left(\bs M^{1/2}\bs s(\bs M^{1/2}\bs w)\right)\left(\bs M^{1/2}\bs s(\bs M^{1/2}\bs w)\right)^\top\left(\det\bs M^{1/2}\right)\pi(\bs M^{1/2}\bs w;\bs U;\bs g)\rmd\bs w} \\
        =&\bs M^{1/2}\left[\int_{\bbR^d}\bs s(\bs v)\bs s(\bs v)^\top\pi(\bs v;\bs U;\bs g)\rmd\bs v\right]\bs M^{1/2}\quad\left(\bs v=\bs M^{1/2}\bs w\right)\\
        \stackrel{a}{=}&\pi^2\bs M^{1/2}\bs U\diag\left\{\frac{1}{\bs g_1^2},\frac{1}{\bs g_2^2},\ldots,\frac{1}{\bs g_d^2}\right\}\bs U^\top\bs M^{1/2},
    \end{aligned}
    \end{equation}
    where $\stackrel{a}{=}$ uses the result of the information matrix of $\pi(\bs w;\bs U;\bs g)$ in \eqref{eq:prior-ball-information-matrix}. Therefore, all the information matrices constitute the set $\left\{\bs M^{1/2}\bs F^{-1}\bs M^{1/2}:\bs F\in\mathbb{S}_{++}^{d\times d},\left\Vert\bs F\right\Vert_*\leq 1/\pi^2\right\}$. By applying Lemma~\ref{lemma:van-trees}, we have
    \begin{equation}
    \begin{aligned}               
        \inf_{\hat{\bs w}}\sup_{\tilde{P} \in \mathcal{P}(W,\bs S,\bs T)}\bbE_{\tilde{P}^{\otimes n}\times P_\xi}\left\Vert\hat{\bs w}-\bs w^*\right\Vert_{\bs T}^2\geq&\sup_{\substack{\bs F\succeq\bs O\\ \left\Vert\bs F\right\Vert_*\leq 1/\pi^2}}\left\langle\bs T,\left(\bs M^{1/2}\bs F^{-1}\bs M^{1/2}+\frac{n\bs S}{\sigma^2}\right)^{-1}\right\rangle \\
        =&\sup_{\substack{\bs F\succeq\bs O\\ \left\Vert\bs F\right\Vert_*\leq 1/\pi^2}}\left\langle\bs T',\left(\bs F^{-1}+\frac{n\bs S'}{\sigma^2}\right)^{-1}\right\rangle.
    \end{aligned}
    \end{equation}
    This completes the proof.
\end{proof}

% \subsection{Technical Lemmas}

\section{Proof of Optimal Algorithm in Section~\ref{sec: main optimal alg}}
\subsection{Proof of Theorem~\ref{thm:optimal-upper-bound}}
\begin{proof}[Proof of Theorem~\ref{thm:optimal-upper-bound}]
    Let $\hat{\bs w}=\frac{1}{n}\bs S^{-1}\sum_{i=1}^n\bs x_i y_i$. Then we have $\hat{\bs w}^\opt=\bs A\hat{\bs w}$. We first show that
    \begin{equation}
        \bbE\hat{\bs w}=\bs w^*, \quad\cov\hat{\bs w}\preceq\frac{2\sigma^2+2\psi\left\Vert\bs w^*\right\Vert_{\bs S}^2}{n}\bs S^{-1}.
    \end{equation}
    Denote $\epsilon_i=y_i-\bs x_i^\top\bs w^*$ as the response noise. Since $\bs w^*$ is an optimal parameter, we have $\bbE\epsilon_i\bs x_i=0$. Recall that $\bs S=\bbE\bs x\bs x^\top$, and $\left\{(\bs x_i,y_i)\right\}_{i=1}^{n}$ are i.i.d, we have
    \begin{equation}
        \bbE\hat{\bs w}=\frac{1}{n}\bs S^{-1}\sum_{i=1}^n\bbE\left[\bs x_i\left(\bs x_i^\top\bs w^*+\epsilon_i\right)\right]=\bs S^{-1}\bbE_{P_{\bs x}}\left[\bs x\bs x^\top\right]\bs w^*=\bs w^*.
    \end{equation}
    Furthermore,
    \begin{equation}
    \begin{aligned}        
        \cov\hat{\bs w}\stackrel{a}{\preceq}&\frac{1}{n}\bs S^{-1}\bbE_{P_{\bs x\times y}}\left[y^2\bs x\bs x^\top\right]\bs S^{-1}=\frac{1}{n}\bs S^{-1}\bbE_{P_{\bs x\times y}}\left[\left(\bs x^\top\bs w^*+\epsilon\right)^2\bs x\bs x^\top\right]\bs S^{-1} \\
        \stackrel{b}{\preceq}&\frac{2}{n}\bs S^{-1}\left(\bbE_{P_{\bs x}}\left[\left(\bs x^\top\bs w^*\right)^2\bs x\bs x^\top\right]+\bbE_{P_{\bs x\times y}}\left[\epsilon^2\bs x\bs x^\top\right]\right)\bs S^{-1} \\
        \stackrel{c}{\preceq}&\frac{2}{n}\bs S^{-1}\left(\psi\left\Vert\bs w^*\right\Vert_{\bs S}^2\bs S+\sigma^2\bs S\right)\bs S^{-1} \\
        =&\frac{2\sigma^2+2\psi\left\Vert\bs w^*\right\Vert_{\bs S}^2}{n}\bs S^{-1},
    \end{aligned}
    \end{equation}
    where $\stackrel{a}{=}$ applies $\cov\hat{\bs w}\preceq\bbE\left[\hat{\bs w}\hat{\bs w}^\top\right]$ and $\hat{\bs w}$ is the average of $n$ independent random variable, $\stackrel{b}{\preceq}$ uses the inequality $(a+b)^2\leq 2a^2+2b^2$, and $\stackrel{c}{\preceq}$ uses Assumption~\ref{assumption:fourth-moment} and Assumption~\ref{assumption:noise}.

    Since $\hat{\bs w}^\opt=\bs M^{-1/2}\bs A\bs M^{1/2}\hat{\bs w}$, we have
    \begin{equation}
        \bbE\hat{\bs w}^\opt=\bs M^{-1/2}\bs A\bs M^{1/2}\bs w^*, 
    \end{equation}
    \begin{equation}            
        \cov\hat{\bs w}^\opt\preceq\frac{2\sigma^2+2\psi\left\Vert\bs w^*\right\Vert_{\bs S}^2}{n}\left(\bs M^{-1/2}\bs A\bs M^{1/2}\right)\bs S^{-1}\left(\bs M^{-1/2}\bs A\bs M^{1/2}\right)^\top.
    \end{equation}
    Apply the bias-variance decomposition to $\bbE\left\Vert\hat{\bs w}^\opt-\bs w^*\right\Vert_{\bs T}^2$, we obtain
    \begin{equation}
        \bbE\left\Vert\hat{\bs w}^\opt-\bs w^*\right\Vert_{\bs T}^2 =\left\Vert\bbE\hat{\bs w}^\opt-\bs w^*\right\Vert_{\bs T}^2+\left\langle\bs T,\cov\hat{\bs w}^\opt\right\rangle.
    \end{equation}
    Recall that $\bs S'=\bs M^{-1/2}\bs S\bs M^{-1/2}$ and $\bs T'=\bs M^{-1/2}\bs T\bs M^{-1/2}$. For the bias term, we have
    \begin{equation}
    \begin{aligned}
        \left\Vert\bbE\hat{\bs w}^\opt-\bs w^*\right\Vert_{\bs T}^2=&\left\Vert\left(\bs I-\bs M^{-1/2}\bs A\bs M^{1/2}\right)\bs w^*\right\Vert_{\bs T}^2=\left\Vert\bs M^{-1/2}\left(\bs I-\bs A\right)\bs M^{1/2}\bs w^*\right\Vert_{\bs T}^2 \\
        =&\left\Vert(\bs I-\bs A)\bs M^{1/2}\bs w^*\right\Vert_{\bs T'}^2.
    \end{aligned}
    \end{equation}
    For the variance term, we have
    \begin{equation}
    \begin{aligned}
        \left\langle\bs T,\cov\hat{\bs w}^\opt\right\rangle\leq&\frac{2\sigma^2+2\psi\left\Vert\bs w^*\right\Vert_{\bs S}^2}{n}\left\langle\bs M^{-1/2}\bs T\bs M^{-1/2},\bs A\left(\bs M^{-1/2}\bs S\bs M^{-1/2}\right)^{-1}\bs A^\top\right\rangle \\
        \leq&\frac{2\sigma^2+2\psi\left\Vert\bs M^{1/2}\bs w^*\right\Vert_{\bs S'}^2}{n}\left\langle\bs T',\bs A\left(\bs S'\right)^{-1}\bs A^\top\right\rangle.
    \end{aligned}
    \end{equation}
    % \begin{equation}
    % \begin{aligned}        
    %     &\bbE\left\Vert\hat{\bs w}^\opt-\bs w^*\right\Vert_{\bs T}^2 =\left\Vert\bbE\hat{\bs w}^\opt-\bs w^*\right\Vert_{\bs T}^2+\left\langle\bs T,\cov\hat{\bs w}^\opt\right\rangle \\
    %     \leq&\left\Vert(\bs I-\bs M^{-1/2}\bs A\bs M^{1/2})\bs w^*\right\Vert_{\bs T}^2 \\
    %     &+\frac{2\sigma^2+2\psi\left\Vert\bs w^*\right\Vert_{\bs S}^2}{n}\left\langle\bs M^{-1/2}\bs T\bs M^{-1/2},\bs A\left(\bs M^{-1/2}\bs S\bs M^{-1/2}\right)^{-1}\bs A^\top\right\rangle \\
    %     =&\left\Vert(\bs I-\bs A)\bs M^{1/2}\bs w^*\right\Vert_{\bs T'}^2+\frac{2\sigma^2+2\psi\left\Vert\bs M^{1/2}\bs w^*\right\Vert_{\bs S'}^2}{n}\left\langle\bs T',\bs A\left(\bs S'\right)^{-1}\bs A^\top\right\rangle,
    % \end{aligned}
    % \end{equation}
    Take the supremum with respect to $\bs w^*\in W=\left\{\bs w^*\in\bbR^d:\left\Vert\bs w^*\right\Vert_{\bs M}^2\leq 1\right\}$, and note that 
    \begin{equation}
        \sup_{\bs w^*\in W}\left\Vert(\bs I-\bs A)\bs M^{1/2}\bs w^*\right\Vert_{\bs T'}^2=\left\Vert(\bs I-\bs A)^\top\bs T'(\bs I-\bs A)\right\Vert, \quad\sup_{\bs w^*\in W}\left\Vert\bs M^{1/2}\bs w^*\right\Vert_{\bs S'}^2=\left\Vert\bs S'\right\Vert.
    \end{equation}
    Thus, we obtain
    \begin{equation}
        \sup_{\tilde{P}\in\mathcal{P}(W,\bs S,\bs T)}\bbE\left\Vert\hat{\bs w}^\opt-\bs w^*\right\Vert_{\bs T}^2\leq\left\Vert(\bs I-\bs A)^\top\bs T'(\bs I-\bs A)\right\Vert+\frac{2\sigma^2+2\psi\left\Vert\bs S'\right\Vert}{n}\left\langle\bs T',\bs A\left(\bs S'\right)^{-1}\bs A^\top\right\rangle.
    \end{equation}
    Minimizing the RHS with respect to $\bs A$ completes the proof.
\end{proof}

\subsection{Proof of Theorem~\ref{thm:lower-upper-match}}
The proof of Theorem~\ref{thm:lower-upper-match} is divided into two parts. First, we assume $\bs T'$ is invertible, and solves the optimization problem in Theorem~\ref{thm:lower-bound} to derive the result. For the second part, we replace $\bs T'$ by $\bs T'+\epsilon\bs I$, which is invertible, and take $\epsilon\to 0$ to complete the proof.
\begin{proof}[Proof of Theorem~\ref{thm:lower-upper-match}]
    For simplicity, let
    \begin{align}
        L(\bs S',\bs T')=&\sup_{\substack{\bs F\succeq\bs O\\ \left\Vert\bs F\right\Vert_*\leq 1/\pi^2}}\left\langle\bs T',\left(\bs F^{-1}+\frac{n\bs S'}{\sigma^2}\right)^{-1}\right\rangle, \\
        % U(\bs S',\bs T')=&\inf_{\bs S'(\bs I-\bs A)\in\bbS_+^{d\times d}}\frac{1}{\pi^2}\left\Vert(\bs I-\bs A)^\top\bs T'(\bs I-\bs A)\right\Vert+\frac{\sigma^2}{n}\left\langle\bs T',\bs A\left(\bs S'\right)^{-1}\bs A^\top\right\rangle.        
        U(\bs S',\bs T')=&\inf_{\bs A\in\bbR^{d\times d}}\frac{1}{\pi^2}\left\Vert(\bs I-\bs A)^\top\bs T'(\bs I-\bs A)\right\Vert+\frac{\sigma^2}{n}\left\langle\bs T',\bs A\left(\bs S'\right)^{-1}\bs A^\top\right\rangle.\label{eq:upper-lower-match-def-U}
    \end{align}
    For the first part of the proof, we assume $\bs T'$ is invertible. We solve the optimization problem in Theorem~\ref{thm:lower-bound}. Note that the objective function $\left\langle\bs T',\left(\bs F^{-1}+\frac{n\bs S'}{\sigma^2}\right)^{-1}\right\rangle$ is concave with respect to $\bs F$ and the feasible set $\left\{\bs F\in\mathbb{S}_{++}^{d\times d}:\left\Vert\bs F\right\Vert_*\leq 1/\pi^2\right\}$ is a convex set. Therefore, we can introduce a Lagrange multiplier $\bs\Delta\in\bbS^{d\times d}$ and obtain
    \begin{equation}
    \begin{aligned}
        % &\sup_{\substack{\bs F\succeq\bs O\\ \left\Vert\bs F\right\Vert_*\leq 1/\pi^2}}\left\langle\bs T',\left(\bs F^{-1}+\frac{n\bs S'}{\sigma^2}\right)^{-1}\right\rangle \\
        L(\bs S',\bs T')=&\sup_{\substack{\bs F\in\bbS^{d\times d}\\\bs B\succeq\bs O\\ \left\Vert\bs B\right\Vert_*\leq 1/\pi^2}}\inf_{\bs\Delta\in\bbS^{d\times d}}\left\langle\bs T',\left(\bs F^{-1}+\frac{n\bs S'}{\sigma^2}\right)^{-1}\right\rangle+\left\langle\bs S'\left(\bs T'\right)^{-1/2}\bs\Delta\left(\bs T'\right)^{-1/2}\bs S',\bs B-\bs F\right\rangle \\
        \stackrel{a}{=}&\inf_{\bs\Delta\in\bbS^{d\times d}}\sup_{\bs F\in\bbS^{d\times d}}\left[\left\langle\bs T',\left(\bs F^{-1}+\frac{n\bs S'}{\sigma^2}\right)^{-1}\right\rangle-\left\langle\bs S'\left(\bs T'\right)^{-1/2}\bs\Delta\left(\bs T'\right)^{-1/2}\bs S',\bs F\right\rangle\right]\\
        &\phantom{\inf_{\bs\Delta\in\bbS^{d\times d}}}+\sup_{\substack{\\\bs B\succeq\bs O\\ \left\Vert\bs B\right\Vert_*\leq 1/\pi^2}}\left\langle\bs S'\left(\bs T'\right)^{-1/2}\bs\Delta\left(\bs T'\right)^{-1/2}\bs S',\bs B\right\rangle \\
        \stackrel{b}{=}&\inf_{\bs\Delta\in\bbS^{d\times d}}\underbrace{\sup_{\bs F\in\bbS^{d\times d}}\left[\left\langle\bs T',\left(\bs F^{-1}+\frac{n\bs S'}{\sigma^2}\right)^{-1}\right\rangle-\left\langle\bs S'\left(\bs T'\right)^{-1/2}\bs\Delta\left(\bs T'\right)^{-1/2}\bs S',\bs F\right\rangle\right]}_{(a)} \\
        &\phantom{\inf_{\bs\Delta\in\bbS^{d\times d}}}+\frac{1}{\pi^2}\left\Vert\bs S'\left(\bs T'\right)^{-1/2}\bs\Delta\left(\bs T'\right)^{-1/2}\bs S'\right\Vert,
    \end{aligned}
    \end{equation}
    where $\stackrel{a}{=}$ follows from the concavity of $\left\langle\bs T',\left(\bs F^{-1}+\frac{n\bs S'}{\sigma^2}\right)^{-1}\right\rangle$ with respect to $\bs F$ and the convexity of feasible set $\left\{\bs F\in\mathbb{S}_{++}^{d\times d}:\left\Vert\bs F\right\Vert_*\leq 1/\pi^2\right\}$, and $\stackrel{b}{=}$ follows from the fact that the dual norm of nuclear norm $\Vert\cdot\Vert_*$ is 2-norm $\Vert\cdot\Vert$.
    To solve $(a)$, let the derivative of $(a)$ with respect to $\bs F$ be equal to $\bs O$, we get
    \begin{equation}
        \left(\bs I+\frac{n\bs S'\bs F}{\sigma^2}\right)^{-1}\bs T' \left(\bs I+\frac{n\bs F\bs S'}{\sigma^2}\right)^{-1}-\bs S'\left(\bs T'\right)^{-1/2}\bs\Delta\left(\bs T'\right)^{-1/2}\bs S'=\bs O.
    \end{equation}
    Note that if $\bs\Delta$ is not a PSD matrix, then $(a)=+\infty$. Thus, $\bs\Delta\succeq\bs O$. Let $\left(\bs I+\frac{n\bs S'\bs F}{\sigma^2}\right)^{-1}\left(\bs T'\right)^{1/2}=\bs S'\left(\bs T'\right)^{-1/2}\bs\Delta^{1/2}$. Solve the equation yields
    \begin{equation}
        \bs F=\frac{\sigma^2}{n}\left[\left(\bs S'\right)^{-1}\left(\bs T'\right)^{1/2}\bs\Delta^{-1/2}\left(\bs T'\right)^{1/2}\left(\bs S'\right)^{-1}-\left(\bs S'\right)^{-1}\right],
    \end{equation}
    which meets the requirement that $\bs F$ is a PSD matrix. Plugging the solution into $(a)$, we have
    \begin{equation}
    \begin{aligned}
        (a)=&\frac{\sigma^2}{n}\left[\left\langle\bs S',\left(\bs T'\right)^{-1/2}\bs\Delta\left(\bs T'\right)^{-1/2}\right\rangle-2\tr\bs\Delta^{1/2}+\tr\bs T'\left(\bs S'\right)^{-1}\right] \\
        =&\frac{\sigma^2}{n}\left\langle\bs T', \left[\left(\bs T'\right)^{-1/2}\bs\Delta^{1/2}\left(\bs T'\right)^{-1/2}\bs S'-\bs I\right]\left(\bs S'\right)^{-1}\left[\bs S'\left(\bs T'\right)^{-1/2}\bs\Delta^{1/2}\left(\bs T'\right)^{-1/2}-\bs I\right]\right\rangle.\notag
    \end{aligned}
    \end{equation}
    Let $\bs A=\bs I-\left(\bs T'\right)^{-1/2}\bs\Delta^{1/2}\left(\bs T'\right)^{-1/2}\bs S'$, we obtain
    % \begin{equation}
    % \begin{aligned}
    %     &\sup_{\substack{\bs F\succeq\bs O\\ \left\Vert\bs F\right\Vert_*\leq 1/\pi^2}}\left\langle\bs T',\left(\bs F^{-1}+\frac{n\bs S'}{\sigma^2}\right)^{-1}\right\rangle \\
    %     =&\inf_{\bs S'(\bs I-\bs A)\in\bbS_+^{d\times d}}\frac{1}{\pi^2}\left\Vert(\bs I-\bs A)^\top\bs T'(\bs I-\bs A)\right\Vert+\frac{\sigma^2}{n}\left\langle\bs T',\bs A\left(\bs S'\right)^{-1}\bs A^\top\right\rangle.\label{eq:upper-lower-bound}
    % \end{aligned}
    % \end{equation}
    \begin{equation}
        L(\bs S',\bs T')
        =\inf_{\bs S'(\bs I-\bs A)\in\bbS_+^{d\times d}}\frac{1}{\pi^2}\left\Vert(\bs I-\bs A)^\top\bs T'(\bs I-\bs A)\right\Vert+\frac{\sigma^2}{n}\left\langle\bs T',\bs A\left(\bs S'\right)^{-1}\bs A^\top\right\rangle.\label{eq:upper-lower-bound}
    \end{equation}
    
    Note that the definition of $U(\bs S',\bs T')$ in \eqref{eq:upper-lower-match-def-U} imposes no constraint $\bs A$. Now we show that the constraint $\bs S'(\bs I-\bs A)\in\bbS_+^{d\times d}$ in \eqref{eq:upper-lower-bound} can be relaxed to $\bs A \in \mathbb{R}^{d\times d}$. For any $\bs A\in\bbR^{d\times d}$, we denote the polar decomposition of $\left(\bs T'\right)^{1/2}(\bs I-\bs A)\left(\bs S'\right)^{-1}\left(\bs T'\right)^{1/2}$ as: 
    \begin{equation}
        \bs U\bs Z=\left(\bs T'\right)^{1/2}(\bs I-\bs A)\left(\bs S'\right)^{-1}\left(\bs T'\right)^{1/2}\label{eq:lower-upper-match-A-reparam},
    \end{equation}
    where $\bs U$ is an orthogonal matrix and $\bs Z$ is a PSD matrix. Substitute \eqref{eq:lower-upper-match-A-reparam} into the objective function of $U(\bs S',\bs T')$ shown in~\eqref{eq:lower-bound-param-A}, we have
    \begin{align}        
        &\frac{\sigma^2}{n}\left\langle\bs T',\bs A\left(\bs S'\right)^{-1}\bs A^\top\right\rangle+\frac{1}{\pi^2}\left\Vert(\bs I-\bs A)^\top\bs T'(\bs I-\bs A)\right\Vert\label{eq:lower-bound-param-A} \\
        \stackrel{a}{=}&\frac{\sigma^2}{n}\left\langle\bs T',\left(\bs I-\left(\bs T'\right)^{-1/2}\bs U\bs Z\left(\bs T'\right)^{-1/2}\bs S'\right)\left(\bs S'\right)^{-1}\left(\bs I-\bs S'\left(\bs T'\right)^{-1/2}\bs Z^\top\bs U^\top\left(\bs T'\right)^{-1/2}\right)\right\rangle\notag \\
        &+\frac{1}{\pi^2}\left\Vert\bs S'\left(\bs T'\right)^{-1/2}\bs Z^\top\bs U^\top\bs U\bs Z\left(\bs T'\right)^{-1/2}\bs S'\right\Vert\notag \\
        =&\frac{\sigma^2}{n}\left[\tr\left(\bs U\bs Z\left(\bs T'\right)^{-1/2}\bs S'\left(\bs T'\right)^{-1/2}\bs Z^\top\bs U^\top\right)-\tr\left(\bs U\bs Z+\bs Z^\top\bs U^\top\right)+\tr\left(\bs T'\left(\bs S'\right)^{-1}\right)\right]\notag \\
        &+\frac{1}{\pi^2}\left\Vert\bs S'\left(\bs T'\right)^{-1/2}\bs Z^\top\bs U^\top\bs U\bs Z\left(\bs T'\right)^{-1/2}\bs S'\right\Vert\notag \\
        \stackrel{a}{=}&\frac{\sigma^2}{n}\left[\left\langle\bs S',\left(\bs T'\right)^{-1/2}\bs Z^2\left(\bs T'\right)^{-1/2}\right\rangle-\tr\left(\bs U\bs Z+\bs Z^\top\bs U^\top\right)+\tr\left(\bs T'\left(\bs S'\right)^{-1}\right)\right]\label{eq:lower-bound-param-Z} \\
        &+\frac{1}{\pi^2}\left\Vert\bs S'\left(\bs T'\right)^{-1/2}\bs Z^2\left(\bs T'\right)^{-1/2}\bs S'\right\Vert\notag,
    \end{align}
    where $\stackrel{a}{=}$ uses $\bs A=\bs I-\left(\bs T'\right)^{-1/2}\bs U\bs Z\left(\bs T'\right)^{-1/2}\bs S'$, and $\stackrel{b}{=}$ uses $\bs U\bs U^\top=\bs I$ and $\bs Z$ is a PSD matrix. We first minimize \eqref{eq:lower-bound-param-Z} with respect to $\bs U$. By Lemma~\ref{lemma:sup-tr-orth-psd}, $-\tr\left(\bs U\bs Z+\bs Z^\top\bs U^\top\right)$ is minimized when $\bs U=\bs I$, which implies $\bs S'(\bs I-\bs A)=\bs S'\left(\bs T'\right)^{-1/2}\bs U\bs Z\left(\bs T'\right)^{-1/2}\bs S'\in\bbS_+^{d\times d}$. Therefore, we have
    \begin{equation}
        \inf_{\bs A\in\bbR^{d\times d}}\eqref{eq:lower-bound-param-A}=\inf_{\bs Z\in\bbS_{+}^{d\times d}}\inf_{\substack{\bs U\in\bbR^{d\times d} \\ \bs U\bs U^\top=\bs I}}\eqref{eq:lower-bound-param-Z}=\inf_{\bs S'(\bs I-\bs A)\in\bbS_+^{d\times d}}\eqref{eq:lower-bound-param-A},
    \end{equation}
    We complete the first part by noting that $U(\bs S',\bs T')=\inf_{\bs A\in\bbR^{d\times d}}\eqref{eq:lower-bound-param-A}$ by definition and $L(\bs S',\bs T')=\inf_{\bs S'(\bs I-\bs A)\in\bbS_+^{d\times d}}\eqref{eq:lower-bound-param-A}$ which is shown in~\eqref{eq:upper-lower-bound}.
    % \begin{equation}
    %     U(\bs S',\bs T')=\inf_{\bs A\in\bbR^{d\times d}}\eqref{eq:lower-bound-param-A}=\inf_{\bs Z\in\bbS_{+}^{d\times d}}\inf_{\substack{\bs U\in\bbR^{d\times d} \\ \bs U\bs U^\top=\bs I}}\eqref{eq:lower-bound-param-Z}=\inf_{\bs S'(\bs I-\bs A)\in\bbS_+^{d\times d}}\eqref{eq:lower-bound-param-A}=L(\bs S',\bs T'),
    % \end{equation}
    % which is solving the RHS of \eqref{eq:upper-lower-bound}.
    % Therefore, minimizing with respect to $\bs Z$ is equivalent to minimizing with respect to $\bs A$ with the restriction $\bs S'(\bs I-\bs A)=\bs S'\left(\bs T'\right)^{-1/2}\bs Z\left(\bs T'\right)^{-1/2}\bs S'\in\bbS_+^{d\times d}$, which is solving the RHS of \eqref{eq:upper-lower-bound}.

    For the second part of the proof, we consider the case where $\bs T'$ is any PSD matrix, \textit{i.e.} $\bs T'$ is possibly singular. Let $\epsilon>0$ be arbitrary. Since $L(\bs S',\bs T')$ is linear in $\bs T'$, we have
    % \begin{equation}
    % \begin{aligned}        
    %     &\sup_{\substack{\bs F\succeq\bs O\\ \left\Vert\bs F\right\Vert_*\leq 1/\pi^2}}\left\langle\bs T',\left(\bs F^{-1}+\frac{n\bs S'}{\sigma^2}\right)^{-1}\right\rangle\leq\sup_{\substack{\bs F\succeq\bs O\\ \left\Vert\bs F\right\Vert_*\leq 1/\pi^2}}\left\langle\bs T'+\epsilon\bs I,\left(\bs F^{-1}+\frac{n\bs S'}{\sigma^2}\right)^{-1}\right\rangle\\
    %     \leq&\sup_{\substack{\bs F\succeq\bs O\\ \left\Vert\bs F\right\Vert_*\leq 1/\pi^2}}\left\langle\bs T',\left(\bs F^{-1}+\frac{n\bs S'}{\sigma^2}\right)^{-1}\right\rangle+\epsilon\sup_{\substack{\bs F\succeq\bs O\\ \left\Vert\bs F\right\Vert_*\leq 1/\pi^2}}\left\langle\bs I,\left(\bs F^{-1}+\frac{n\bs S'}{\sigma^2}\right)^{-1}\right\rangle.
    % \end{aligned}
    % \end{equation}
    \begin{equation}
        L(\bs S',\bs T')\leq L(\bs S',\bs T'+\epsilon\bs I)\leq L(\bs S',\bs T')+\epsilon L(\bs S',\bs I)
    \end{equation}
    Note that
    \begin{equation}
        \sup_{\substack{\bs F\succeq\bs O\\ \left\Vert\bs F\right\Vert_*\leq 1/\pi^2}}\left\langle\bs I,\left(\bs F^{-1}+\frac{n\bs S'}{\sigma^2}\right)^{-1}\right\rangle\leq\sup_{\substack{\bs F\succeq\bs O\\ \left\Vert\bs F\right\Vert_*\leq 1/\pi^2}}\tr\bs F\leq \frac{1}{\pi^2}.
    \end{equation}
    Therefore, we have
    \begin{equation}
        L(\bs S',\bs T')
        =\lim_{\epsilon\to0^+}L(\bs S',\bs T'+\epsilon\bs I).\label{eq:lower-lim-eq}
    \end{equation}
    Similarly, we have    
    \begin{equation}
    \begin{aligned}
        &\inf_{\bs A\in\bbR^{d\times d}}\frac{1}{\pi^2}\left\Vert(\bs I-\bs A)^\top\bs T'(\bs I-\bs A)\right\Vert+\frac{\sigma^2}{n}\left\langle\bs T',\bs A\left(\bs S'\right)^{-1}\bs A^\top\right\rangle \\
        \leq&\inf_{\bs A\in\bbR^{d\times d}}\frac{1}{\pi^2}\left\Vert(\bs I-\bs A)^\top\bs (\bs T'+\epsilon\bs I)(\bs I-\bs A)\right\Vert+\frac{\sigma^2}{n}\left\langle\bs T'+\epsilon\bs I,\bs A\left(\bs S'\right)^{-1}\bs A^\top\right\rangle \\
        \leq&\frac{1}{\pi^2}\left\Vert(\bs I-\bs A_0)^\top\bs T'(\bs I-\bs A_0)\right\Vert+\frac{\sigma^2}{n}\left\langle\bs T',\bs A_0\left(\bs S'\right)^{-1}\bs A_0^\top\right\rangle \\
        &+\epsilon\left[\frac{1}{\pi^2}\left\Vert(\bs I-\bs A_0)^\top(\bs I-\bs A_0)\right\Vert+\frac{\sigma^2}{n}\tr\left(\bs A_0\left(\bs S'\right)^{-1}\bs A_0^\top\right)\right], \\
    \end{aligned}
    \end{equation}
    where $\bs A_0$ is a minimizer of $\frac{1}{\pi^2}\left\Vert(\bs I-\bs A)^\top\bs T'(\bs I-\bs A)\right\Vert+\frac{\sigma^2}{n}\left\langle\bs T',\bs A\left(\bs S'\right)^{-1}\bs A^\top\right\rangle$. Thus, we have
    \begin{equation}
        U(\bs S',\bs T')=\lim_{\epsilon\to0^+}U(\bs S',\bs T'+\epsilon\bs I).\label{eq:upper-lim-eq}
    \end{equation}
    Finally, combine \eqref{eq:lower-lim-eq}, \eqref{eq:upper-lim-eq} and the first part of the proof, we obtain
    % \begin{equation}
    % \begin{aligned}        
    %     &\sup_{\substack{\bs F\succeq\bs O\\ \left\Vert\bs F\right\Vert_*\leq 1/\pi^2}}\left\langle\bs T',\left(\bs F^{-1}+\frac{n\bs S'}{\sigma^2}\right)^{-1}\right\rangle
    %     =\lim_{\epsilon\to0^+}\sup_{\substack{\bs F\succeq\bs O\\ \left\Vert\bs F\right\Vert_*\leq 1/\pi^2}}\left\langle\bs T'+\epsilon\bs I,\left(\bs F^{-1}+\frac{n\bs S'}{\sigma^2}\right)^{-1}\right\rangle \\
    %     =&\lim_{\epsilon\to0^+}\inf_{\bs A\in\bbR^{d\times d}}\frac{1}{\pi^2}\left\Vert(\bs I-\bs A)^\top(\bs T'+\epsilon\bs I)(\bs I-\bs A)\right\Vert+\frac{\sigma^2}{n}\left\langle\bs T'+\epsilon\bs I,\bs A\left(\bs S'\right)^{-1}\bs A^\top\right\rangle \\
    %     =&\inf_{\bs A\in\bbR^{d\times d}}\frac{1}{\pi^2}\left\Vert(\bs I-\bs A)^\top\bs T'(\bs I-\bs A)\right\Vert+\frac{\sigma^2}{n}\left\langle\bs T',\bs A\left(\bs S'\right)^{-1}\bs A^\top\right\rangle.
    % \end{aligned}
    % \end{equation}
    \begin{equation}
        L(\bs S',\bs T')
        =\lim_{\epsilon\to0^+}L(\bs S',\bs T'+\epsilon\bs I)=\lim_{\epsilon\to0^+}U(\bs S',\bs T'+\epsilon\bs I)=U(\bs S',\bs T').
    \end{equation}
    This completes the proof for any PSD matrix $\bs T'$.
\end{proof}

\begin{lemma}\label{lemma:sup-tr-orth-psd}
Let $\bs Z$ be a PSD matrix and $\bs U$ be a orthogonal matrix. Then $\tr(\bs U\bs Z)\leq\tr\bs Z$.    
\end{lemma}
\begin{proof}
    Without loss of generality, we assume $\bs Z=\diag\{z_1,z_2\ldots,z_d\}$. Let $u_{ij}$ denote the $ij$-th entry of $\bs U$. Note that $\bs U$ is orthogonal implies $\left\vert u_{ij}\right\vert\leq 1$, so
    \begin{equation}
        \tr(\bs U\bs Z)=\sum_{i=1}^d u_{ii}z_i\leq\sum_{i=1}^d z_i=\tr\bs Z,
    \end{equation}
    where $=$ holds when $\bs U=\bs I$.
\end{proof}

\section{Proof of ASGD Upper Bound in Section~\ref{sec: asgd upper}}
In this section, we provide the analysis of ASGD upper bound. The organization of this section is as follows:
\begin{itemize}
    \item In Section~\ref{sec:asgd-pre}, we present the tools for analyzing ASGD. Bias-variance decomposition is used to decompose the excess risk to the bias part and variance part. The definition of linear operators on matrices allows us to write the matrix form of the iteration of bias and variance.
    \item In Section~\ref{sec:asgd-proof}, we summarize the proof. We begin by defining semi-stochastic versions of variance and bias iteration following~\citet{Dieuleveut2016NonparametricStochasticApproximation}, and further bounds the difference. Some proofs are deferred to Section~\ref{sec:asgd-variance-ub} and Section~\ref{sec:asgd-bias-ub}.
    \item In Section~\ref{sec:prop-momentum-matrix}, we establish the bounds of the momentum matrix. The bounds are based on the spectral radius of the momentum matrix. 
    \item Section~\ref{sec:asgd-variance-ub} and Section~\ref{sec:asgd-bias-ub} provide bounds of semi-stochastic iterations in terms of algorithmic parameters, and covariance matrix of source and target distribution.
\end{itemize}

\subsection{Preliminaries}\label{sec:asgd-pre}
\subsubsection{Bias-Variance Decomposition}
Given a sequence of data $\left\{\left(\bs x_i,y_i\right)\right\}_{i=1}^n$, ASGD starts from initial weight $\bs w_0=\bs v_0$ and recursively calculates
\begin{equation}
    \bs u_{t-1}\gets\alpha\bs w_{t-1}+(1-\alpha)\bs v_{t-1},
\end{equation}
\begin{equation}
    \bs w_t\gets\bs u_{t-1}-\delta_t\left(\bs x_t^\top\bs u_{t-1}-y_t\right)\bs x_t,
\end{equation}
\begin{equation}
    \bs v_t\gets\beta \bs u_{t-1}+(1-\beta)\bs v_{t-1}-\gamma_t\left(\bs x_t^\top\bs u_{t-1}-y_t\right)\bs x_t,\label{eq:vt-iteration}
\end{equation}
where $\delta_t$ and $\gamma_t$ are step sizes at iteration $t$. We consider the exponentially decaying step-size schedule
\begin{equation}
    \delta_t=\delta/4^{\ell-1},\gamma_t=\gamma/4^{\ell-1}, \text{if } K(\ell-1)+1\leq t\leq K\ell,\label{eq:step-size}
\end{equation}
where $n$ is the number of observations and $K=n/\log_2 n$. For theoretical analysis, we define $\bs\eta_t=\begin{bmatrix} \bs w_t-\bs w^* \\ \bs u_t-\bs w^* \end{bmatrix}$, where $\bs w^*$ is the ground-truth weight. Let $c=\alpha(1-\beta)$, $q=\alpha\delta+(1-\alpha)\gamma$, and $q_t=\alpha\delta_t+(1-\alpha)\gamma_t$, by eliminating $\bs v_t$ in \eqref{eq:vt-iteration}, ASGD iteration can be written in the following compact form,
\begin{equation}
    \bs\eta_t=\widehat{\bs A}_t\bs\eta_{t-1}+\bs\zeta_t, \text{\enspace where } \widehat{\bs A}_t=\begin{bmatrix}
        \bs O & \bs I-\delta_t\bs x_t\bs x_t^\top \\
        -c\bs I & (1+c)\bs I-q_t\bs x_t\bs x_t^\top
    \end{bmatrix}, \bs\zeta_t=\begin{bmatrix}
        \delta_t\epsilon_t\bs x_t \\        
        q_t\epsilon_t\bs x_t
    \end{bmatrix},
\end{equation}
where $\epsilon_t$ is defined in Assumption~\ref{assumption:noise}.

Following the standard bias-variance decomposition technique~\citep{Jain2017MarkovChainTheory, wu2022last,lirisk}, we decompose the iteration $\bs\eta_t$ into the bias component $\bs\eta_t^\mathrm{bias}$ and the variance component $\bs\eta_t^\mathrm{var}$,
\begin{equation}
    \bs\eta_t^\mathrm{bias}=\widehat{\bs A}_t\bs\eta_{t-1}^\mathrm{bias},\quad\bs\eta_0^\mathrm{bias}=\bs\eta_0;
\end{equation}
\begin{equation}
    \bs\eta_t^\mathrm{var}=\widehat{\bs A}_t\bs\eta_{t-1}^\mathrm{var}+\bs\zeta_t,\quad\bs\eta_0^\mathrm{var}=\bs0.
\end{equation}
The decomposition of $\bs\eta_t$ induces the decomposition of excess risk:
\begin{equation}
\begin{aligned}
    \bbE\left\Vert\bs w_n-\bs w^*\right\Vert_{\bs T}^2=&\left\langle\begin{bmatrix}
        \bs T & \bs O \\
        \bs O & \bs O    \end{bmatrix},\bbE\left[\bs\eta_n\bs\eta_n^\top\right]\right\rangle \\
    \leq&2\cdot\underbrace{\left\langle\begin{bmatrix}
        \bs T & \bs O \\
        \bs O & \bs O    \end{bmatrix},\bbE\left[\bs\eta_n^\text{bias}\left(\bs\eta_n^\text{bias}\right)^\top\right]\right\rangle}_\mathrm{Bias} + 2\cdot\underbrace{\left\langle\begin{bmatrix}
        \bs T & \bs O \\
        \bs O & \bs O    \end{bmatrix},\bbE\left[\bs\eta_n^\text{var}\left(\bs\eta_n^\text{var}\right)^\top\right]\right\rangle}_\mathrm{Variance}.\label{eq:bias-variance-decomposition}
\end{aligned}
\end{equation}

\subsubsection{Linear Operators}\label{sec:asgd-operators}
We introduce the following linear operators on matrices to analyze the recursion of $\bbE\left[\bs\eta_n^\text{bias}\left(\bs\eta_n^\text{bias}\right)^\top\right]$ and $\bbE\left[\bs\eta_n^\text{var}\left(\bs\eta_n^\text{var}\right)^\top\right]$.
\begin{equation}
    \mathcal{I}=\bs I\otimes\bs I, \quad
    \cB_t=\bbE\left[\widehat{\bs A}_t\otimes\widehat{\bs A}_t\right].
\end{equation}
% \begin{equation}
%     \mathcal{M}=\bbE\left[\bs x\otimes\bs x\otimes\bs x\otimes\bs x\right],
% \end{equation}
\begin{equation}
\end{equation}
Let $\bs A_t=\bbE\widehat{\bs A}_t$ be the deterministic version of $\widehat{\bs A}_t$, and define
\begin{equation}
    \tilde\cB_t=\bs A_t\otimes\bs A_t.
\end{equation}
We decompose $\widehat{\bs A}_t$ into two components:
\begin{equation}
    \bs V=\begin{bmatrix}
        \bs O & \bs I \\
        -c\bs I & (1+c)\bs I
    \end{bmatrix}, \quad\widehat{\bs G}_t=\begin{bmatrix}
        \bs O & \delta_t\bs x_t\bs x_t^\top \\
        \bs O & q_t\bs x_t\bs x_t^\top \\
    \end{bmatrix}.
\end{equation}
The deterministic version of $\widehat{\bs G}_t$ is defined as
\begin{equation}
    \bs G_t=\bbE\widehat{\bs G}_t.
\end{equation}
Therefore, $\widehat{\bs A}_t=\bs V-\widehat{\bs G}_t$ and $\bs A_t=\bs V-\bs G_t$.

The following lemma provides properties of the linear operators.
\begin{lemma}\label{lemma:op-prop}
    The above operators have the following properties:
    \begin{enumerate}
        \item $\cB_t\preceq\tilde{\cB}_t+\bbE\left[\widehat{\bs G}_t\otimes\widehat{\bs G}_t\right]$.
        
        \item Suppose Assumption~\ref{assumption:fourth-moment} holds. For any PSD matrix $\bs M$, we have
        \begin{equation}
            \bbE\left[\widehat{\bs G}_t\otimes\widehat{\bs G}_t\right]\circ\bs M\preceq\psi\left\langle\begin{bmatrix}
                \bs O & \bs O \\
                \bs O & \bs S
            \end{bmatrix},\bs M\right\rangle\begin{bmatrix}
            \delta_t^2\bs S & \delta_t q_t\bs S \\
            \delta_t q_t\bs S & q_t^2\bs S
            \end{bmatrix}.
        \end{equation}
    \end{enumerate}
\end{lemma}
\begin{proof}
    \begin{enumerate}
        \item From the definiton of $\cB_t$, we have
        \begin{equation}
        \begin{aligned}            
            \cB_t=&\bbE\left[\left(\bs V-\widehat{\bs G}_t\right)\otimes\left(\bs V-\widehat{\bs G}_t\right)\right] \\
            \stackrel{a}{\preceq}&\left(\bs V-\bs G_t\right)\otimes\left(\bs V-\bs G_t\right)
            -\bs G_t\otimes\bs G_t+\bbE\left[\widehat{\bs G}_t\otimes\widehat{\bs G}_t\right] \\
            =&\tilde{\cB}_t+\bbE\left[\widehat{\bs G}_t\otimes\widehat{\bs G}_t\right],
        \end{aligned}
        \end{equation}
        where $\stackrel{a}{\preceq}$ uses $\bbE\left[\bs V\otimes\widehat{\bs G}_t\right]=\bs V\otimes\bs G_t$ and $\stackrel{a}{\preceq}$ uses $\bbE\left[\widehat{\bs G}_t\otimes\bs V\right]=\bs G_t\otimes\bs V$.
        \item Apply the partition of $\widehat{\bs G}_t$ to $\bs M=\begin{bmatrix}
                \bs M_{11} & \bs M_{12} \\
                \bs M_{21} & \bs M_{22}
            \end{bmatrix}$, we have
        \begin{equation}
        \begin{aligned}
            \bbE\left[\widehat{\bs G}_t\otimes\widehat{\bs G}_t\right]\circ\bs M=&\bbE\begin{bmatrix}
                \delta_t^2\bs x\bs x^\top\bs M_{22}\bs x\bs x^\top & \delta_t q_t\bs x\bs x^\top\bs M_{22}\bs x\bs x^\top \\
                \delta_t q_t\bs x\bs x^\top\bs M_{22}\bs x\bs x^\top & q_t^2\bs x\bs x^\top\bs M_{22}\bs x\bs x^\top
            \end{bmatrix} \\
            =&\begin{bmatrix}
                \delta_t^2 & \delta_t q_t \\
                \delta_t q_t & q_t^2
            \end{bmatrix}\odot\bbE\left[\bs x\bs x^\top\bs M_{22}\bs x\bs x^\top\right] \\
            \stackrel{a}{\preceq}&\begin{bmatrix}
                \delta_t^2 & \delta_t q_t \\
                \delta_t q_t & q_t^2
            \end{bmatrix}\odot\left[\psi\left\langle\begin{bmatrix}
                \bs O & \bs O \\
                \bs O & \bs S
            \end{bmatrix},\bs M\right\rangle\bs S\right] \\
            \preceq&\psi\left\langle\begin{bmatrix}
                \bs O & \bs O \\
                \bs O & \bs S
            \end{bmatrix},\bs M\right\rangle\begin{bmatrix}
            \delta_t^2\bs S & \delta_t q_t\bs S \\
            \delta_t q_t\bs S & q_t^2\bs S
            \end{bmatrix},
        \end{aligned}
        \end{equation}
        where $\odot$ denotes Kronecker product, and $\stackrel{a}{\preceq}$ holds for Assumption~\ref{assumption:fourth-moment} and property of Kronecker product, which is, for any PSD matrices $\bs A$, $\bs B\preceq\bs C$, we have $\bs A\odot\bs B\preceq\bs A\odot\bs C$.
    \end{enumerate}
\end{proof}

\subsubsection{Parameter Choice}\label{sec:param-choice}
We select the parameters of ASGD using the following procedure:
\begin{enumerate}
    \item Select a positive integer $\tilde{\kappa}<d$;
    \item Select auxiliary step sizes $\delta'$ and $\gamma'$ as
    \begin{equation}
        \delta'\leq\frac{1}{\psi\tr\bs S}, \quad\gamma'\in\left[\delta',\frac{1}{\psi\sum_{i>\tilde{\kappa}}{\lambda_i}}\right];
    \end{equation}
    \item Select the total number of ASGD iterations $n$ and momentum parameters $\alpha$ and $\beta$ as
    \begin{equation}
        \beta=\frac{\delta'}{4376\psi\tilde{\kappa}\gamma'\ln n}, \quad\alpha=\frac{1}{1+\beta}, \quad\frac{n\left[1-\alpha(1-\beta)\right]}{\log_2 n\ln n}\geq 16;\label{eq:para-choice-req}
    \end{equation}
    \item Select step sizes $\delta$ and $\gamma$ as
    \begin{equation}
        \delta=\frac{\delta'}{2188\ln n},\quad \gamma=\frac{\gamma'}{2188\ln n}.
    \end{equation}
\end{enumerate}
From the above procedure, we have
\begin{equation}
    \delta\leq\frac{1}{2188\psi\ln n\tr\bs S}, \quad\gamma\in\left[\delta,\frac{1}{2188\psi\ln n\sum_{i>\tilde{\kappa}}{\lambda_i}}\right], \quad\beta=\frac{\delta}{4376\psi\tilde{\kappa}\gamma\ln n}.
\end{equation}

% The following lemma provides properties of the parameter choice.
\begin{lemma}\label{lemma:param-prop}
    Recall that $c=\alpha(1-\beta)$, $q=\alpha\delta+(1-\alpha)\gamma$ and $K=n/\log_2 n$. We have the following properties of the parameter choice.
    \begin{enumerate}
        \item We have
        \begin{equation}
            \dfrac{q-\delta}{1-c}=\dfrac{\gamma-\delta}{2}, \quad\dfrac{q-c\delta}{1-c}=\dfrac{\gamma+\delta}{2}.
        \end{equation}
        % \item $K(1-c)\geq 16\ln n$.
        \item For $i\in [d]$, we have
        \begin{equation}
            \delta\lambda_i\leq\dfrac{1}{2188\psi\ln n}\leq1,\quad q\lambda_i\leq 1+c.
        \end{equation}
    \end{enumerate}
\end{lemma}
\begin{proof}
    \begin{enumerate}
        \item Note that $1-c=2(1-\alpha)$. Thus, we have
        \begin{equation}
            \frac{q-\delta}{1-c}=\frac{(1-\alpha)(\gamma-\delta)}{1-c}=\frac{\gamma-\delta}{2}, \quad\frac{q-c\delta}{1-c}=\frac{q-\delta}{1-c}+\delta=\frac{\gamma+\delta}{2}.
        \end{equation}
        \item Since $\lambda_i\leq\tr\bs S$, we have
        \begin{equation}
            \delta\lambda_i\leq\dfrac{\lambda_i}{2188\psi\ln n\tr\bs S}\leq\dfrac{1}{2188\psi\ln n}\leq 1.
        \end{equation}
        Note that $1-\alpha=\alpha\beta$ and $2\alpha=1+c$ , we have
        \begin{equation}
            q=\alpha\delta+(1-\alpha)\gamma=\alpha\delta+\alpha\beta\gamma\leq2\alpha\delta=(1+c)\delta.
        \end{equation}
        Therefore, $q\lambda_i\leq(1+c)\delta\lambda_i\leq 1+c$.
    \end{enumerate}
\end{proof}

\subsection{Proof Outline}\label{sec:asgd-proof}
We express the recursions of $\bbE\left[\bs\eta_t^\text{bias}\left(\bs\eta_t^\text{bias}\right)^\top\right]$ and $\bbE\left[\bs\eta_t^\text{var}\left(\bs\eta_t^\text{var}\right)^\top\right]$ using the operators:
\begin{equation}
    \bbE\left[\bs\eta_t^\text{bias}\left(\bs\eta_t^\text{bias}\right)^\top\right]=\cB_t\circ\bbE\left[\bs\eta_{t-1}^\text{bias}\left(\bs\eta_{t-1}^\text{bias}\right)^\top\right], \quad \bbE\left[\bs\eta_0^\text{bias}\left(\bs\eta_0^\text{bias}\right)^\top\right]=\bs\eta_0\bs\eta_0^\top\label{eq:bias-recursion};
\end{equation}
\begin{equation}
    \bbE\left[\bs\eta_t^\text{var}\left(\bs\eta_t^\text{var}\right)^\top\right]=\cB_t\circ\bbE\left[\bs\eta_{t-1}^\text{var}\left(\bs\eta_{t-1}^\text{var}\right)^\top\right]+\bbE\left[\bs\zeta_t\bs\zeta_t^\top\right], \quad \bbE\left[\bs\eta_0^\text{var}\left(\bs\eta_0^\text{var}\right)^\top\right]=\bs O\label{eq:var-recursion}.
\end{equation}
Then we construct two recursions similar to the above update rule:
\begin{equation}
    \bs B_t=\cB\circ\bs B_{t-1}, \quad\bs B_0=\bs\eta_0\bs\eta_0^\top,
\end{equation}
\begin{equation}
    \bs C_t=\cB_t\circ\bs C_{t-1}+\sigma^2\begin{bmatrix}
        \delta_t^2\bs S & \delta_t q_t\bs S \\
        \delta_t q_t\bs S & q_t^2\bs S
        \end{bmatrix}, \quad\bs C_0=\bs O.
\end{equation}
The following lemma characterizes $\bbE\left[\bs\eta_t^\text{bias}\left(\bs\eta_t^\text{bias}\right)^\top\right]$ and $\bbE\left[\bs\eta_t^\text{var}\left(\bs\eta_t^\text{var}\right)^\top\right]$ by $\bs B_t$ and $\bs C_t$, respectively.
\begin{lemma}\label{lemma:recursion}
    For $0\leq t\leq n$, $\bbE\left[\bs\eta_t^\text{bias}\left(\bs\eta_t^\text{bias}\right)^\top\right]=\bs B_t$. Furthermore, under Assumption~\ref{assumption:noise}, we have $\bbE\left[\bs\eta_t^\text{var}\left(\bs\eta_t^\text{var}\right)^\top\right]\preceq\bs C_t$.
\end{lemma}
\begin{proof}    
    From \eqref{eq:bias-recursion}, the recursion of $\bs B_t$ is identical to the recursion of $\bbE\left[\bs\eta_t^\text{bias}\left(\bs\eta_t^\text{bias}\right)^\top\right]$. This proves the first part of the lemma. For the second part, from \eqref{eq:var-recursion}, we know the conclusion holds for $t=0$. We assume that $\bbE\left[\bs\eta_{t-1}^\text{var}\left(\bs\eta_{t-1}^\text{var}\right)^\top\right]\preceq\bs C_{t-1}$, then
    \begin{equation}
    \begin{aligned}
        \bbE\left[\bs\eta_t^\text{var}\left(\bs\eta_t^\text{var}\right)^\top\right]=&\cB\circ\bbE\left[\bs\eta_{t-1}^\text{var}\left(\bs\eta_{t-1}^\text{var}\right)^\top\right]+\bbE\left[\bs\zeta_t\bs\zeta_t^\top\right] \\
        \preceq&\cB\circ\bs C_{t-1}+\bbE\left[\bs\zeta_t\bs\zeta_t^\top\right] \\
        \stackrel{a}{\preceq}&\cB\circ\bs C_{t-1}+\sigma^2\begin{bmatrix}
        \delta_t^2\bs S & \delta_t q_t\bs S \\
        \delta_t q_t\bs S & q_t^2\bs S
        \end{bmatrix} \\
        =&\bs C_t,
    \end{aligned}
    \end{equation}
    where $\stackrel{a}{\preceq}$ holds because Assumption~\ref{assumption:noise} implies $\bbE\left[\epsilon_t^2\bs x_t\bs x_t^\top\right]\preceq\sigma^2\bs S$, and
    \begin{equation}
    \begin{aligned}        
        \bbE\left[\bs\zeta_t\bs\zeta_t^\top\right]=&\bbE\left[\begin{bmatrix}
            \delta_t^2\epsilon_t^2\bs x_t\bs x_t^\top & \delta_t q_t\epsilon_t^2\bs x_t\bs x_t^\top \\
            \delta_t q_t\epsilon_t^2\bs x_t\bs x_t^\top & q_t^2\epsilon_t^2\bs x_t\bs x_t^\top
        \end{bmatrix}\right]=\begin{bmatrix}
            \delta_t^2 & \delta_t q_t \\
            \delta_t q_t & q_t^2
        \end{bmatrix}\odot\bbE\left[\epsilon_t^2\bs x_t\bs x_t^\top\right] \\
        \stackrel{a}{\preceq}&\begin{bmatrix}
            \delta_t^2 & \delta_t q_t \\
            \delta_t q_t & q_t^2
        \end{bmatrix}\odot\sigma^2\bs S=\sigma^2\begin{bmatrix}
        \delta_t^2\bs S & \delta_t q_t\bs S \\
        \delta_t q_t\bs S & q_t^2\bs S
        \end{bmatrix},
    \end{aligned}
    \end{equation}
    where $\odot$ denotes Kronecker product, and $\stackrel{a}{\preceq}$ holds because for any PSD matrices $\bs A$, $\bs B\preceq\bs C$, we have $\bs A\odot\bs B\preceq\bs A\odot\bs C$.
\end{proof}

With Lemma~\ref{lemma:recursion}, we have $\bbE\left[\bs\eta_n^\text{bias}\left(\bs\eta_n^\text{bias}\right)^\top\right]=\bs B_n$ and $\bbE\left[\bs\eta_n^\text{var}\left(\bs\eta_n^\text{var}\right)^\top\right]\preceq\bs C_n$. Thus,
\begin{equation}
    \mathrm{Bias}\leq\left\langle\tilde{\bs T},\bs B_n\right\rangle, \quad\mathrm{Variance}\leq\left\langle\tilde{\bs T},\bs C_n\right\rangle,
\end{equation}
where $\tilde{\bs T}=\begin{bmatrix}
    \bs T & \bs O \\
    \bs O & \bs O
\end{bmatrix}$.

The main technical challenge to directly bound $\bs B_n$ and $\bs C_n$ originates from the effect of the fourth moment (\textit{i.e.} $\cB\neq\tilde\cB$), which prevents us from analyzing $\bs B_t$ in each eigenspace of $\bs S$. Our proof defines the semi-stochasitc iteration $\tilde{\bs\eta}_t^\text{bias}$ and $\tilde{\bs\eta}_t^\text{var}$ following~\cite{Dieuleveut2016NonparametricStochasticApproximation}. We analyzes two new recursions $\tilde{\bs B}_t$ and $\tilde{\bs C}_t$ induced by $\tilde{\bs\eta}_t^\text{bias}$ and $\tilde{\bs\eta}_t^\text{var}$. For the variance component, we establish a uniform bound on $\tilde{\bs C}_t$ to show that the effect of the fourth moment is actually ``self-governed". Specifically, the fourth moment amplifies the excess risk up to a constant. For the bias component, $\bs B_t$ is decomposed into $\tilde{\bs B}_t$ and a new term $\bs B_t^{(1)}$ which resembles $\bs C_t$. The bound of $\bs B_t^{(1)}$ is established by applying the bound of 
$\bs C_t$.

\subsubsection{Variance Upper Bound}
We start with the construction of $\tilde{\bs\eta}_t$ by replacing $\widehat{\bs A}_t$ by $\bs A_t$:
\begin{equation}
    \tilde{\bs\eta}_t^\mathrm{var}=\bs A_t\tilde{\bs\eta}_{t-1}^\mathrm{var}+\bs\zeta_t,\quad\tilde{\bs\eta}_0^\mathrm{var}=\bs 0.
\end{equation}
From this definition, we have $\bbE\left[\tilde{\bs\eta}_0^\text{var}\left(\tilde{\bs\eta}_0^\text{var}\right)^\top\right]=\bs O$ and
\begin{equation}
\begin{aligned}
    \bbE\left[\tilde{\bs\eta}_t^\text{var}\left(\tilde{\bs\eta}_t^\text{var}\right)^\top\right]=\tilde{\cB}_t\circ\bbE\left[\tilde{\bs\eta}_{t-1}^\text{var}\left(\tilde{\bs\eta}_{t-1}^\text{var}\right)^\top\right]+\bbE\left[\bs\zeta_t\bs\zeta_t^\top\right] \\
    \preceq\tilde{\cB}_t\circ\bbE\left[\tilde{\bs\eta}_{t-1}^\text{var}\left(\tilde{\bs\eta}_{t-1}^\text{var}\right)^\top\right]+\sigma^2\begin{bmatrix}
        \delta_t^2\bs S & \delta_t q_t\bs S \\
        \delta_t q_t\bs S & q_t^2\bs S
        \end{bmatrix}.
\end{aligned}
\end{equation}
Therefore, we define $\tilde{\bs C}_t$ as
\begin{equation}
    \tilde{\bs C}_t=\tilde\cB_t\circ\tilde{\bs C}_{t-1}+\sigma^2\begin{bmatrix}
        \delta_t^2\bs S & \delta_t q_t\bs S \\
        \delta_t q_t\bs S & q_t^2\bs S
        \end{bmatrix}, \quad\tilde{\bs C}_0=\bs O.\label{eq:def-tilde-var}
\end{equation}
By induction, we have $\bbE\left[\tilde{\bs\eta}_t^\text{var}\left(\tilde{\bs\eta}_t^\text{var}\right)^\top\right]\preceq\tilde{\bs C}_t$.

The following lemma characterizes $\left\langle\tilde{\bs T},\tilde{\bs C}_n\right\rangle$, which is the first step of our proof.
\begin{lemma}\label{lemma:tilde-var-ub}
    We have
    \begin{equation}
        \left\langle\tilde{\bs T},\tilde{\bs C}_n\right\rangle\leq\sigma^2\left[\sum_{i=1}^{k^*}{\frac{t_{ii}}{2K\lambda_i}}+\frac{128}{15}K\left(\frac{q-c\delta}{1-c}\right)^2\sum_{i=k^*+1}^{d}\lambda_i t_{ii}\right].
    \end{equation}
\end{lemma}

The second step is to understand the effect of the fourth moment on the variance component. We first construct an auxiliary recursion $\bs C_t^{(1)}$ as
\begin{equation}
    \bs C_t^{(1)}=\cB_t\circ\bs C_{t-1}^{(1)}+\bbE\left[\widehat{\bs G}_t\otimes\widehat{\bs G}_t\right]\circ\tilde{\bs C}_{t-1}, \quad\bs C_0^{(1)}=\bs O.
\end{equation}
The following lemma bounds $\bs C_t$ from above.
\begin{lemma}\label{lemma:var-ub-abs}
    For $0\leq t\leq n$, we have $\bs C_t\preceq\tilde{\bs C}_t+\bs C_t^{(1)}$.
\end{lemma}
\begin{proof}
    We prove the conclusion by induction. By definition, we have $\bs C_0=\tilde{\bs C}_0=\bs C_t^{(1)}=\bs O$. Therefore, the conclusion holds for $t=0$. We assume $\bs C_{t-1}\preceq\tilde{\bs C}_{t-1}+\bs C_{t-1}^{(1)}$. Note that
    \begin{equation}
    \begin{aligned}        
        \bs C_t=&\cB_t\circ\bs C_{t-1}+\sigma^2\begin{bmatrix}
        \delta_t^2\bs S & \delta_t q_t\bs S \\
        \delta_t q_t\bs S & q_t^2\bs S
        \end{bmatrix} \\
        \preceq&\cB_t\circ\left(\tilde{\bs C}_{t-1}+\bs C_{t-1}^{(1)}\right)+\sigma^2\begin{bmatrix}
        \delta_t^2\bs S & \delta_t q_t\bs S \\
        \delta_t q_t\bs S & q_t^2\bs S
        \end{bmatrix} \\
        =&\cB_t\circ\tilde{\bs C}_{t-1}+\cB_t\circ\bs C_{t-1}^{(1)}+\sigma^2\begin{bmatrix}
        \delta_t^2\bs S & \delta_t q_t\bs S \\
        \delta_t q_t\bs S & q_t^2\bs S
        \end{bmatrix} \\
        \stackrel{a}{\preceq}&\tilde{\cB}_t\circ\tilde{\bs C}_{t-1}+\bbE\left[\widehat{\bs G}_t\otimes\widehat{\bs G}_t\right]\circ\tilde{\bs C}_{t-1}+\cB_t\circ\bs C_{t-1}^{(1)}+\sigma^2\begin{bmatrix}
        \delta_t^2\bs S & \delta_t q_t\bs S \\
        \delta_t q_t\bs S & q_t^2\bs S
        \end{bmatrix} \\
        =&\tilde{\cB}_t\circ\tilde{\bs C}_{t-1}+\sigma^2\begin{bmatrix}
        \delta_t^2\bs S & \delta_t q_t\bs S \\
        \delta_t q_t\bs S & q_t^2\bs S
        \end{bmatrix}+\cB_t\circ\bs C_{t-1}^{(1)}+\bbE\left[\widehat{\bs G}_t\otimes\widehat{\bs G}_t\right]\circ\tilde{\bs C}_{t-1} \\
        =&\tilde{\bs C}_t+\bs C_{t-1}^{(1)},
    \end{aligned}
    \end{equation}
    where $\stackrel{a}{\preceq}$ uses $\cB_t\preceq\tilde{\cB}_t+\bbE\left[\widehat{\bs G}_t\otimes\widehat{\bs G}_t\right]$ in Lemma~\ref{lemma:op-prop}.
\end{proof}
The following lemma characterizes the noise term $\bbE\left[\widehat{\bs G}_t\otimes\widehat{\bs G}_t\right]\circ\tilde{\bs C}_{t-1}$.
\begin{lemma}\label{lemma:tilde-var-uniform-ub}
    Suppose Assumption~\ref{assumption:fourth-moment} holds. Then for $1\leq t\leq n$ we have
    \begin{equation}
        \bbE\left[\widehat{\bs G}_t\otimes\widehat{\bs G}_t\right]\circ\tilde{\bs C}_{t-1}\preceq\frac{1}{2}\sigma^2\begin{bmatrix}
        \delta_t^2\bs S & \delta_t q_t\bs S \\
        \delta_t q_t\bs S & q_t^2\bs S
        \end{bmatrix}.
    \end{equation}
\end{lemma}
Lemma~\ref{lemma:tilde-var-uniform-ub} shows that the noise term in the recursion of $\bs C_t^{(1)}$ is uniformly less than that of $\bs C_t$, Therefore, we can show that $\bs C_t^{(1)}\preceq\frac{1}{2}\bs C_t$ for $0\leq t\leq n$, which is the following lemma.
\begin{lemma}\label{lemma:var-sg}
    Suppose Assumption~\ref{assumption:fourth-moment} holds. Then for $1\leq t\leq n$ we have
    \begin{equation}
        \bs C_t^{(1)}\preceq\frac{1}{2}\bs C_t.
    \end{equation}
\end{lemma}
\begin{proof}
    We proceed by induction. For $t=0$, the conclusion holds by the initial value of $\bs C_t$ and $\bs C_t^{(1)}$. We assume that $\bs C_{t-1}^{(1)}\preceq\frac{1}{2}\bs C_{t-1}$. By Lemma~\ref{lemma:tilde-var-uniform-ub}, we have
    \begin{equation}
    \begin{aligned}
        \bs C_t^{(1)}=&\cB\circ\bs C_{t-1}^{(1)}+\bbE\left[\widehat{\bs G}_t\otimes\widehat{\bs G}_t\right]\circ\tilde{\bs C}_{t-1} \\        
        \preceq&\cB\circ\bs C_{t-1}^{(1)}+\frac{1}{2}\sigma^2\begin{bmatrix}
        \delta_t^2\bs S & \delta_t q_t\bs S \\
        \delta_t q_t\bs S & q_t^2\bs S
        \end{bmatrix} \\        
        \preceq&\cB\circ\left(\frac{1}{2}\bs C_{t-1}\right)+\frac{1}{2}\sigma^2\begin{bmatrix}
        \delta_t^2\bs S & \delta_t q_t\bs S \\
        \delta_t q_t\bs S & q_t^2\bs S
        \end{bmatrix} \\
        =&\frac{1}{2}\left(\bs C_{t-1}+\sigma^2\begin{bmatrix}
        \delta_t^2\bs S & \delta_t q_t\bs S \\
        \delta_t q_t\bs S & q_t^2\bs S
        \end{bmatrix}\right)=\frac{1}{2}\bs C_t.
    \end{aligned}
    \end{equation}
    This completes the proof.
\end{proof}
Finally, we show that $\bs C_t$ is ``self-governed" and obtain the upper bound of variance.
\begin{lemma}\label{lemma:var-ub}
    Suppose Assumptions~\ref{assumption:fourth-moment} and~\ref{assumption:noise} hold. Then we have $\bs C_n\preceq2\tilde{\bs C}_n$ and
    \begin{equation}
    \begin{aligned}
        \mathrm{Variance}\leq\sigma^2\left[\sum_{i=1}^{k^*}{\frac{t_{ii}}{K\lambda_i}}+\frac{256}{15}K\left(\frac{q-c\delta}{1-c}\right)^2\sum_{i=k^*+1}^{d}\lambda_i t_{ii}\right],
    \end{aligned}
    \end{equation}
    where $k^*=\max\left\{k:\lambda_k>\frac{16(1-c)\ln n}{(q-c\delta)K}\right\}$.
\end{lemma}
\begin{proof}
    We apply Lemma~\ref{lemma:var-ub-abs} and Lemma~\ref{lemma:var-sg}. For $0\leq t\leq n$,
    \begin{equation}
        \bs C_t\preceq\tilde{\bs C}_t+\bs C_t^{(1)}\preceq\tilde{\bs C}_t+\frac{1}{2}\bs C_t.
    \end{equation}
    Therefore, $\bs C_n\preceq 2\tilde{\bs C}_n$. By Lemma~\ref{lemma:tilde-var-ub}, taking the inner product with $\tilde{\bs T}$ yields
    \begin{equation}
    \begin{aligned}
        \mathrm{Variance}\leq&\left\langle\tilde{\bs T},\bs C_n\right\rangle\leq2\left\langle\tilde{\bs T},\tilde{\bs C}_n\right\rangle \\
        \leq&\sigma^2\left[\sum_{i=1}^{k^*}{\frac{t_{ii}}{K\lambda_i}}+\frac{256}{15}K\left(\frac{q-c\delta}{1-c}\right)^2\sum_{i=k^*+1}^{d}\lambda_i t_{ii}\right].
    \end{aligned}
    \end{equation}
\end{proof}

\subsubsection{Bias Upper Bound}\label{sec:intro-bias-ub}
We follow the similar approach to construct $\tilde{\bs\eta}_t^\text{bias}$:
\begin{equation}
    \tilde{\bs\eta}_t^\mathrm{bias}=\bs A_t\tilde{\bs\eta}_{t-1}^\mathrm{bias},\quad\tilde{\bs\eta}_0^\mathrm{bias}=\bs\eta_0.
\end{equation}
Then we define $\tilde{\bs B}_t=\tilde{\cB}_t\circ\tilde{\bs B}_{t-1}$. Therefore,
\begin{equation}
    \tilde{\bs B}_t=\tilde{\cB}_t\circ\tilde{\bs B}_{t-1}, \quad\tilde{\bs B}_0=\bs\eta_0\bs\eta_0^\top,\label{eq:def-tilde-bias}
\end{equation}

The first step is to characterize $\tilde{\bs B}_t$. The following lemma bound $\left\langle\tilde{\bs T},\tilde{\bs B}_n\right\rangle$ from above.

\begin{lemma}\label{lemma:tilde-bias-ub}
    With $\tilde{\bs B}_t$ defined in \eqref{eq:def-tilde-bias}, we have
    % \begin{equation}
    %     \left\langle\tilde{\bs T},\tilde{\bs B}_n\right\rangle=\left\Vert\prod_{t=1}^n\bs A_t\begin{bmatrix}
    %         \bs w_0-\bs w^* \\\bs w_0-\bs w^*
    %     \end{bmatrix}\right\Vert_{\tilde{\bs T}}^2
    %     \leq\max_{\bs w\in S(\bs w_0-\bs w^*)}\frac{\left\Vert\bs w\right\Vert_{\bs T_{0:k^*}}^2}{8n^2(\log_2 n)^4}+4\left\Vert\bs w\right\Vert_{\bs T_{k^*:\infty}}^2,
    % \end{equation}
    \begin{equation}
        \left\langle\tilde{\bs T},\tilde{\bs B}_n\right\rangle\leq\max_{\bs w\in S(\bs w_0-\bs w^*)}\frac{\left\Vert\bs w\right\Vert_{\bs T_{0:k^*}}^2}{8n^2(\log_2 n)^4}+4\left\Vert\bs w\right\Vert_{\bs T_{k^*:\infty}}^2,
    \end{equation}    
    where $k^*=\max\left\{k:\lambda_k>\frac{16(1-c)\ln n}{(q-c\delta)K}\right\}$ and $S(\bs w_0-\bs w^*)=\left\{\bs w\in\bbR^d:\left\vert\bs w_i\right\vert\leq\left\vert(\bs w_0-\bs w^*)_i\right\vert\right\}$.
\end{lemma}

The second step is to bound $\bs B_t$ by $\tilde{\bs B}_t$. Define a new recursion $\bs B_t^{(1)}$ as follows:
\begin{equation}
    \bs B_t^{(1)}=\cB_t\circ\bs B_{t-1}^{(1)}+\bbE\left[\widehat{\bs G}_t\otimes\widehat{\bs G}_t\right]\circ\tilde{\bs B}_{t-1}, \quad\bs B_0^{(1)}=\bs O.\label{eq:def-bias-1}
\end{equation}
The following lemma bounds $\bs B_t$ from above.
\begin{lemma}\label{lemma:bias-ub-abs}
    For $0\leq t\leq n$, we have $\bs B_t\preceq\tilde{\bs B}_t+\bs B_t^{(1)}$.
\end{lemma}
\begin{proof}
    We prove the conclusion by induction. By definition, we have $\bs B_0=\tilde{\bs B}_0=\bs\eta_0\bs\eta_0^\top$ and $\bs B_t^{(1)}=\bs O$. Therefore, the conclusion holds for $t=0$. We assume $\bs B_{t-1}\preceq\tilde{\bs B}_{t-1}+\bs B_{t-1}^{(1)}$. Note that
    \begin{equation}
    \begin{aligned}        
        \bs B_t=&\cB_t\circ\bs B_{t-1}\preceq\cB_t\circ\left(\tilde{\bs B}_{t-1}+\bs B_{t-1}^{(1)}\right)\\
        % =&\cB_t\circ\tilde{\bs B}_{t-1}+\cB_t\circ\bs B_{t-1}^{(1)} \\
        \stackrel{a}{\preceq}&\tilde{\cB}_t\circ\tilde{\bs B}_{t-1}+\bbE\left[\widehat{\bs G}_t\otimes\widehat{\bs G}_t\right]\circ\tilde{\bs B}_{t-1}+\cB_t\circ\bs B_{t-1}^{(1)} \\
        =&\tilde{\bs B}_t+\bs B_t^{(1)},
    \end{aligned}
    \end{equation}
    where $\stackrel{a}{\preceq}$ uses that $\cB_t\preceq\tilde{\cB}_t+\bbE\left[\widehat{\bs G}_t\otimes\widehat{\bs G}_t\right]$ in Lemma~\ref{lemma:op-prop}.
\end{proof}

The following step parallels Appendix~\ref{sec:intro-bias-ub}, if we replace $\bs C_t$ with $\bs B_t^{(1)}$. We include detailed proofs for completeness.

\begin{lemma}\label{lemma:bias-1-ub}
    Suppose Assumptions~\ref{assumption:fourth-moment} and~\ref{assumption:noise} hold. With $\bs B_t^{(1)}$ defined in \eqref{eq:def-bias-1}, we have
    \begin{equation}
        \left\langle\tilde{\bs T},\bs B_n^{(1)}\right\rangle\leq\left\Vert\bs w_0-\bs w^*\right\Vert_{\bs S}^2\cdot\left[\sum_{i=1}^{k^*}{\frac{2t_{ii}}{K\lambda_i}}+\frac{512}{15}K\left(\frac{q-c\delta}{1-c}\right)^2\sum_{i=k^*+1}^{d}\lambda_i t_{ii}\right].
    \end{equation}
\end{lemma}

% The third step is to understand $\bs B_t^{(1)}$, which is similar to $\bs C_t$. 
% The following lemma characterizes $\left\langle\tilde{\bs T},\bs B_n^{(1)}\right\rangle$.
% \begin{lemma}\label{lemma:bias-1-ub}
%     Suppose Assumptions~\ref{assumption:fourth-moment} and~\ref{assumption:noise} hold. With $\bs B_t^{(1)}$ defined in \eqref{eq:def-bias-1}, we have
%     \begin{equation}
%         \left\langle\tilde{\bs T},\bs B_n^{(1)}\right\rangle\leq\left\Vert\bs w_0-\bs w^*\right\Vert_{\bs S}^2\cdot\left[\sum_{i=1}^{k^*}{\frac{t_{ii}}{2K\lambda_i}}+\frac{128}{15}K\left(\frac{q-c\delta}{1-c}\right)^2\sum_{i=k^*+1}^{d}\lambda_i t_{ii}\right].
%     \end{equation}
% \end{lemma}

Finally, we bound $\left\langle\tilde{\bs T},\bs B_n\right\rangle$ and obtain the upper bound of bias.
\begin{lemma}\label{lemma:bias-ub}
    Suppose Assumptions~\ref{assumption:fourth-moment} and~\ref{assumption:noise} hold. Then we have
    \begin{equation}
    \begin{aligned}
        \mathrm{Bias}\leq&\left\Vert\bs w_0-\bs w^*\right\Vert_{\bs S}^2\cdot\left[\sum_{i=1}^{k^*}{\frac{2t_{ii}}{K\lambda_i}}+\frac{512}{15}K\left(\frac{q-c\delta}{1-c}\right)^2\sum_{i=k^*+1}^{d}\lambda_i t_{ii}\right] \\
        &+\max_{\bs w\in S(\bs w_0-\bs w^*)}\frac{\left\Vert\bs w\right\Vert_{\bs T_{0:k^*}}^2}{8n^2(\log_2 n)^4}+4\left\Vert\bs w\right\Vert_{\bs T_{k^*:\infty}}^2,
    \end{aligned}
    \end{equation}
    where $k^*=\max\left\{k:\lambda_k>\frac{16(1-c)\ln n}{(q-c\delta)K}\right\}$ and $S(\bs w_0-\bs w^*)=\left\{\bs w\in\bbR^d:\left\vert\bs w_i\right\vert\leq\left\vert(\bs w_0-\bs w^*)_i\right\vert\right\}$.
\end{lemma}
\begin{proof}
    From Lemma~\ref{lemma:bias-ub-abs}, we have $\bs B_n\preceq\tilde{\bs B}_n+\bs B_n^{(1)}$. Taking the inner product with $\tilde{\bs T}$, we get
    \begin{equation}
        \mathrm{Bias}\leq\left\langle\tilde{\bs T},\bs B_n\right\rangle\leq\left\langle\tilde{\bs T},\tilde{\bs B}_n\right\rangle+\left\langle\tilde{\bs T},\bs B_n^{(1)}\right\rangle.
    \end{equation}
    Recall the definition of $\tilde{\bs B}_n$ and $\tilde{\cB}_t$, we have
    \begin{equation}
        \tilde{\bs B}_n=\tilde{\cB}_n\circ\tilde{\cB}_{n-1}\circ\cdots\circ\tilde{\cB}_1\circ\bs B_0=\left(\prod_{t=1}^n\bs A_t\begin{bmatrix}
            \bs w_0-\bs w^* \\\bs w_0-\bs w^*
        \end{bmatrix}\right)\left(\prod_{t=1}^n\bs A_t\begin{bmatrix}
            \bs w_0-\bs w^* \\\bs w_0-\bs w^*
        \end{bmatrix}\right)^\top.
    \end{equation}
    We apply Lemma~\ref{lemma:bias-1-ub} and Lemma~\ref{lemma:tilde-bias-ub} to obtain
    \begin{equation}
    \begin{aligned}
        &\mathrm{Bias}\leq\left\langle\tilde{\bs T},\tilde{\bs B}_n\right\rangle+\left\langle\tilde{\bs T},\bs B_n^{(1)}\right\rangle \\
        \leq&\left\Vert\bs w_0-\bs w^*\right\Vert\cdot\left[\sum_{i=1}^{k^*}{\frac{2t_{ii}}{K\lambda_i}}+\frac{512}{15}K\left(\frac{q-c\delta}{1-c}\right)^2\sum_{i=k^*+1}^{d}\lambda_i t_{ii}\right] \\
        &+\max_{\bs w\in S(\bs w_0-\bs w^*)}\frac{\left\Vert\bs w\right\Vert_{\bs T_{0:k^*}}^2}{8n^2(\log_2 n)^4}+4\left\Vert\bs w\right\Vert_{\bs T_{k^*:\infty}}^2.
    \end{aligned}
    \end{equation}
    This completes the proof.
\end{proof}

\subsection{Proof of Theorem~\ref{thm:asgd-upper-bound}}
\begin{proof}[Proof of Theorem~\ref{thm:asgd-upper-bound}]
    Following the bias-variance decomposition, \eqref{eq:bias-variance-decomposition} shows that
    \begin{equation}
        \bbE\left\Vert\bs w_n-\bs w^*\right\Vert_{\bs T}^2\leq2\cdot\mathrm{Bias}+2\cdot\mathrm{Variance}.
    \end{equation}
    Lemma~\ref{lemma:bias-ub} provides the following upper bound on the bias term:
    \begin{equation}
    \begin{aligned}
        \mathrm{Bias}\leq&\left\Vert\bs w_0-\bs w^*\right\Vert_{\bs S}^2\cdot\left[\sum_{i=1}^{k^*}{\frac{2t_{ii}}{K\lambda_i}}+\frac{512}{15}K\left(\frac{q-c\delta}{1-c}\right)^2\sum_{i=k^*+1}^{d}\lambda_i t_{ii}\right] \\
        &+\max_{\bs w\in S(\bs w_0-\bs w^*)}\frac{\left\Vert\bs w\right\Vert_{\bs T_{0:k^*}}^2}{8n^2(\log_2 n)^4}+4\left\Vert\bs w\right\Vert_{\bs T_{k^*:\infty}}^2,
    \end{aligned}
    \end{equation}
    where $S(\bs w_0-\bs w^*)=\left\{\bs w\in\bbR^d:\left\vert\bs w_i\right\vert\leq\left\vert(\bs w_0-\bs w^*)_i\right\vert\right\}$. Recall that we set $\bs w_0=\bs 0$ and $\bs w^*\in W$ implies $\left\Vert\bs w^*\right\Vert_{\bs M}^2\leq 1$, so
    \begin{equation}
        \max_{\bs w\in S(\bs w_0-\bs w^*)}\left\Vert\bs w\right\Vert_{\bs T_{0:k^*}}^2\leq\max_{\bs w^*\in W}\left\Vert\bs w\right\Vert_{\bs T_{0:k^*}}^2=\max_{\bs w^*\in W}\left\Vert\bs M^{1/2}\bs w\right\Vert_{\bs T'_{0:k^*}}^2=\left\Vert\bs T'_{0:k^*}\right\Vert,
    \end{equation}
    Similarly, we have $\max_{\bs w\in S(\bs w_0-\bs w^*)}\left\Vert\bs w\right\Vert_{\bs T_{k^*:\infty}}^2\leq\left\Vert\bs T'_{k^*:\infty}\right\Vert$ and $\left\Vert\bs w_0-\bs w^*\right\Vert_{\bs S}^2\leq\left\Vert\bs S'\right\Vert\leq c$. Furthermore, we apply Lemma~\ref{lemma:var-ub},
    \begin{equation}
    \begin{aligned}        
        \bbE\left\Vert\bs w_n-\bs w^*\right\Vert_{\bs T}^2\leq&\left(\sigma^2+2c\right)\cdot\left[\sum_{i=1}^{k^*}{\frac{2t_{ii}}{K\lambda_i}}+\frac{512}{15}K\left(\frac{q-c\delta}{1-c}\right)^2\sum_{i=k^*+1}^{d}\lambda_i t_{ii}\right]\\
        &+\frac{\left\Vert\bs T'_{0:k^*}\right\Vert}{8n^2(\log_2 n)^4}+4\left\Vert\bs T'_{k^*:\infty}\right\Vert. \\
    \end{aligned}
    \end{equation}
    where $k^*=\max\left\{k:\lambda_k>\frac{16(1-c)\ln n}{(q-c\delta)K}\right\}$.
\end{proof}

\subsection{Properties of Momentum Matrix}\label{sec:prop-momentum-matrix}
\subsubsection{Bound of Spectral Radius}
Recall that the definition of $\bs A_t$ is
\begin{equation}
    \bs A_t=\bbE\widehat{\bs A}_t=\begin{bmatrix}
        \bs O & \bs I-\delta_t\bs S \\
        -c\bs I & (1+c)\bs I-q_t\bs S
    \end{bmatrix}.
\end{equation}
Note that $\bs S$ is diagonal and $\bs A_t$ is block-diagonal in the eigenspace of $\bs S$. Let $\bs A_{t,i}$ denotes the $i$-th block corresponding to $\lambda_i$, the $i$-th largest eigenvalue. Therefore,
\begin{equation}
    \bs A_{t,i}=\begin{bmatrix}
        0 & 1-\delta_t\lambda_i \\
        -c & 1+c-q_t\lambda_i
    \end{bmatrix}.
\end{equation}
For convenience, we also define $\ell$-th stage version
\begin{equation}
    \bs A_{(\ell)}=\begin{bmatrix}
        \bs O & \bs I-\delta_{(\ell)}\bs S \\
        -c\bs I & (1+c)\bs I-q_{(\ell)}\bs S
    \end{bmatrix},\quad\bs A_{(\ell),i}=\begin{bmatrix}
        0 & 1-\delta_{(\ell)}\lambda_i \\
        -c & 1+c-q_{(\ell)}\lambda_i
    \end{bmatrix}.
\end{equation}
Note that only the product of step size and eigenvalue appears in $\bs A_{t,i}$, we further define
\begin{equation}
    \bs A(\lambda)=\begin{bmatrix}
        0 & 1-\delta\lambda \\
        -c & 1+c-q\lambda
    \end{bmatrix}.\label{eq:def-A-lambda}
\end{equation}
Recall the exponential decaying step size schedule \eqref{eq:step-size}, we have 
\begin{equation}
    \bs A_{t,i}=\bs A_{(\ell),i}=\bs A\left(\frac{\lambda_i}{4^{\ell-1}}\right), \quad \text{if } K(\ell-1)+1\leq t\leq K\ell.
\end{equation}

The eigenvalues of $\bs A(\lambda)$ are
\begin{equation}
    x_1=\frac{1+c-q\lambda}{2}-\frac{\sqrt{\left(1+c-q\lambda\right)^2-4c\left(1-\delta\lambda\right)}}{2},
\end{equation}
\begin{equation}
    x_2=\frac{1+c-q\lambda}{2}+\frac{\sqrt{\left(1+c-q\lambda\right)^2-4c\left(1-\delta\lambda\right)}}{2}.
\end{equation}
Solving $\left(1+c-q\lambda\right)^2-4c\left(1-\delta\lambda\right)\leq 0$ yields
\begin{equation}
   \underbrace{\frac{\left(1-c\right)^2}{\left(\sqrt{q-c\delta}+\sqrt{c\left(q-\delta\right)}\right)^2}}_{\lambda^\dag}<\lambda<\underbrace{\frac{\left(1-c\right)^2}{\left(\sqrt{q-c\delta}-\sqrt{c\left(q-\delta\right)}\right)^2}}_{\lambda^\ddag}.
\end{equation}
We define three intervals
\begin{equation}
    I_1=\left[0,\lambda^\dag\right], \quad I_2=\left(\lambda^\dag,\lambda^\ddag\right), \quad I_3=\left[\lambda^\ddag,+\infty\right).
\end{equation}

Note that the spectral radius $\rho(\bs A(\lambda))=\vert x_2\vert$. We adopt Lemma E.2 from \cite{lirisk}, which characterizes $x_1$ and $x_2$.
\begin{lemma}\label{lemma:radius-bound}
    Let $\lambda\geq 0$.
    \begin{itemize}
        \item If $\lambda\in I_1$, then $x_1$ and $x_2$ are real, and
        \begin{equation}
            x_1\leq x_2\leq 1-\frac{q-c\delta}{1-c}\lambda;
        \end{equation}
        \item If $\lambda\in I_2$, then $x_1$ and $x_2$ are complex, and
        \begin{equation}
            \vert x_1\vert=\vert x_2\vert=\sqrt{c\left(1-\delta\lambda\right)};
        \end{equation}
        \item If $\lambda\in I_3$, then $x_1$ and $x_2$ are real, and
        \begin{equation}
            x_1\leq x_2\leq\frac{c\delta}{q}.
        \end{equation}
    \end{itemize}
\end{lemma}

\subsubsection{Bound of Product of Momentum Matrix}
In this section, we provide bounds of $\bs A^k(\lambda)$. The following lemma provides upper bound of $\left\Vert\bs A^k(\lambda)\right\Vert$.
\begin{lemma}\label{lemma:general-2-norm-bound}
    Given $\bs A(\lambda)$ that are defined in \eqref{eq:def-A-lambda}, we have
    \begin{equation}
        \left\Vert\bs A^k(\lambda)\right\Vert\leq\left\Vert\bs A^k(\lambda)\right\Vert_F\leq\sqrt{6}k\left[\rho(\bs A(\lambda))\right]^{k-1}.
    \end{equation}
\end{lemma}
\begin{proof}
    Define
    \begin{equation}
        a_k=\frac{x_2^k-x_1^k}{x_2-x_1},
    \end{equation}
    we have $a_k\in\bbR$ and
    \begin{equation}\label{eq:explicit-Ak}
        \bs A^k(\lambda)=\begin{bmatrix}
            -c\left(1-\delta\lambda\right)a_{k-1} & \left(1-\delta\lambda\right)a_k \\
            -ca_k & a_{k+1}
        \end{bmatrix}.
    \end{equation}
    Note that for any $\lambda\geq 0$, we have $\vert x_1\vert\leq\vert x_2\vert$, and
    \begin{equation}
    \begin{aligned}
        \vert a_k\vert=&\left\vert\frac{x_2^k-x_1^k}{x_2-x_1}\right\vert=\left\vert\sum_{i=0}^{k-1}{x_1^ix_2^{k-1-i}}\right\vert \\
        \stackrel{a}{\leq}&\sum_{i=0}^{k-1}{\vert x_1\vert^i\vert x_2\vert ^{k-1-i}}\stackrel{b}{\leq}\sum_{i=0}^{k-1}{\vert x_2\vert ^{k-1}} \\
        =&k\vert x_2\vert^{k-1},
    \end{aligned}
    \end{equation}
    where $\stackrel{a}{\leq}$ uses the triangular inequality for complex number, and $\stackrel{b}{\leq}$ uses $\vert x_1\vert\leq\vert x_2\vert$. Direct calculation yields $x_1x_2=c\left(1-\delta\lambda\right)$, so $\vert c\left(1-\delta\lambda\right)\vert\leq\vert x_2\vert^2$. We bound $\left\Vert\bs A^k(\lambda)\right\Vert_F^2$ by
    \begin{equation}
    \begin{aligned}        
        \left\Vert\bs A^k(\lambda)\right\Vert_F^2=&
            \left[-c\left(1-\delta\lambda\right)a_{k-1}\right]^2+\left[\left(1-\delta\lambda\right)a_k\right]^2+\left(-ca_k\right)^2+a_{k+1}^2 \\
        \leq&(k-1)^2\vert x_2\vert^{2k}+k^2\vert x_2\vert^{2(k-1)}+k^2\vert x_2\vert^{2(k-1)}+(k+1)^2\vert x_2\vert^{2k} \\
        \leq&\left[(k-1)^2+k^2+k^2+(k+1)^2\right]\vert x_2\vert^{2(k-1)} \\
        =&\left(4k^2+2\right)\vert x_2\vert^{2(k-1)}. \\
        \leq&6k^2\vert x_2\vert^{2(k-1)}.
    \end{aligned}
    \end{equation}
    Therefore, $\left\Vert\bs A^k(\lambda)\right\Vert\leq\left\Vert\bs A^k(\lambda)\right\Vert_F\leq\sqrt{6}k\vert x_2\vert^{2(k-1)}=\sqrt{6}k\left[\rho(\bs A(\lambda))\right]^{k-1}$.
\end{proof}

For $k\leq K$, the following lemma bounds $\left\Vert\bs A^k(\lambda)\right\Vert$ from above uniformly.
\begin{lemma}\label{lemma:2-norm-uniform-bound}
    For $\lambda\leq\lambda_1$, we have
    \begin{equation}
        \left\Vert\bs A^k(\lambda)\right\Vert\leq\sqrt 6 K.
    \end{equation}
\end{lemma}
\begin{proof}
    For $k=0$, the conclusion is obvious. If $k\geq 1$, for $\lambda\leq\lambda_1$, we have $\rho(\bs A(\lambda)\leq 1$. Thus, by Lemma~\ref{lemma:general-2-norm-bound},
    \begin{equation}
        \left\Vert\bs A^k(\lambda)\right\Vert\leq\sqrt 6 K\left[\rho(\bs A(\lambda)\right]^{k-1}.
    \end{equation}
\end{proof}

The following lemma bounds $\left\Vert\bs A^K(\lambda)\right\Vert$ from above.
\begin{lemma}\label{lemma:2-norm-bound}
    For $\lambda\geq\frac{4(1-c)\ln n}{(q-c\delta)K}$ and $n\geq 16$, we have
    \begin{equation}
        \left\Vert\bs A^K(\lambda)\right\Vert\leq\frac{\sqrt 6}{n^2\log_2 n}\leq 1.
    \end{equation}
\end{lemma}
\begin{proof}
    We bound $\left\Vert\bs A^K(\lambda)\right\Vert$ for $\lambda\in I_1, I_2, I_3$, respectively.
    \begin{enumerate}
        \item If $\lambda\in I_1$, by Lemma~\ref{lemma:radius-bound} and assumption,
        \begin{equation}
            \rho(\bs A(\lambda))=\left\vert x_2\right\vert\leq 1-\frac{q-c\delta}{1-c}\lambda\leq 1-\frac{4\ln n}{4K}.
        \end{equation}
        Thus, by Lemma~\ref{lemma:general-2-norm-bound},
        \begin{equation}
        \begin{aligned}
            \left\Vert\bs A^K(\lambda)\right\Vert\leq&\sqrt{6}K\left[\rho(\bs A(\lambda))\right]^{K-1}\leq\sqrt{6}K\left(1-4\ln n\right)^{K-1} \\
            =&\sqrt{6}K\exp\left[(K-1)\ln\left(1-4\ln n\right)\right] \\
            \stackrel{a}{\leq}&\sqrt{6}K\exp\left[-\frac{4(K-1)\ln n}{K}\right]\stackrel{b}{\leq}\sqrt{6}K\exp\left(-3\ln n\right) \\
            =&\frac{\sqrt 6}{n^2\log_2 n},
        \end{aligned}
        \end{equation}
        where $\stackrel{a}{\leq}$ uses $\ln x\leq x-1, \forall x\in\bbR$, and $\stackrel{b}{\leq}$ holds for $n\geq 16\implies K\geq 4\implies\frac{K-1}{K}\geq\frac{3}{4}$.

        \item If $\lambda\in I_2$, by Lemma~\ref{lemma:radius-bound} and assumption,
        \begin{equation}
            \rho(\bs A(\lambda))=\left\vert x_2\right\vert=\sqrt{c(1-\delta\lambda)}\leq\sqrt c.
        \end{equation}
        Thus, by Lemma~\ref{lemma:general-2-norm-bound},
        \begin{equation}
        \begin{aligned}
            \left\Vert\bs A^K(\lambda)\right\Vert\leq&\sqrt{6}K\left[\rho(\bs A(\lambda))\right]^{K-1}\leq\sqrt{6}K\left(\sqrt c\right)^{K-1} \\
            =&\sqrt{6}K\exp\left[-\frac{(K-1)\ln c}{2}\right] \\
            \stackrel{a}{\leq}&\sqrt{6}K\exp\left[-\frac{(K-1)(1-c)}{2}\right]\stackrel{b}{\leq}\sqrt{6}K\exp\left[-\frac{8(K-1)\ln n}{K}\right] \\
            \leq&\sqrt{6}K\exp\left(-6\ln n\right)=\frac{\sqrt 6}{n^5\log_2 n}\leq\frac{\sqrt 6}{n^2\log_2 n},
        \end{aligned}
        \end{equation}
        where $\stackrel{a}{\leq}$ uses $\ln x\leq x-1, \forall x\in\bbR$, and $\stackrel{b}{\leq}$ holds for $K(1-c)\geq 16\ln n$ in \eqref{eq:para-choice-req} and $\frac{K-1}{K}\geq\frac{3}{4}$.

        \item If $\lambda\in I_2$, by Lemma~\ref{lemma:radius-bound} and assumption,
        \begin{equation}
            \rho(\bs A(\lambda))=\left\vert x_2\right\vert\leq\frac{c\delta}{q}\leq c.
        \end{equation}
        Thus, by Lemma~\ref{lemma:general-2-norm-bound},
        \begin{equation}
        \begin{aligned}
            \left\Vert\bs A^K(\lambda)\right\Vert\leq&\sqrt{6}K\left[\rho(\bs A(\lambda))\right]^{K-1}\leq\sqrt{6}Kc^{K-1} \\
            =&\sqrt{6}K\exp\left[-(K-1)\ln c\right] \\
            \stackrel{a}{\leq}&\sqrt{6}K\exp\left[-(K-1)(1-c)\right]\stackrel{b}{\leq}\sqrt{6}K\exp\left[-\frac{16(K-1)\ln n}{K}\right] \\
            \leq&\sqrt{6}K\exp\left(-12\ln n\right)=\frac{\sqrt 6}{n^{11}\log_2 n}\leq\frac{\sqrt 6}{n^2\log_2 n},
        \end{aligned}
        \end{equation}
        where $\stackrel{a}{\leq}$ uses $\ln x\leq x-1, \forall x\in\bbR$, and $\stackrel{b}{\leq}$ holds for $\frac{K-1}{K}\geq\frac{3}{4}$.
    \end{enumerate}
\end{proof}

For $k\in\bbN$, we have a uniform bound of $\left\vert\left(\bs A^k(\lambda)\begin{bmatrix}
    1 \\ 1
\end{bmatrix}\right)_2\right\vert$, which is tighter than Lemma~\ref{lemma:2-norm-uniform-bound}.
\begin{lemma}\label{lemma:prod-A-bias-bound}
    For $\lambda\leq\lambda_1$ and $k\in\bbN$, we have
    \begin{equation}
        \left\vert\left(\bs A^k(\lambda)\begin{bmatrix}
            1 \\ 1
        \end{bmatrix}\right)_2\right\vert\leq\begin{cases}
            1, & \lambda\in I_1, I_3; \\
            2, & \lambda\in I_2.
        \end{cases}.
    \end{equation}
\end{lemma}
\begin{proof}
    From \eqref{eq:explicit-Ak}, we have
    \begin{equation}
        \left\vert\left(\bs A^k(\lambda)\begin{bmatrix}
            1 \\ 1
        \end{bmatrix}\right)_2\right\vert=\left\vert a_{k+1}-ca_k\right\vert.
    \end{equation}
    We bound $\left\vert a_{k+1}-ca_k\right\vert\leq 2$ for $\lambda\in I_1, I_2, I_3$, respectively.
    \begin{enumerate}
        \item If $\lambda\in I_1$, by Lemma~\ref{lemma:radius-bound}, and $\delta\leq q$, we have $a_k\geq 0$, and
        \begin{equation}
            x_1\leq x_2\leq 1-\frac{q-c\delta}{1-c}\lambda\leq 1-\delta\lambda.
        \end{equation}
        Since $x_1x_2=c(1-\delta\lambda)$, we have $c\leq x_1\leq x_2$. Therefore,
        \begin{equation}
        \begin{aligned}
            a_{k+1}-ca_k\geq& a_{k+1}-x_1a_k=x_2^k>0,\\
            a_{k+1}-ca_k\leq& a_{k+1}-x_1x_2a_k=\sum_{i=0}^{k}{x_1^ix_2^{k-i}}-x_1x_2\sum_{i=0}^{k-1}{x_1^ix_2^{k-1-i}} \\
            =&\sum_{i=0}^{k}{x_1^ix_2^{k-i}}-x_2\sum_{i=1}^{k}{x_1^ix_2^{k-i}}=x_2^k+(1-x_2)\sum_{i=1}^{k}{x_1^ix_2^{k-i}} \\
            \leq& x_2^k+k(1-x_2)x_2^k=x_2^k\left[1+k(1-x_2)\right]\stackrel{a}{\leq} 1,
        \end{aligned}
        \end{equation}
        where $\stackrel{a}{\leq}$ applies Lemma~\ref{lemma:aux-lemma-r}.
        \item If $\lambda\in I_2$, by Lemma~\ref{lemma:radius-bound}, $x_1$ and $x_2$ are complex and conjugate. Let $x_{1,2}=r(\cos\theta\pm\mathrm{i}\sin\theta)$, we have $r=\sqrt{c(1-\delta\lambda)}\leq 1$ and $0\leq\theta\leq\pi/2$ where $2r\cos\theta=x_1+x_2=1+c-q\lambda\geq 0$ from Lemma~\ref{lemma:param-prop}. Thus
        \begin{equation}
        \begin{aligned}            
            a_{k+1}-ca_k=&\frac{r^k\sin\left((k+1)\theta\right)}{\sin\theta}-\frac{r^{k-1}\sin\left(k\theta\right)}{\sin\theta} \\
            \stackrel{a}{=}&r^{k-1}\left(r\cos k\theta+\frac{r\cos\theta-c}{\sin\theta}\sin k\theta\right) \\
            =&r^{k-1}\left(r\cos k\theta+\frac{r-c}{\sin\theta}\sin k\theta-\frac{r(1-\cos\theta)}{\sin\theta}\sin k\theta\right) \\
            \stackrel{b}{=}&r^{k-1}\left(r\cos k\theta+\frac{r-c}{\sin\theta}\sin k\theta-r\tan\frac{\theta}{2}\sin k\theta\right),
        \end{aligned}
        \end{equation}
        where $\stackrel{a}{=}$ is from $\sin\left((k+1)\theta\right)=\sin k\theta\cos\theta+\cos k\theta\sin\theta$, and $\stackrel{b}{=}$ is from
        \begin{equation}
            \frac{1-\cos\theta}{\sin\theta}=\frac{2\sin^2\frac{\theta}{2}}{2\sin\frac{\theta}{2}\cos\frac{\theta}{2}}=\tan\frac{\theta}{2}.
        \end{equation}
        By triangular inequality, and $\left\vert\sin k\theta\right\vert\leq 1$, $\left\vert\cos k\theta\right\vert\leq 1$, $\left\vert\tan\frac{\theta}{2}\right\vert\leq 1$
        \begin{equation}
        \begin{aligned}            
            \left\vert a_{k+1}-ca_k\right\vert\leq& r^{k-1}\left(r\left\vert\cos k\theta\right\vert+\left\vert r-c\right\vert\left\vert\frac{\sin k\theta}{\sin\theta}\right\vert\right)+r^k\left\vert\tan\frac{\theta}{2}\right\vert\left\vert\sin k\theta\right\vert \\
            \stackrel{a}{\leq}& r^{k-1}\left(r+k(1-r)\right)+r^k=r^{k-1}\left(1+(k-1)(1-r)\right)+r^k \\
            \stackrel{b}{\leq}&2,
        \end{aligned}
        \end{equation}
        where $\stackrel{a}{\leq}$ holds for $r^2\leq c\leq 1\implies\left\vert r-c\right\vert\leq \max\left\{\left\vert r-r^2\right\vert,\left\vert r-1\right\vert\right\}=1-r$ and $\left\vert\frac{\sin k\theta}{\sin\theta}\right\vert\leq k$ in Lemma~\ref{lemma:aux-lemma-sin}, $\stackrel{b}{\leq}$ is from Lemma~\ref{lemma:aux-lemma-r} and $0\leq r\leq 1$.
        \item If $\lambda\in I_3$, by Lemma~\ref{lemma:radius-bound}, and $\delta\leq q$, we have $a_k\geq 0$, and
        \begin{equation}
            x_1\leq x_2\leq\frac{c\delta}{q}\leq c.
        \end{equation}
        Therefore,
        \begin{equation}
        \begin{aligned}
            a_{k+1}-ca_k\geq& a_{k+1}-a_k=\sum_{i=0}^{k}{x_1^ix_2^{k-i}}-\sum_{i=0}^{k-1}{x_1^ix_2^{k-1-i}} \\
            =&x_1^k-(1-x_2)\sum_{i=0}^{k-1}{x_1^ix_2^{k-1-i}} \\
            \geq&x_1^k-k(1-x_2)x_2^{k-1} \\
            \geq&-x_2^{k-1}\left(1+(k-1)x_2^k\right)\stackrel{a}{\geq} -1, \\
            a_{k+1}-ca_k\leq& a_{k+1}-x_2a_k=x_1^k\leq 1,
        \end{aligned}
        \end{equation}
        where $\stackrel{a}{\geq}$ holds for Lemma~\ref{lemma:aux-lemma-r}.
    \end{enumerate}
\end{proof}

For $\lambda\leq\frac{(1-c)^2}{q-c\delta}$, we define $\bs P$, which diagonalizes $\bs V$:
\begin{equation}
    \bs P=\begin{bmatrix}
    1 & -1 \\
    1 & -c \\
\end{bmatrix}, \quad\bs P^{-1}=\dfrac{1}{1-c}\begin{bmatrix}
    -c & 1 \\
    -1 & 1 \\
\end{bmatrix}.\label{eq:def-P}
\end{equation}The following lemma provides bound of $\bs P^{-1}\bs A(\lambda)\bs P$.
\begin{lemma}\label{lemma:tf-A-bound}
    Let $\bs P$ and $\bs P^{-1}$ defined in \eqref{eq:def-P}. Suppose $\lambda\leq\frac{(1-c)^2}{q-c\delta}$, we have
    \begin{equation}
        \left\Vert\bs P^{-1}\bs A(\lambda)\bs P\right\Vert\leq 1.
    \end{equation}
\end{lemma}
\begin{proof}
    Let
    \begin{equation}
        \bs M=\bs P^{-1}\bs A(\lambda)\bs P=\begin{bmatrix}            
             1-\xi\lambda  & c\xi\lambda  \\
             -\eta  \lambda  & c+c\eta\lambda \\
        \end{bmatrix},
    \end{equation}
    we will show that $\bs I-\bs M^\top\bs M$ is a PSD matrix. Let $\xi=\frac{q-c\delta}{1-c}$ and $\eta=\frac{q-\delta}{1-c}$, so $\xi\lambda<1-c$. Direct calculation shows that
    \begin{equation}
        \bs M^\top\bs M=\begin{bmatrix}    
            (1-\xi\lambda)^2+\eta^2\lambda^2 & c\lambda\left(\xi-\eta-\left(\xi^2+\eta^2\right)\lambda\right) \\
            c\lambda\left(\xi-\eta-\left(\xi^2+\eta^2\right)\lambda\right) & c^2\xi^2\lambda^2+c^2(1+\eta\lambda)^2 \\
        \end{bmatrix}.
    \end{equation}
    Furthermore, 
    \begin{equation}
    \begin{aligned}        
        \left(\bs I-\bs M^\top\bs M\right)_{11}=&1-\left[(1-\xi\lambda)^2+\eta^2\lambda^2\right]=2\xi\lambda-\xi^2\lambda^2-\eta^2\lambda^2 \\
        \stackrel{a}{\geq}&2\xi\lambda-2\xi^2\lambda^2=2\xi\lambda(1-\xi\lambda) \\
        \stackrel{b}{\geq}&0, \\
        \det\left(\bs I-\bs M^\top\bs M\right)=&\lambda\left[2(1-c^2)\xi-(\xi^2+\eta^2+2c^2\xi\eta)\lambda\right] \\
        \stackrel{c}{\geq}&\lambda\left[2(1-c^2)\xi-2(1+c^2)\xi^2\lambda\right] \\
        \stackrel{d}{\geq}&\lambda\left[2(1-c^2)\xi-2(1+c^2)(1-c)\xi\right] \\
        =&\left[2(1-c^2)-2(1+c^2)(1-c)\right]\xi\lambda \\
        =&2c(1-c)^2\xi\lambda\geq0
    \end{aligned}
    \end{equation}
    where $\stackrel{a}{\geq}$ and $\stackrel{c}{\geq}$ uses $\eta\leq\xi$, $\stackrel{b}{\geq}$ and $\stackrel{d}{\geq}$ a uses $\xi\lambda\leq 1-c\leq 1$. Therefore, by Sylvester's criterion, $\bs I-\bs M^\top\bs M$ is a PSD matrix. From the definition of $\bs M$, we have
    \begin{equation}
        \left\Vert\bs P^{-1}\bs A(\lambda)\bs P\right\Vert=\left\Vert\bs M\right\Vert=\sup_{\bs x}\frac{\left\Vert\bs M\bs x\right\Vert}{\left\Vert\bs x\right\Vert}=\sup_{\bs x}\frac{\bs x^\top\bs M\bs M\bs x}{\bs x^\top\bs x}\leq 1.
    \end{equation}
\end{proof}

The following lemma provides upper bound of the product of momentum matrices.
\begin{lemma}\label{lemma:tf-prod-A-bound}
    For $\mu_1,\mu_2,\ldots,\mu_k\leq\frac{(1-c)^2}{q-c\delta}$, we have
    \begin{equation}
        \left\Vert\prod_{i=1}^{k}{\bs A(\mu_i)}\right\Vert\leq\frac{4}{1-c}.
    \end{equation}
\end{lemma}
\begin{proof}
    Note that
    \begin{equation}
    \begin{aligned}        
        \left\Vert\prod_{i=1}^{k}{\bs A(\mu_i)}\right\Vert=&\left\Vert\bs P\left(\prod_{i=1}^{k}{\bs P^{-1}\bs A(\mu_i)\bs P}\right)\bs P^{-1}\right\Vert \\
        \leq&\left\Vert\bs P\right\Vert\prod_{i=1}^{k}{\left\Vert\bs P^{-1}\bs A(\mu_i)\bs P\right\Vert}\left\Vert\bs P^{-1}\right\Vert \\
        \stackrel{a}{\leq}&2\cdot1\cdot\frac{2}{1-c}=\frac{4}{1-c},
    \end{aligned}
    \end{equation}
    where $\stackrel{a}{\leq}$ applies $\left\Vert\bs P\right\Vert\leq\left\Vert\bs P\right\Vert_F\leq2$, $\left\Vert\bs P^{-1}\right\Vert\leq\left\Vert\bs P^{-1}\right\Vert_F\leq\frac{2}{1-c}$ and Lemma~\ref{lemma:tf-A-bound}.
\end{proof}

The following lemma provides an upper bound of the product of momentum matrices applied to noise vector $\begin{bmatrix}
    \delta & q
\end{bmatrix}^\top$.
\begin{lemma}\label{lemma:tf-prod-A-noise-bound}
    For $\mu_1,\mu_2,\ldots,\mu_k\leq\frac{(1-c)^2}{q-c\delta}$, we have
    \begin{equation}
        \left\Vert\prod_{i=1}^{k}{\bs A(\mu_i)}\begin{bmatrix}
            \delta \\
            q
        \end{bmatrix}\right\Vert\leq\frac{2\sqrt 2 (q-c\delta)}{1-c}.
    \end{equation}
\end{lemma}
\begin{proof}
    Note that
    \begin{equation}
    \begin{aligned}        
        \left\Vert\prod_{i=1}^{k}{\bs A(\mu_i)}\right\Vert=&\left\Vert\bs P\left(\prod_{i=1}^{k}{\bs P^{-1}\bs A(\mu_i)\bs P}\right)\bs P^{-1}\begin{bmatrix}
            \delta \\
            q
        \end{bmatrix}\right\Vert \\=&\left\Vert\bs P\left(\prod_{i=1}^{k}{\bs P^{-1}\bs A(\mu_i)\bs P}\right)\begin{bmatrix}
            \frac{q-c\delta}{1-c} \\
            \frac{q-\delta}{1-c}
        \end{bmatrix}\right\Vert \\
        \leq&\left\Vert\bs P\right\Vert\prod_{i=1}^{k}{\left\Vert\bs P^{-1}\bs A(\mu_i)\bs P\right\Vert}\left\Vert\begin{bmatrix}
            \frac{q-c\delta}{1-c} \\
            \frac{q-\delta}{1-c}
        \end{bmatrix}\right\Vert \\
        \stackrel{a}{\leq}&2\cdot1\cdot\frac{\sqrt 2 (q-c\delta)}{1-c}=\frac{2\sqrt 2 (q-c\delta)}{1-c},
    \end{aligned}
    \end{equation}
    where $\stackrel{a}{\leq}$ applies $\left\Vert\bs P\right\Vert\leq2$, $\frac{q-\delta}{1-c}\leq\frac{q-c\delta}{1-c}$ and Lemma~\ref{lemma:tf-A-bound}.
\end{proof}

The following lemma provides an upper bound of the product of momentum matrices applied to bias vector $\begin{bmatrix}
    1 & 1
\end{bmatrix}^\top$.
\begin{lemma}\label{lemma:tf-prod-A-bias-bound}
    For $\mu_1,\mu_2,\ldots,\mu_k\leq\frac{(1-c)^2}{q-c\delta}$, we have
    \begin{equation}
        \left\Vert\prod_{i=1}^{k}{\bs A(\mu_i)}\begin{bmatrix}
            1 \\
            1
        \end{bmatrix}\right\Vert\leq 2.
    \end{equation}
\end{lemma}
\begin{proof}
    Note that
    \begin{equation}
    \begin{aligned}        
        \left\Vert\prod_{i=1}^{k}{\bs A(\mu_i)}\right\Vert=&\left\Vert\bs P\left(\prod_{i=1}^{k}{\bs P^{-1}\bs A(\mu_i)\bs P}\right)\bs P^{-1}\begin{bmatrix}
            1 \\
            1
        \end{bmatrix}\right\Vert \\=&\left\Vert\bs P\left(\prod_{i=1}^{k}{\bs P^{-1}\bs A(\mu_i)\bs P}\right)\begin{bmatrix}
            1 \\
            0
        \end{bmatrix}\right\Vert \\
        \leq&\left\Vert\bs P\right\Vert\prod_{i=1}^{k}{\left\Vert\bs P^{-1}\bs A(\mu_i)\bs P\right\Vert}\left\Vert\begin{bmatrix}
            1 \\
            0
        \end{bmatrix}\right\Vert \\
        \stackrel{a}{\leq}&2\cdot1\cdot1=2,
    \end{aligned}
    \end{equation}
    where $\stackrel{a}{\leq}$ applies $\left\Vert\bs P\right\Vert\leq2$ and Lemma~\ref{lemma:tf-A-bound}.
\end{proof}

\subsection{Variance Upper Bound}\label{sec:asgd-variance-ub}
This section analyzes $\tilde{\bs C}_t$ which defined in \eqref{eq:def-tilde-var}. We first provide a characterization of the stationary state, and then prove Lemma~\ref{lemma:tilde-var-ub} and~\ref{lemma:tilde-var-uniform-ub}.

\subsubsection{Analysis of Stationary State}
We introduce the stationary state matrix at $\ell$-th stage:
\begin{equation}\label{eq:def-Q}
    \tilde{\bs Q}_{(\ell)}=\sum_{k=1}^{\infty}\tilde{\cB}_{(\ell)}^k\circ\begin{bmatrix}
    \delta_{(\ell)}^2\bs S & \delta_{(\ell)} q_{(\ell)}\bs S \\
    \delta_{(\ell)} q_{(\ell)}\bs S & q_{(\ell)}^2\bs S
    \end{bmatrix}.
\end{equation}
Lemma F.4 in \cite{lirisk} shows $\tilde{\bs Q}_{(\ell)}$ exists and finite. Note that since $\tilde{\cB}_t=\bs A_{(\ell)}\otimes\bs A_{(\ell)}$ and $\bs A_{(\ell)}$ is block-diagonal, each $\tilde{\cB}_{(\ell)}^k\circ\begin{bmatrix}
    \delta_{(\ell)}^2\bs S & \delta_{(\ell)} q_{(\ell)}\bs S \\
    \delta_{(\ell)} q_{(\ell)}\bs S & q_{(\ell)}^2\bs S
\end{bmatrix}$ is also block-diagonal. Thus, $\tilde{\bs Q}_{(\ell)}$ is block-diagonal, and we denote the $i$-th block as $\tilde{\bs Q}_{(\ell),i}\in\bbR^{2\times2}$. Furthermore, we define
\begin{equation}
    \tilde{\cB}_{t,i}=\bs A_{t,i}\otimes\bs A_{t,i}, \quad\tilde{\cB}_{(\ell),i}=\bs A_{(\ell),i}\otimes\bs A_{(\ell),i},\label{eq:def-tilde-B-op}
\end{equation}
Then $\tilde{\bs Q}_{(\ell),i}$ can be represented as
\begin{equation}\label{eq:def-Qi}
    \tilde{\bs Q}_{(\ell),i}=\sum_{k=1}^{\infty}\tilde{\cB}_{(\ell),i}^k\circ\begin{bmatrix}
    \delta_{(\ell)}^2\lambda_i & \delta_{(\ell)} q_{(\ell)}\lambda_i \\
    \delta_{(\ell)} q_{(\ell)}\lambda_i & q_{(\ell)}^2\lambda_i
    \end{bmatrix}.
\end{equation}
    
Define an operator $\cT_{(\ell)}=\cI-\tilde{\cB}_{(\ell)}+\bs G_{(\ell)}\otimes\bs G_{(\ell)}=\cI-\bs V\otimes\bs V+\bs V\otimes\bs G_{(\ell)}+\bs G_{(\ell)}\otimes\bs V$, and
\begin{equation}
    \bs U_{(\ell)}=\cT_{(\ell)}^{-1}\circ\begin{bmatrix}
    \delta_{(\ell)}^2\bs S & \delta_{(\ell)} q_{(\ell)}\bs S \\
    \delta_{(\ell)} q_{(\ell)}\bs S & q_{(\ell)}^2\bs S
    \end{bmatrix}.\label{eq:def-U}
\end{equation}
The same argument holds for $\bs U_{(\ell)}$ to be block-diagonal, and $i$-th block of $\bs U_{(\ell)}$ is denoted as $\bs U_{(\ell),i}\in\bbR^{2\times2}$. The following lemma characterize $\tilde{\bs Q}_{(\ell)}$ using $\bs U_{(\ell),i}$.
\begin{lemma}\label{lemma:calc-S}
    Let $\tilde{\bs Q}_{(\ell)}$ defined in \eqref{eq:def-Q}. Then we have
    \begin{equation}
        \tilde{\bs Q}_{(\ell),i}=\frac{1}{1-\left(\bs U_{(\ell),i}\right)_{22}\lambda_i}\bs U_{(\ell),i}.\label{eq:calc-S}
    \end{equation}
\end{lemma}
\begin{proof}
    Note that
    \begin{equation}
    \begin{aligned}
        \sum_{k=0}^{\infty}\tilde{\cB}_{(\ell)}^k=&\left(\cI-\tilde{\cB}_{(\ell)}\right)^{-1}=\left(\cT_{(\ell)}-\bs G_{(\ell)}\otimes\bs G_{(\ell)}\right)^{-1} \\
        =&\left[\cT_{(\ell)}\circ\left(\cI-\cT_{(\ell)}^{-1}\circ\left(\bs G_{(\ell)}\otimes\bs G_{(\ell)}\right)\right)\right]^{-1} \\
        =&\left(\cI-\cT_{(\ell)}^{-1}\circ\left(\bs G_{(\ell)}\otimes\bs G_{(\ell)}\right)\right)^{-1}\circ\cT_{(\ell)}^{-1} \\
        =&\sum_{k=0}^{\infty}{\left(\cT_{(\ell)}^{-1}\circ\left(\bs G_{(\ell)}\otimes\bs G_{(\ell)}\right)\right)^k\circ\cT_{(\ell)}^{-1}}.
    \end{aligned}
    \end{equation}
    We introduce $\cT_{(\ell),i}=\cI-\bs V_i\otimes\bs V_i+\bs V_i\otimes\bs G_{(\ell),i}+\bs G_{(\ell),i}\otimes\bs V_i$, which operates on $\bbR^{2\times2}$ matrix. Therefore, we can calculate the $i$-th block of $\tilde{\bs Q}_{(\ell)}$ as follows:
    \begin{equation}\label{eq:recursion-S}
    \begin{aligned}
        \tilde{\bs Q}_{(\ell),i}=&\sum_{k=0}^{\infty}{\left(\cT_{(\ell),i}^{-1}\circ\left(\bs G_{(\ell),i}\otimes\bs G_{(\ell),i}\right)\right)^k\circ\cT_{(\ell),i}^{-1}}\circ\begin{bmatrix}
            \delta_{(\ell)}^2\lambda_i & \delta_{(\ell)}q_{(\ell)}\lambda_i \\
            \delta_{(\ell)}q_{(\ell)}\lambda_i & q_{(\ell)}^2\lambda_i
        \end{bmatrix} \\
        =&\sum_{k=0}^{\infty}{\left(\cT_{(\ell),i}^{-1}\circ\left(\bs G_{(\ell),i}\otimes\bs G_{(\ell),i}\right)\right)^k\circ\bs U_{(\ell),i}} \\
        =&\bs U_{(\ell),i}+\sum_{k=1}^{\infty}{\left(\cT_{(\ell),i}^{-1}\circ\left(\bs G_{(\ell),i}\otimes\bs G_{(\ell),i}\right)\right)^k\circ\bs U_{(\ell),i}} \\
        =&\bs U_{(\ell),i}+\sum_{k=0}^{\infty}{\left(\cT_{(\ell),i}^{-1}\circ\left(\bs G_{(\ell),i}\otimes\bs G_{(\ell),i}\right)\right)^k\circ\cT_{(\ell),i}^{-1}\circ\left(\bs G_{(\ell),i}\otimes\bs G_{(\ell),i}\right)\circ\bs U_{(\ell),i}} \\
        \stackrel{a}{=}&\bs U_{(\ell),i}+\left(\bs U_{(\ell),i}\right)_{22}\lambda_i\sum_{k=0}^{\infty}{\left(\cT_{(\ell),i}^{-1}\circ\left(\bs G_{(\ell),i}\otimes\bs G_{(\ell),i}\right)\right)^k\circ\bs U_{(\ell),i}} \\
        =&\bs U_{(\ell),i}+\left(\bs U_{(\ell),i}\right)_{22}\lambda_i\tilde{\bs Q}_{(\ell),i},
    \end{aligned}
    \end{equation}
    where $\stackrel{a}{=}$ uses
    \begin{equation}
    \begin{aligned}
        \cT_{(\ell),i}^{-1}\circ\left(\bs G_{(\ell),i}\otimes\bs G_{(\ell),i}\right)\circ\bs U_{(\ell),i}=&\cT_{(\ell),i}^{-1}\circ\left(\bs G_{(\ell),i}\bs U_{(\ell),i}\bs G_{(\ell),i}^\top\right) \\
        =&\cT_{(\ell),i}^{-1}\circ\left(\left(\bs U_{(\ell),i}\right)_{22}\lambda_i\begin{bmatrix}
            \delta_{(\ell)}^2\lambda_i & \delta_{(\ell)}q_{(\ell)}\lambda_i \\
            \delta_{(\ell)}q_{(\ell)}\lambda_i & q_{(\ell)}^2\lambda_i
        \end{bmatrix}\right) \\
        =&\left(\bs U_{(\ell),i}\right)_{22}\lambda_i\cdot\cT_{(\ell),i}^{-1}\circ\begin{bmatrix}
            \delta_{(\ell)}^2\lambda_i & \delta_{(\ell)}q_{(\ell)}\lambda_i \\
            \delta_{(\ell)}q_{(\ell)}\lambda_i & q_{(\ell)}^2\lambda_i
        \end{bmatrix} \\
        =&\left(\bs U_{(\ell),i}\right)_{22}\lambda_i\bs U_{(\ell),i}.
    \end{aligned}
    \end{equation}
    Solving the recursion \eqref{eq:recursion-S} yields the desired result.
\end{proof}

The following lemma characterizes $\bs U_{(\ell),i}$ and $\bs Q_{(\ell),i}$.
\begin{lemma}\label{lemma:U-S-prop}
    With $\bs U_{(\ell),i}$ defined in \eqref{eq:def-U}, we have
    \begin{enumerate}
        \item By Equation (F.9) of \cite{lirisk}, we have
        \begin{equation}
            \left(\bs U_{(\ell),i}\right)_{22}=\frac{\delta_{(\ell)}}{2}+\frac{(1+c)(q_{(\ell)}-\delta_{(\ell)})}{2\left(1-c^2+c\lambda_i(q_{(\ell)}+c\delta_{(\ell)})\right)};
        \end{equation}
        \item We have
        \begin{equation}
            \left(\bs U_{(\ell),i}\right)_{22}\leq\frac{\delta}{2}+\frac{c_3}{2\psi\tilde{\kappa}\lambda_i}
        \end{equation}
        \item By Equation (44) of \cite{jain2018accelerating}, we have $\left(\bs U_{(\ell),i}\right)_{11}=(1-2\delta_{(\ell)}\lambda_i)\left(\bs U_{(\ell),i}\right)_{22}+\delta_{(\ell)}^2\lambda_i$;
        \item We have $\left(\bs U_{(\ell),i}\right)_{11}\leq\left(\bs U_{(\ell),i}\right)_{22}$, and $\bs U_{(\ell),i}\preceq 2\left(\bs U_{(\ell),i}\right)_{22}\bs I$.
        \item $\bs U_{(\ell),i}\preceq\bs Q_{(\ell),i}\preceq\frac{4}{3}\bs U_{(\ell),i}$
        \item By Equation (56), (61) and (63) of \cite{jain2018accelerating}, we have
        \begin{equation}
        \begin{aligned}            
            \left(\bs U_{(\ell),i}\right)_{11}=&\frac{(1+c-c\delta_i\lambda_i)(q_{(\ell)}-c\delta_{(\ell)})-2\delta_{(\ell)}\lambda_i(q_{(\ell)}-c\delta_{(\ell)})+2\delta_{(\ell)}^2\lambda_i}{2(1-c^2+c\lambda_i(q_{(\ell)}+c\delta_{(\ell)}))}, \\
            \left(\bs U_{(\ell),i}\right)_{12}=&\frac{\left(1+c-\lambda_i(q_{(\ell)}+c\delta_{(\ell)})\right)(q_{(\ell)}-c\delta_{(\ell)})+\delta_{(\ell)}\lambda_i(q_{(\ell)}+c\delta_{(\ell)})}{2(1-c^2+c\lambda_i(q_{(\ell)}+c\delta_{(\ell)}))}, \\
            \left(\bs U_{(\ell),i}\right)_{22}=&\frac{(1+c-c\delta_i\lambda_i)(q_{(\ell)}-c\delta_{(\ell)})+2cq_{(\ell)}\delta_{(\ell)}\lambda_i}{2(1-c^2+c\lambda_i(q_{(\ell)}+c\delta_{(\ell)}))}.      
        \end{aligned}
        \end{equation}
        \item We have $\bs U_{(\ell),i}\preceq 16\bs U_{(\ell+1),i}$.
        \item We have $\bs Q_{(\ell),i}\preceq 20\bs Q_{(\ell+1),i}$.
    \end{enumerate}
\end{lemma}
\begin{proof}
    For Item 2,
    
    For Item 4, from Item 1, we know $\left(\bs U_{(\ell),i}\right)_{22}\geq\delta/2$. And from Item 3,
        \begin{equation}
        \begin{aligned}        
            \left(\bs U_{(\ell),i}\right)_{11}=&\left(\bs U_{(\ell),i}\right)_{22}-2\delta_{(\ell)}\lambda_i\left(\bs U_{(\ell),i}\right)_{22}+\delta_{(\ell)}^2\lambda_i \\
            \leq&\left(\bs U_{(\ell),i}\right)_{22}-2\delta_{(\ell)}\lambda_i\cdot\frac{\delta_{(\ell)}}{2}+\delta_{(\ell)}^2\lambda_i=\left(\bs U_{(\ell),i}\right)_{22}.
        \end{aligned}
        \end{equation}
    Thus, we have
        \begin{equation}
            \bs U_{(\ell),i}\preceq\left(\tr\bs U_{(\ell),i}\right)\bs I\leq 2\left(\bs U_{(\ell),i}\right)_{22}\bs I.
        \end{equation}

    For Item 5, since parameter choice procedure implies that $\left(\bs U_{(\ell),i}\right)_{22}\lambda_i\leq\frac{1}{4}$, we have
    \begin{equation}
        1\leq\frac{1}{1-\left(\bs U_{(\ell),i}\right)_{22}\lambda_i}\leq\frac{4}{3}.
    \end{equation}
    Plugging this into \eqref{eq:calc-S} completes the proof.
        
    For Item 7, from Item 6, we split the numerator of $\bs U_{(\ell),i}$ into two parts, based on whether the term contains $\lambda_i$,
    \begin{equation}
    \begin{aligned}            
        \mathop{\mathrm{num}}\left(\bs U_{(\ell),i}\right)_{11}=&\underbrace{(1+c)(q_{(\ell)}-c\delta_{(\ell)})}_{\bs M_{11}}+\underbrace{\left[-c\delta_i(q_{(\ell)}-c\delta_{(\ell)})-2\delta_{(\ell)}(q_{(\ell)}-c\delta_{(\ell)})+2\delta_{(\ell)}^2\lambda_i\right]\lambda_i}_{\bs N_{11}}, \\
        \mathop{\mathrm{num}}\left(\bs U_{(\ell),i}\right)_{12}=&\underbrace{(1+c)(q_{(\ell)}-c\delta_{(\ell)})}_{\bs M_{12}}+\underbrace{\left[-(q_{(\ell)}+c\delta_{(\ell)})(q_{(\ell)}-c\delta_{(\ell)})+\delta_{(\ell)}(q_{(\ell)}+c\delta_{(\ell)})\right]\lambda_i}_{\bs N_{12}}, \\
        \mathop{\mathrm{num}}\left(\bs U_{(\ell),i}\right)_{22}=&\underbrace{(1+c)(q_{(\ell)}-c\delta_{(\ell)})}_{\bs M_{22}}+\underbrace{\left[-c\delta_i(q_{(\ell)}-c\delta_{(\ell)})+2cq\delta_{(\ell)}\right]\lambda_i}_{\bs N_{22}},
    \end{aligned}
    \end{equation}
    where $\mathop{\mathrm{num}}$ represents the numerator. Note that $\bs M=(1+c)(q_{(\ell)}-c\delta_{(\ell)})\begin{bmatrix}
            1 & 1 \\
            1 & 1
        \end{bmatrix}\succeq0$. Therefore, 
        \begin{equation}
        \begin{aligned}
            \bs U_{(\ell+1),i}=&\frac{\bs M/4+\bs N/16}{2(1-c^2+c\lambda_i(q_{(\ell)}/4+c\delta_{(\ell)}/4))} \\
            \succeq&\frac{\bs M/16+\bs N/16}{2(1-c^2+c\lambda_i(q_{(\ell)}+c\delta_{(\ell)}))}=\frac{1}{16}\bs U_{(\ell),i}.
        \end{aligned}
        \end{equation}
    Thus, $\bs U_{(\ell),i}\preceq16\bs U_{(\ell+1),i}$.
    
    For Item 8, parameter choice procedure implies that $\left(\bs U_{(\ell),i}\right)_{22}\lambda_i\leq\frac{1}{4}$. Thus, from Lemma~\ref{lemma:calc-S} and Item 7, we have
    \begin{equation}
    \begin{aligned}            
        \tilde{\bs Q}_{(\ell),i}=&\frac{1}{1-\left(\bs U_{(\ell),i}\right)_{22}\lambda_i}\bs U_{(\ell),i} \\
        \preceq&\frac{16}{1-\left(\bs U_{(\ell),i}\right)_{22}\lambda_i}\bs U_{(\ell),i}=\frac{16(1-\left(\bs U_{(\ell+1),i}\right)_{22}\lambda_i)}{1-\left(\bs U_{(\ell),i}\right)_{22}\lambda_i}\bs Q_{(\ell+1),i} \\
        \preceq&\frac{16(1-\left(\bs U_{(\ell),i}\right)_{22}\lambda_i/4)}{1-\left(\bs U_{(\ell),i}\right)_{22}\lambda_i}\bs Q_{(\ell+1),i}\preceq 20\bs Q_{(\ell+1),i},
    \end{aligned}
    \end{equation}
    where we uses that $\bs U_{(\ell),i}$ is PSD matrix and $\left(\bs U_{(\ell+1),i}\right)_{22}\geq\left(\bs U_{(\ell),i}\right)_{22}/4$.
\end{proof}

\subsubsection{Proof of Lemma~\ref{lemma:tilde-var-ub}}
\begin{proof}[Proof of Lemma~\ref{lemma:tilde-var-ub}]
    We aim to bound $\left\langle\tilde{\bs T},\tilde{\bs C}_n\right\rangle$ from above. By unrolling recursive definition of $\tilde{\bs C}_{t-1}$ in \eqref{eq:def-tilde-var}, we obtain
    \begin{equation}
    \begin{aligned}
        \tilde{\bs C}_n=&\tilde\cB_n\circ\tilde{\bs C}_{n-1}+\sigma^2\begin{bmatrix}
        \delta_n^2\bs S & \delta_n q_n\bs S \\
        \delta_n q_n\bs S & q_n^2\bs S
        \end{bmatrix} \\
        =&\sigma^2\sum_{s=1}^{n}{\tilde{\cB}_n\circ\cdots\circ\tilde{\cB}_{s+1}\circ\begin{bmatrix}
        \delta_s^2\bs S & \delta_s q_s\bs S \\
        \delta_s q_s\bs S & q_s^2\bs S
        \end{bmatrix}}.
    \end{aligned}
    \end{equation}
    Therefore, taking the inner product with $\tilde{\bs T}$ and using that $\tilde\cB_{s,i}=\bs A_{s,i}\otimes\bs A_{s,i}$, we get
    \begin{equation}
    \begin{aligned}
        \left\langle\tilde{\bs T},\tilde{\bs C}_n\right\rangle=\sigma^2\sum_{i=1}^{d}{t_{ii}\sum_{s=1}^{n}\left(\tilde{\cB}_{n,i}\circ\cdots\circ\tilde{\cB}_{s+1,i}\circ\begin{bmatrix}
        \delta_s^2\lambda_i & \delta_s q_s\lambda_i \\
        \delta_s q_s\lambda_i & q_s^2\lambda_i
        \end{bmatrix}\right)_{11}},
    \end{aligned}
    \end{equation}
    where $t_{ii}$ denotes the $i$-th diagonal element of $\bs T$. In the following, we will bound each term of the sum $\sum_{i=1}^{d}$ separately.
    
    Let $k^*=\max\left\{k:\lambda_k>\frac{16(1-c)\ln n}{(q-c\delta)K}\right\}$. For each $i$, define $\ell_i^*=\max\left\{\ell:\frac{\lambda_i}{4^{\ell-1}}>\frac{16(1-c)\ln n}{(q-c\delta)K}\right\}$. Note that $i\leq k^*$ implies $\ell_i^*\geq 1$.

    If $i\leq k^*$, we bound $\sum_{s=1}^n=\sum_{s=1}^{K\ell_i^*}+\sum_{s=K\ell_i^*+1}^n$, respectively.   
    \begin{equation}
    \begin{aligned}
        &\sum_{s=1}^{K\ell_i^*}{\left(\tilde{\cB}_{n,i}\circ\cdots\circ\tilde{\cB}_{s+1,i}\circ\begin{bmatrix}
            \delta_s^2\lambda_i & \delta_s q_s\lambda_i \\
            \delta_s q_s\lambda_i & q_s^2\lambda_i
        \end{bmatrix}\right)_{11}} \\
        =&\sum_{m=1}^{\ell_i^*}{\left(\tilde{\cB}_{n,i}\circ\cdots\circ\tilde{\cB}_{K(\ell_i^*+1)+1,i}\circ\tilde{\cB}_{(\ell_i^*+1),i}^K\circ\cdots\circ\tilde{\cB}_{(m+1),i}^K\circ\sum_{s=1}^{K}\tilde{\cB}_{(m),i}^{K-s}\circ\begin{bmatrix}
            \delta_{(m)}^2\lambda_i & \delta_{(m)} q_{(m)}\lambda_i \\
            \delta_{(m)} q_{(m)}\lambda_i & q_{(m)}^2\lambda_i
        \end{bmatrix}\right)_{11}} \\
        \stackrel{a}{\leq}&\sigma^2\sum_{m=1}^{\ell_i^*}{\left(\tilde{\cB}_{n,i}\circ\cdots\circ\tilde{\cB}_{K(\ell_i^*+1)+1,i}\circ\tilde{\cB}_{(\ell_i^*+1),i}^K\circ\cdots\circ\tilde{\cB}_{(m+1),i}^K\circ\bs Q_{(m),i}\right)_{11}} \\
        \stackrel{b}{\leq}&\sigma^2\sum_{m=1}^{\ell_i^*}{\left(\tilde{\cB}_{n,i}\circ\cdots\circ\tilde{\cB}_{K(\ell_i^*+1)+1,i}\circ\tilde{\cB}_{(\ell_i^*+1),i}^K\circ\cdots\circ\tilde{\cB}_{(m+1),i}^K\circ\left[\frac{8}{3}\left(\bs U_{(m),i}\right)_{22}\bs I\right]\right)_{11}} \\
        % \leq&\frac{8\sigma^2\left(\bs U_{(1),i}\right)_{22}}{3}\sum_{m=1}^{\ell-1}{\left(\bs A_{(\ell),i}^{n-1-K(\ell-1)}\bs A_{(\ell-1),i}^K\cdots\bs A_{(m+1),i}^K\left(\bs A_{(m+1),i}^K\right)^\top\cdots\left(\bs A_{(\ell-1),i}^K\right)^\top\left(\bs A_{(\ell),i}^{n-1-K(\ell-1)}\right)^\top\right)_{11}} \\
        % \leq&\frac{8\sigma^2\left(\bs U_{(1),i}\right)_{22}}{3}\sum_{m=1}^{\ell-1}{\underbrace{\left\Vert\bs A_{n,i}\cdots\bs A_{K(\ell_i^*+1)+1,i}\right\Vert^2}_{\text{Lemma~\ref{lemma:tf-prod-A-bound}}}} \\    
        \leq&\frac{8\left(\bs U_{(1),i}\right)_{22}}{3}\sum_{m=1}^{\ell-1}\left[\bs A_{n,i}\cdots\bs A_{K(\ell_i^*+1)+1,i}\bs A_{\ell_i^*,i}^K\cdots\bs A_{(m+1),i}^K\right. \\
        &\hspace{4cm}\left.\left(\bs A_{(m+1),i}^K\right)^\top\cdots\left(\bs A_{(\ell-1),i}^K\right)^\top\bs A_{K(\ell_i^*+1)+1,i}^\top\cdots\bs A_{n,i}^\top\right]_{22} \\
        \leq&\frac{8\left(\bs U_{(1),i}\right)_{22}}{3}\sum_{m=1}^{\ell-1}{\underbrace{\left\Vert\bs A_{n,i}\cdots\bs A_{K(\ell_i^*+1)+1,i}\right\Vert^2}_{\text{Lemma~\ref{lemma:tf-prod-A-bound}}}\underbrace{\left\Vert\bs A_{(\ell_i^*),i}^K\right\Vert^2}_{\text{Lemma~\ref{lemma:2-norm-bound}}}\cdots\underbrace{\left\Vert\bs A_{(m+1),i}^K\right\Vert^2}_{\text{Lemma~\ref{lemma:2-norm-bound}}}} \\
        \leq&\frac{8\left(\bs U_{(1),i}\right)_{22}}{3}\cdot\frac{16}{(1-c)^2}\cdot\frac{6}{n^4(\log_2 n)^2}\cdot\log_2 n \\
        \stackrel{c}{\leq}&\frac{\left(\bs U_{(1),i}\right)_{22}}{256n^2},\label{eq:tilde-var-uniform-term-1}
    \end{aligned}
    \end{equation}
    where $\stackrel{a}{\leq}$ uses the definition of $\bs Q_{(m),i}$, $\stackrel{b}{\leq}$ uses Lemma~\ref{lemma:U-S-prop}, and $\stackrel{c}{\leq}$ uses $n\geq 16$. For the second term, we have $\lambda_i/4^{\ell_i^*}\leq\frac{16(1-c)\ln n}{(q-c\delta)K}\leq\frac{(1-c)^2}{q-c\delta}$. Thus, we apply Lemma~\ref{lemma:tf-prod-A-noise-bound}:
    \begin{equation}
    \begin{aligned}
        &\sum_{s=K\ell_i^*+1}^n{\left(\tilde{\cB}_{n,i}\circ\cdots\circ\tilde{\cB}_{s+1,i}\circ\begin{bmatrix}
            \delta_s^2\lambda_i & \delta_s q_s\lambda_i \\
            \delta_s q_s\lambda_i & q_s^2\lambda_i
        \end{bmatrix}\right)_{22}} \\
        \leq&\sum_{s=K\ell_i^*+1}^n{\lambda_i\underbrace{\left\Vert\bs A_{n,i}\cdots\bs A_{s+1,i}\begin{bmatrix}
            \delta_s \\ q_s
        \end{bmatrix}\right\Vert^2}_{\text{Lemma~\ref{lemma:tf-prod-A-noise-bound}}}}\leq 8\sigma^2\sum_{s=K\ell_i^*+1}^n{\lambda_i\left(\frac{q_s-c\delta_s}{1-c}\right)^2} \\
        =&\frac{128\sigma^2}{15}\lambda_iK\left(\frac{q_{(\ell_i^*+1)}-c\delta_{(\ell_i^*+1)}}{1-c}\right)^2 \\
        =&\frac{128\sigma^2}{15}\left(\frac{K\lambda_i}{4^{\ell_i^*}}\cdot\frac{q-c\delta}{1-c}\right)\left(\frac{q_{(\ell_i^*+1)}-c\delta_{(\ell_i^*+1)}}{1-c}\right) \\
        \stackrel{a}{\leq}&\frac{128\sigma^2}{15}\cdot\frac{16\ln n}{K}\cdot4\left(\bs U_{(\ell_i^*+1),i}\right)_{22}\leq\frac{8192\sigma^2\ln n}{15K}\left(\bs U_{(1),i}\right)_{22},
    \end{aligned}
    \end{equation}
    where $\stackrel{a}{\leq}$ uses $\lambda_i/4^{\ell_i^*}\leq\frac{16(1-c)\ln n}{(q-c\delta)K}$ and from Lemma~\ref{lemma:U-S-prop},
    \begin{equation}
    \begin{aligned}        
        \left(\bs U_{(\ell_i^*+1),i}\right)_{22}=&\frac{\delta_{(\ell_i^*+1)}}{2}+\frac{(1+c)(q_{(\ell_i^*+1)}-\delta_{(\ell_i^*+1)})}{2\left(1-c^2+c\lambda_i(q_{(\ell_i^*+1)}+c\delta_{(\ell_i^*+1)})\right)} \\
        \geq&\frac{\delta_{(\ell_i^*+1)}}{2}+\frac{(1+c)(q_{(\ell_i^*+1)}-\delta_{(\ell_i^*+1)})}{2\left(1-c^2+\frac{c\lambda_i}{4^{\ell_i^*}}(q+c\delta)\right)} \\
        \geq&\frac{\delta_{(\ell_i^*+1)}}{2}+\frac{(1+c)(q_{(\ell_i^*+1)}-\delta_{(\ell_i^*+1)})}{2\left(1-c^2+\frac{c(1-c)^2(q+c\delta)}{(q-c\delta)}\right)} \\
        \geq&\frac{\delta_{(\ell_i^*+1)}}{2}+\frac{(1+c)(q_{(\ell_i^*+1)}-\delta_{(\ell_i^*+1)})}{2\left(1-c^2+\frac{c(1-c)^2(q+cq)}{(q-cq)}\right)} \\
        =&\frac{\delta_{(\ell_i^*+1)}}{2}+\frac{q_{(\ell_i^*+1)}-\delta_{(\ell_i^*+1)}}{2(1+c)(1-c)} \\
        \geq&\frac{\delta_{(\ell_i^*+1)}}{4}+\frac{q_{(\ell_i^*+1)}-\delta_{(\ell_i^*+1)}}{4(1-c)}=\frac{q_{(\ell_i^*+1)}-c\delta_{(\ell_i^*+1)}}{4(1-c)}. \\
    \end{aligned}
    \end{equation}
    
    If $i>k^*$, we have    
    \begin{equation}
    \begin{aligned}
        &\sum_{s=1}^n{\left(\tilde{\cB}_{n,i}\circ\cdots\circ\tilde{\cB}_{s+1,i}\circ\begin{bmatrix}
            \delta_s^2\lambda_i & \delta_s q_s\lambda_i \\
            \delta_s q_s\lambda_i & q_s^2\lambda_i
        \end{bmatrix}\right)_{22}} \\
        \stackrel{a}{\leq}&\sum_{s=1}^n{\lambda_i\underbrace{\left\Vert\bs A_{n,i}\cdots\bs A_{s+1,i}\begin{bmatrix}
            \delta_s \\ q_s
        \end{bmatrix}\right\Vert^2}_{\text{Lemma~\ref{lemma:tf-prod-A-noise-bound}}}}\leq 8\sigma^2\sum_{s=1}^n{\lambda_i\left(\frac{q_s-c\delta_s}{1-c}\right)^2} \\
        =&\frac{128}{15}\lambda_iK\left(\frac{q-c\delta}{1-c}\right)^2.
    \end{aligned}
    \end{equation}

    Finally, we have
    \begin{equation}
    \begin{aligned}        
        \left\langle\tilde{\bs T},\tilde{\bs C}_n\right\rangle=&\sigma^2\sum_{i=1}^{k^*}{t_{ii}\left(\frac{\left(\bs U_{(1),i}\right)_{22}}{256N^2}+\frac{8192\ln n}{15K}\left(\bs U_{(1),i}\right)_{22}\right)}+\sigma^2\sum_{i=k^*+1}^{d}t_{ii}\cdot\frac{128}{15}\lambda_i K\left(\frac{q-c\delta}{1-c}\right)^2 \\
        \leq&\sigma^2\left[\sum_{i=1}^{k^*}{\frac{547t_{ii}\ln n}{K\lambda_i}\left(\bs U_{(1),i}\right)_{22}\lambda_i}+\frac{128}{15}K\left(\frac{q-c\delta}{1-c}\right)^2\sum_{i=k^*+1}^{d}\lambda_i t_{ii}\right] \\
        \leq&\sigma^2\left[\left(\sum_{i=1}^{k^*}{\frac{547t_{ii}\ln n}{K\lambda_i}}\right)\left(\sum_{j=1}^{k^*}{\left(\bs U_{(1),j}\right)_{22}\lambda_j}\right)+\frac{128}{15}K\left(\frac{q-c\delta}{1-c}\right)^2\sum_{i=k^*+1}^{d}\lambda_i t_{ii}\right] \\
        \stackrel{a}{\leq}&\sigma^2\left[\sum_{i=1}^{k^*}{\frac{t_{ii}}{2K\lambda_i}}+\frac{128}{15}K\left(\frac{q-c\delta}{1-c}\right)^2\sum_{i=k^*+1}^{d}\lambda_i t_{ii}\right],
    \end{aligned}
    \end{equation}
    where $\stackrel{a}{\leq}$ uses $\sum_i x_iy_i\leq\sum x_i\sum_j y_j$ if $x_i,y_i\geq 0$, and from the parameter choice procedure, we have $\sum_{j=1}^{k^*}{\left(\bs U_{(1),j}\right)_{22}\lambda_j}\leq\frac{1}{1094\ln n}$.
\end{proof}
\subsubsection{Proof of Lemma~\ref{lemma:tilde-var-uniform-ub}}
We bound the noise of $\tilde{\bs C}_t$ of two consecutive stages.
\begin{lemma}\label{lemma:tilde-var-noise-bound-two-stages}
    Let $\ell\geq 2$. If $K(\ell-1)+1\leq t\leq K(\ell+1)$, we have
    \begin{equation}
        \sum_{s=K(\ell-1)+1}^{t}\tilde{\cB}_{t-1,i}\circ\cdots\circ\tilde{\cB}_{s+1,i}\circ\begin{bmatrix}
        \delta_s^2\lambda_i & \delta_s q_s\lambda_i \\
        \delta_s q_s\lambda_i & q_s^2\lambda_i
        \end{bmatrix}\preceq 20\bs Q_{(\ell+1),i}.\label{eq:tilde-var-noise-bound-two-stages}
    \end{equation}
\end{lemma}
\begin{proof}
    For $K(\ell-1)+1\leq t\leq K\ell+1$, we have $t$ belongs to the $\ell-1$-th stage. From the definition of $\bs Q_{(\ell)}$, we have
    \begin{equation}
    \begin{aligned}
        &\sum_{s=K(\ell-1)+1}^{t}\tilde{\cB}_{t-1,i}\circ\cdots\circ\tilde{\cB}_{s+1,i}\circ\begin{bmatrix}
        \delta_s^2\lambda_i & \delta_s q_s\lambda_i \\
        \delta_s q_s\lambda_i & q_s^2\lambda_i        \end{bmatrix} \\
        =&\sum_{s=K(\ell-1)+1}^{t}\tilde{\cB}_{(\ell)),i}^{t-K(\ell-1)}\circ\begin{bmatrix}
        \delta_s^2\lambda_i & \delta_s q_s\lambda_i \\
        \delta_s q_s\lambda_i & q_s^2\lambda_i
        \end{bmatrix} \\
        \preceq&\sum_{s=K(\ell-1)+1}^{\infty}\tilde{\cB}_{(\ell)),i}^{t-K(\ell-1)}\circ\begin{bmatrix}
        \delta_s^2\lambda_i & \delta_s q_s\lambda_i \\
        \delta_s q_s\lambda_i & q_s^2\lambda_i
        \end{bmatrix}\stackrel{a}{=}\bs Q_{(\ell-1),i} \\
        \stackrel{b}{\preceq}&20\bs Q_{(\ell),i},
    \end{aligned}
    \end{equation}
    where $\stackrel{a}{=}$ uses the definition of $\bs Q_{(\ell)}$ and $\stackrel{b}{\preceq}$ uses Lemma~\ref{lemma:U-S-prop}.

    For $K\ell+1\leq t\leq K(\ell+1)$, we prove by induction. The case where $t=K(\ell+1)$ has been proven. We suppose \eqref{eq:tilde-var-noise-bound-two-stages} holds. Note that by the definition of $\bs Q_{(\ell),i}$, we have
    \begin{equation}
        \bs Q_{(\ell),i}=(\cI-\tilde{\cB}_{(\ell),i})^{-1}\circ\begin{bmatrix}
            \delta_{(\ell)}^2\lambda_i & \delta_{(\ell)} q_{(\ell)}\lambda_i \\
            \delta_{(\ell)} q_{(\ell)}\lambda_i & q_{(\ell)}^2\lambda_i
        \end{bmatrix}\implies\tilde{\cB}_{(\ell),i}\circ\bs Q_{(\ell),i}=\bs Q_{(\ell),i}-\begin{bmatrix}
            \delta_{(\ell)}^2\lambda_i & \delta_{(\ell)} q_{(\ell)}\lambda_i \\
            \delta_{(\ell)} q_{(\ell)}\lambda_i & q_{(\ell)}^2\lambda_i
        \end{bmatrix}.
    \end{equation}
    Therefore, for $t+1$, we have
    \begin{equation}
    \begin{aligned}
        &\sum_{s=K(\ell-1)+1}^{t+1}\tilde{\cB}_{t,i}\circ\cdots\circ\tilde{\cB}_{s+1,i}\circ\begin{bmatrix}
        \delta_s^2\lambda_i & \delta_s q_s\lambda_i \\
        \delta_s q_s\lambda_i & q_s^2\lambda_i        \end{bmatrix} \\
        =&\tilde{\cB}_{(\ell),i}\circ\sum_{s=K(\ell-1)+1}^{t}\tilde{\cB}_{t,i}\circ\cdots\circ\tilde{\cB}_{s+1,i}\circ\begin{bmatrix}
        \delta_s^2\lambda_i & \delta_s q_s\lambda_i \\
        \delta_s q_s\lambda_i & q_s^2\lambda_i        \end{bmatrix}+\begin{bmatrix}
        \delta_s^2\lambda_i & \delta_s q_s\lambda_i \\
        \delta_s q_s\lambda_i & q_s^2\lambda_i        \end{bmatrix} \\
        \stackrel{a}{\preceq}&\tilde{\cB}_{(\ell),i}\circ\left(20\bs Q_{(\ell),i}\right)+\begin{bmatrix}
        \delta_s^2\lambda_i & \delta_s q_s\lambda_i \\
        \delta_s q_s\lambda_i & q_s^2\lambda_i        \end{bmatrix} \\
        =&20\bs Q_{(\ell),i}-19\begin{bmatrix}
        \delta_s^2\lambda_i & \delta_s q_s\lambda_i \\
        \delta_s q_s\lambda_i & q_s^2\lambda_i        \end{bmatrix}\preceq 20\bs Q_{(\ell),i}.
    \end{aligned}
    \end{equation}
    By induction, the lemma holds.
\end{proof}

Now, we are ready for the proof.
\begin{proof}[Proof of Lemma~\ref{lemma:tilde-var-uniform-ub}]
    Our goal is to show that for $1\leq t\leq n$, we have
    \begin{equation}        
        \bbE\left[\widehat{\bs G}_t\otimes\widehat{\bs G}_t\right]\circ\tilde{\bs C}_{t-1}\preceq\frac{1}{2}\sigma^2\begin{bmatrix}
        \delta_t^2\bs S & \delta_t q_t\bs S \\
        \delta_t q_t\bs S & q_t^2\bs S
        \end{bmatrix}.
    \end{equation}
    Note that by Lemma~\ref{lemma:op-prop}, we have $\bbE\left[\widehat{\bs G}_t\otimes\widehat{\bs G}_t\right]\circ\tilde{\bs C}_{t-1}\preceq\psi\left\langle\begin{bmatrix}
                \bs O & \bs O \\
                \bs O & \bs S
            \end{bmatrix},\tilde{\bs C}_{t-1}\right\rangle\begin{bmatrix}
            \delta_t^2\bs S & \delta_t q_t\bs S \\
            \delta_t q_t\bs S & q_t^2\bs S
            \end{bmatrix}$. Therefore, we only have to show that for all $1\leq i\leq d$,
    \begin{equation}
        \psi\left\langle\begin{bmatrix}
            \bs O & \bs O \\
            \bs O & \bs S
        \end{bmatrix},\tilde{\bs C}_{t-1}\right\rangle\leq \frac{1}{2}\sigma^2.
    \end{equation}
    From the recursive definition of $\tilde{\bs C}_{t-1}$ in \eqref{eq:def-tilde-var}, we have:
    \begin{equation}
    \begin{aligned}
        \tilde{\bs C}_{t-1}=&\tilde\cB_{t-1}\circ\tilde{\bs C}_{t-2}+\sigma^2\begin{bmatrix}
        \delta_{t-1}^2\bs S & \delta_{t-1} q_{t-1}\bs S \\
        \delta_{t-1} q_{t-1}\bs S & q_{t-1}^2\bs S
        \end{bmatrix} \\
        =&\sigma^2\sum_{s=1}^{t-1}{\tilde{\cB}_{t-1}\circ\cdots\circ\tilde{\cB}_{s+1}\circ\begin{bmatrix}
        \delta_s^2\bs S & \delta_s q_s\bs S \\
        \delta_s q_s\bs S & q_s^2\bs S
        \end{bmatrix}}.
    \end{aligned}
    \end{equation}
    Therefore, taking the inner product with $\begin{bmatrix}
        \bs O & \bs O \\
        \bs O & \bs S
    \end{bmatrix}$ and using that $\tilde\cB_{s,i}=\bs A_{s,i}\otimes\bs A_{s,i}$, we get
    \begin{equation}
    \begin{aligned}
        \left\langle\begin{bmatrix}
            \bs O & \bs O \\
            \bs O & \bs S
        \end{bmatrix},\tilde{\bs C}_{t-1}\right\rangle=\sigma^2\sum_{i=1}^{d}{\lambda_i\sum_{s=1}^{t-1}\left(\tilde{\cB}_{t-1,i}\circ\cdots\circ\tilde{\cB}_{s+1,i}\circ\begin{bmatrix}
        \delta_s^2\lambda_i & \delta_s q_s\lambda_i \\
        \delta_s q_s\lambda_i & q_s^2\lambda_i
        \end{bmatrix}\right)_{22}}.
    \end{aligned}
    \end{equation}
    
    Suppose $t-1$ belongs to the $\ell$-th stage, namely, $K(\ell-1)+1\leq t-1\leq K\ell$. For each $i$, define $\ell_i^*=\max\left\{\ell:\frac{\lambda_i}{4^{\ell-1}}>\frac{16(1-c)\ln n}{(q-c\delta)K}\right\}$. 
    
    If $\ell\leq\ell_i^*+1$, we bound $\sum_{s=1}^{t-1}=\sum_{s=1}^{K(\ell-1)}+\sum_{s=K(\ell-1)+1}^{t-1}$, respectively. For the first term, we have
    \begin{equation}
    \begin{aligned}
        &\sum_{s=1}^{K(\ell-1)}{\left(\tilde{\cB}_{t-1,i}\circ\cdots\circ\tilde{\cB}_{s+1,i}\circ\begin{bmatrix}
            \delta_s^2\lambda_i & \delta_s q_s\lambda_i \\
            \delta_s q_s\lambda_i & q_s^2\lambda_i
        \end{bmatrix}\right)_{22}} \\
        =&\sum_{m=1}^{\ell-1}{\left(\tilde{\cB}_{(\ell),i}^{t-1-K(\ell-1)}\circ\tilde{\cB}_{(\ell-1),i}^K\circ\cdots\circ\tilde{\cB}_{(m+1),i}^K\circ\sum_{s=1}^{K}\tilde{\cB}_{(m),i}^{K-s}\circ\begin{bmatrix}
            \delta_{(m)}^2\lambda_i & \delta_{(m)} q_{(m)}\lambda_i \\
            \delta_{(m)} q_{(m)}\lambda_i & q_{(m)}^2\lambda_i
        \end{bmatrix}\right)_{22}} \\
        \stackrel{a}{\preceq}&\sum_{m=1}^{\ell-1}{\left(\tilde{\cB}_{(\ell),i}^{t-1-K(\ell-1)}\circ\tilde{\cB}_{(\ell-1),i}^K\circ\cdots\circ\tilde{\cB}_{(m+1),i}\circ\bs Q_{(m),i}\right)_{22}} \\
        \stackrel{b}{\preceq}&\sum_{m=1}^{\ell-1}{\left(\tilde{\cB}_{(\ell),i}^{t-1-K(\ell-1)}\circ\tilde{\cB}_{(\ell-1),i}^K\circ\cdots\circ\tilde{\cB}_{(m+1),i}\circ\left[\frac{8}{3}\left(\bs U_{(m),i}\right)_{22}\bs I\right]\right)_{22}} \\
        \leq&\frac{8\left(\bs U_{(1),i}\right)_{22}}{3}\sum_{m=1}^{\ell-1}{\left(\bs A_{(\ell),i}^{t-1-K(\ell-1)}\bs A_{(\ell-1),i}^K\cdots\bs A_{(m+1),i}^K\left(\bs A_{(m+1),i}^K\right)^\top\cdots\left(\bs A_{(\ell-1),i}^K\right)^\top\left(\bs A_{(\ell),i}^{t-1-K(\ell-1)}\right)^\top\right)_{22}} \\
        \leq&\frac{8\left(\bs U_{(1),i}\right)_{22}}{3}\sum_{m=1}^{\ell-1}{\underbrace{\left\Vert\bs A_{(\ell),i}^{t-1-K(\ell-1)}\right\Vert^2}_{\text{Lemma~\ref{lemma:2-norm-uniform-bound}}}\underbrace{\left\Vert\bs A_{(\ell-1),i}^K\right\Vert^2}_{\text{Lemma~\ref{lemma:2-norm-bound}}}\cdots\underbrace{\left\Vert\bs A_{(m+1),i}^K\right\Vert^2}_{\text{Lemma~\ref{lemma:2-norm-bound}}}} \\
        \leq&\frac{8\left(\bs U_{(1),i}\right)_{22}}{3}\cdot6K^2\cdot\frac{6}{n^4(\log_2 n)^2}\cdot\log_2 n \\
        \stackrel{c}{\leq}&\frac{3\left(\bs U_{(1),i}\right)_{22}}{2N^2},%\label{eq:tilde-var-uniform-term-1}
    \end{aligned}
    \end{equation}
    where $\stackrel{a}{\preceq}$ uses the definition of $\bs Q_{(m),i}$, $\stackrel{b}{\preceq}$ uses $\bs Q_{(m),i}\preceq\frac{4}{3}\bs U_{(m),i}\preceq\frac{8}{3}\left(\bs U_{(m),i}\right)_{22}\bs I$ from Lemma~\ref{lemma:U-S-prop}, and $\stackrel{c}{\leq}$ uses $n\geq 16$. For the second term, we apply Lemma~\ref{lemma:tilde-var-noise-bound-two-stages},
    \begin{equation}
    \begin{aligned}
        &\sum_{s=K(\ell-1)+1}^{t-1}{\left(\tilde{\cB}_{t-1,i}\circ\cdots\circ\tilde{\cB}_{s+1,i}\circ\begin{bmatrix}
            \delta_s^2\lambda_i & \delta_s q_s\lambda_i \\
            \delta_s q_s\lambda_i & q_s^2\lambda_i
        \end{bmatrix}\right)_{22}} \\
        \leq&\left(20\bs Q_{(\ell),i}\right)_{22}\leq\frac{80}{3}\left(\bs U_{(\ell),i}\right)_{22}\leq\frac{80}{3}\left(\bs U_{(1),i}\right)_{22}. \\
    \end{aligned}
    \end{equation}
    Thus, we have
    \begin{equation}        
        \left\langle\begin{bmatrix}
            \bs O & \bs O \\
            \bs O & \bs S
        \end{bmatrix},\tilde{\bs C}_{t-1}\right\rangle\leq\sigma^2\sum_{i=1}^{d}\left(\frac{3}{2N^2}+\frac{80}{3}\right)\left(\bs U_{(1),i}\right)_{22}\lambda_i\leq\frac{1}{2}\sigma^2.
    \end{equation}

    If $\ell>\ell_i^*+1$, we have $\lambda_i/4^{\ell-1}\in I_1$. We bound $\sum_{s=1}^{t-1}=\sum_{s=1}^{K\ell_i^*}+\sum_{s=K\ell_i^*+1}^{t-1}$, respectively. The bound of the first term parallels \eqref{eq:tilde-var-uniform-term-1}:    
    \begin{equation}
    \begin{aligned}
        &\sigma^2\sum_{s=1}^{K\ell_i^*}{\left(\tilde{\cB}_{t-1,i}\circ\cdots\circ\tilde{\cB}_{s+1,i}\circ\begin{bmatrix}
            \delta_s^2\lambda_i & \delta_s q_s\lambda_i \\
            \delta_s q_s\lambda_i & q_s^2\lambda_i
        \end{bmatrix}\right)_{22}} \\
        =&\sigma^2\sum_{m=1}^{\ell_i^*}{\left(\tilde{\cB}_{t-1,i}\circ\cdots\circ\tilde{\cB}_{K(\ell_i^*+1)+1,i}\circ\tilde{\cB}_{(\ell_i^*+1),i}^K\circ\cdots\circ\tilde{\cB}_{(m+1),i}^K\circ\sum_{s=1}^{K}\tilde{\cB}_{(m),i}^{K-s}\circ\begin{bmatrix}
            \delta_{(m)}^2\lambda_i & \delta_{(m)} q_{(m)}\lambda_i \\
            \delta_{(m)} q_{(m)}\lambda_i & q_{(m)}^2\lambda_i
        \end{bmatrix}\right)_{22}} \\
        \stackrel{a}{\preceq}&\sigma^2\sum_{m=1}^{\ell_i^*}{\left(\tilde{\cB}_{t-1,i}\circ\cdots\circ\tilde{\cB}_{K(\ell_i^*+1)+1,i}\circ\tilde{\cB}_{(\ell_i^*+1),i}^K\circ\cdots\circ\tilde{\cB}_{(m+1),i}^K\circ\bs Q_{(m),i}\right)_{22}} \\
        \stackrel{b}{\preceq}&\sigma^2\sum_{m=1}^{\ell_i^*}{\left(\tilde{\cB}_{t-1,i}\circ\cdots\circ\tilde{\cB}_{K(\ell_i^*+1)+1,i}\circ\tilde{\cB}_{(\ell_i^*+1),i}^K\circ\cdots\circ\tilde{\cB}_{(m+1),i}^K\circ\left[\frac{8}{3}\left(\bs U_{(m),i}\right)_{22}\bs I\right]\right)_{22}} \\
        \leq&\frac{8\sigma^2\left(\bs U_{(1),i}\right)_{22}}{3}\sum_{m=1}^{\ell-1}\left[\bs A_{t-1,i}\cdots\bs A_{K(\ell_i^*+1)+1,i}\bs A_{\ell_i^*,i}^K\cdots\bs A_{(m+1),i}^K\right. \\
        &\hspace{4cm}\left.\left(\bs A_{(m+1),i}^K\right)^\top\cdots\left(\bs A_{(\ell-1),i}^K\right)^\top\bs A_{K(\ell_i^*+1)+1,i}^\top\cdots\bs A_{t-1,i}^\top\right]_{22} \\
        \leq&\frac{8\sigma^2\left(\bs U_{(1),i}\right)_{22}}{3}\sum_{m=1}^{\ell-1}{\underbrace{\left\Vert\bs A_{t-1,i}\cdots\bs A_{K(\ell_i^*+1)+1,i}\right\Vert^2}_{\text{Lemma~\ref{lemma:tf-prod-A-bound}}}\underbrace{\left\Vert\bs A_{(\ell_i^*),i}^K\right\Vert^2}_{\text{Lemma~\ref{lemma:2-norm-bound}}}\cdots\underbrace{\left\Vert\bs A_{(m+1),i}^K\right\Vert^2}_{\text{Lemma~\ref{lemma:2-norm-bound}}}} \\
        \leq&\frac{8\sigma^2\left(\bs U_{(1),i}\right)_{22}}{3}\cdot\frac{16}{(1-c)^2}\cdot\frac{6}{n^4(\log_2 n)^2}\cdot\log_2 n \\
        \stackrel{c}{\leq}&\frac{\sigma^2\left(\bs U_{(1),i}\right)_{22}}{256n^2},\label{eq:tilde-var-uniform-case-2}
    \end{aligned}
    \end{equation}
    where $\stackrel{a}{\preceq}$ uses the definition of $\bs Q_{(m),i}$, $\stackrel{b}{\preceq}$ uses Lemma~\ref{lemma:U-S-prop}, and $\stackrel{c}{\leq}$ uses $n\geq 16$. For the second term, we have $\lambda_i/4^{\ell_i^*}\leq\frac{16(1-c)\ln n}{(q-c\delta)K}\leq\frac{(1-c)^2}{q-c\delta}$. Thus, we apply Lemma~\ref{lemma:tf-prod-A-noise-bound}:
    \begin{equation}
    \begin{aligned}
        &\sigma^2\sum_{s=K\ell_i^*+1}^{t-1}{\left(\tilde{\cB}_{t-1,i}\circ\cdots\circ\tilde{\cB}_{s+1,i}\circ\begin{bmatrix}
            \delta_s^2\lambda_i & \delta_s q_s\lambda_i \\
            \delta_s q_s\lambda_i & q_s^2\lambda_i
        \end{bmatrix}\right)_{22}} \\
        \leq&\sigma^2\sum_{s=K\ell_i^*+1}^{t-1}{\lambda_i\underbrace{\left\Vert\bs A_{t-1,i}\cdots\bs A_{s+1,i}\begin{bmatrix}
            \delta_s \\ q_s
        \end{bmatrix}\right\Vert^2}_{\text{Lemma~\ref{lemma:tf-prod-A-noise-bound}}}}\leq 8\sigma^2\sum_{s=K\ell_i^*+1}^{t-1}{\lambda_i\left(\frac{q_s-c\delta_s}{1-c}\right)^2} \\
        =&\frac{128\sigma^2}{15}\lambda_iK\left(\frac{q_{(\ell_i^*+1)}-c\delta_{(\ell_i^*+1)}}{1-c}\right)^2 \\
        =&\frac{128\sigma^2}{15}\left(\frac{K\lambda_i}{4^{\ell_i^*}}\cdot\frac{q-c\delta}{1-c}\right)\left(\frac{q_{(\ell_i^*+1)}-c\delta_{(\ell_i^*+1)}}{1-c}\right) \\
        \stackrel{a}{\leq}&\frac{128\sigma^2}{15}\cdot 16\ln n\cdot4\left(\bs U_{(\ell_i^*+1),i}\right)_{22}\leq\frac{8192\sigma^2\ln n}{15}\left(\bs U_{(1),i}\right)_{22},
    \end{aligned}
    \end{equation}
    where $\stackrel{a}{\leq}$ uses $\lambda_i/4^{\ell_i^*}\leq\frac{16(1-c)\ln n}{(q-c\delta)K}$ and from Lemma~\ref{lemma:U-S-prop},
    \begin{equation}
    \begin{aligned}        
        \left(\bs U_{(\ell_i^*+1),i}\right)_{22}=&\frac{\delta_{(\ell_i^*+1)}}{2}+\frac{(1+c)(q_{(\ell_i^*+1)}-\delta_{(\ell_i^*+1)})}{2\left(1-c^2+c\lambda_i(q_{(\ell_i^*+1)}+c\delta_{(\ell_i^*+1)})\right)} \\
        \geq&\frac{\delta_{(\ell_i^*+1)}}{2}+\frac{(1+c)(q_{(\ell_i^*+1)}-\delta_{(\ell_i^*+1)})}{2\left(1-c^2+\frac{c\lambda_i}{4^{\ell_i^*}}(q+c\delta)\right)} \\
        \geq&\frac{\delta_{(\ell_i^*+1)}}{2}+\frac{(1+c)(q_{(\ell_i^*+1)}-\delta_{(\ell_i^*+1)})}{2\left(1-c^2+\frac{c(1-c)^2(q+c\delta)}{(q-c\delta)}\right)} \\
        \geq&\frac{\delta_{(\ell_i^*+1)}}{2}+\frac{(1+c)(q_{(\ell_i^*+1)}-\delta_{(\ell_i^*+1)})}{2\left(1-c^2+\frac{c(1-c)^2(q+cq)}{(q-cq)}\right)} \\
        =&\frac{\delta_{(\ell_i^*+1)}}{2}+\frac{q_{(\ell_i^*+1)}-\delta_{(\ell_i^*+1)}}{2(1+c)(1-c)} \\
        \geq&\frac{\delta_{(\ell_i^*+1)}}{4}+\frac{q_{(\ell_i^*+1)}-\delta_{(\ell_i^*+1)}}{4(1-c)}=\frac{q_{(\ell_i^*+1)}-c\delta_{(\ell_i^*+1)}}{4(1-c)}. \\
    \end{aligned}
    \end{equation}
    
    Thus, we have
    \begin{equation}
    \begin{aligned}        
        \left\langle\begin{bmatrix}
            \bs O & \bs O \\
            \bs O & \bs S
        \end{bmatrix},\tilde{\bs C}_{t-1}\right\rangle\leq&\sigma^2\sum_{i=1}^{d}\left(\frac{1}{256N^2}+\frac{8192\ln n}{15}\right)\left(\bs U_{(1),i}\right)_{22}\lambda_i \\
        \leq& 547\sigma^2\ln n\sum_{i=1}^{d}\left(\bs U_{(1),i}\right)_{22}\lambda_i\leq\frac{1}{2}\sigma^2.
    \end{aligned}
    \end{equation}
\end{proof}

\subsection{Bias Upper Bound}\label{sec:asgd-bias-ub}
\subsubsection{Proof of Lemma~\ref{lemma:tilde-bias-ub}}
\begin{proof}[Proof of Lemma~\ref{lemma:tilde-bias-ub}]
    Recall the definition of $\tilde{\bs B}_n$ and $\tilde{\cB}_t$, we have
    \begin{equation}
        \tilde{\bs B}_n=\tilde{\cB}_n\circ\tilde{\cB}_{n-1}\circ\cdots\circ\tilde{\cB}_1\circ\bs B_0=\left(\prod_{t=1}^n\bs A_t\begin{bmatrix}
            \bs w_0-\bs w^* \\\bs w_0-\bs w^*
        \end{bmatrix}\right)\left(\prod_{t=1}^n\bs A_t\begin{bmatrix}
            \bs w_0-\bs w^* \\\bs w_0-\bs w^*
        \end{bmatrix}\right)^\top.
    \end{equation}
    Note that $\bs A_t$ is block-diagonal, we have
    \begin{equation}
        \left(\prod_{t=1}^n\bs A_t\begin{bmatrix}
            \bs w_0-\bs w^* \\ \bs w_0-\bs w^*
        \end{bmatrix}\right)_i=\left(\bs w_{0,i}-\bs w_i^*\right)\prod_{t=1}^n\bs A_{t,i}\begin{bmatrix}
            1 \\ 1
        \end{bmatrix}.
    \end{equation}
    For $i\leq k^*$, we have $\lambda_i\in I_1$. Let $\ell^*=\max\left\{\ell:\frac{\lambda_i}{4^{\ell-1}}>\frac{16(1-c)\ln n}{(q-c\delta)K}\right\}$, and note that for $\ell\geq\ell^*$, $\lambda_i/4^{\ell-1}\leq\frac{(1-c)^2}{q-c\delta}$. Therefore, we have
    \begin{equation}
    \begin{aligned}
        \left(\prod_{t=1}^n\bs A_{t,i}\begin{bmatrix}
            1 \\ 1
        \end{bmatrix}\right)_1^2\leq&\underbrace{\left\Vert\prod_{t=K\ell^*+1}^n\bs A_{t}\right\Vert^2}_{\text{Lemma~\ref{lemma:tf-prod-A-bound}}}\underbrace{\left\Vert\bs A_{(\ell-1)}^K\right\Vert^2\cdots\left\Vert\bs A_{(1)}^K\right\Vert^2}_{\text{Lemma~\ref{lemma:2-norm-bound}}}\left\Vert\begin{bmatrix}
            1 \\ 1
        \end{bmatrix}\right\Vert^2 \\
        \leq&\frac{16}{(1-c)^2}\cdot\frac{6}{n^4(\log_2 n)^2}\cdot 2 \\
        \stackrel{a}{\leq}&\frac{1}{8n^2(\log_2 n)^4}.\label{eq:tilde-bias-upper-bound-term-1}
    \end{aligned}
    \end{equation}
    where $\stackrel{a}{\leq}$ uses $K(1-c)\geq 16\ln n$. For $i>k^*$, from Lemma~\ref{lemma:tf-prod-A-bias-bound} we have
    \begin{equation}
    \begin{aligned}
        \left(\prod_{t=1}^n\bs A_{t,i}\begin{bmatrix}
            1 \\ 1
        \end{bmatrix}\right)_1^2\leq\left\Vert\prod_{t=1}^n\bs A_{t,i}\begin{bmatrix}
            1 \\ 1
        \end{bmatrix}\right\Vert^2\leq 4.\label{eq:tilde-bias-upper-bound-term-2}
    \end{aligned}
    \end{equation}
    Consider the following decomposition:
    \begin{equation}
        \prod_{t=1}^n\bs A_t\begin{bmatrix}
            \bs w_0-\bs w^* \\ \bs w_0-\bs w^*
        \end{bmatrix}=\begin{bmatrix}
            \bs\xi_1 \\ \bs O
        \end{bmatrix}+\begin{bmatrix}
            \bs O \\ \bs\xi_2
        \end{bmatrix},
    \end{equation}
    where $\bs\xi_1\in\bbR^{k^*}$ and $\bs\xi_2\in\bbR^{d-k^*}$. Then \eqref{eq:tilde-bias-upper-bound-term-1} and \eqref{eq:tilde-bias-upper-bound-term-1} implies that
    \begin{equation}
        \left(\begin{bmatrix}
            \bs\xi_1 \\ \bs O
        \end{bmatrix}\right)_i^2\leq\frac{(\bs w_0-\bs w^*)_i^2}{8n^2(\log_2 n)^4}, \quad\left(\begin{bmatrix}
            \bs O \\ \bs\xi_2
        \end{bmatrix}\right)_i^2\leq4(\bs w_0-\bs w^*)_i^2.
    \end{equation}
    Note that $\bs T\preceq 2\bs T_{0:k^*}+2\bs T_{k^*:\infty}$. Then we have
    \begin{equation}
    \begin{aligned}
        \left\langle\tilde{\bs T},\tilde{\bs B}_n\right\rangle\leq& 2\left\langle\bs T_{0:k^*},\tilde{\bs B}_n\right\rangle+2\left\langle\bs T_{k^*:\infty},\tilde{\bs B}_n\right\rangle \\
        =&2\left\Vert\begin{bmatrix}
            \bs\xi_1 \\ \bs O
        \end{bmatrix}\right\Vert_{\bs T_{0:k^*}}^2+2\left\Vert\begin{bmatrix}
            \bs O \\ \bs\xi_2
        \end{bmatrix}\right\Vert_{\bs T_{k^*:\infty}}^2 \\
        \leq&\max_{\bs w\in S(\bs w_0-\bs w^*)}\frac{\left\Vert\bs w\right\Vert_{\bs T_{0:k^*}}^2}{8n^2(\log_2 n)^4}+4\left\Vert\bs w\right\Vert_{\bs T_{k^*:\infty}}^2.
    \end{aligned}
    \end{equation}
    This completes the proof.
\end{proof}
\subsubsection{Proof of Lemma~\ref{lemma:bias-1-ub}}
We first analyze $\left\langle\begin{bmatrix}
    \bs O & \bs O \\
    \bs O & \bs S
\end{bmatrix},\tilde{\bs B}_t\right\rangle$.
\begin{lemma}\label{lemma:tilde-bias-uniform-ub}
    For $t\leq K$, we have
    \begin{equation}
        \left\langle\begin{bmatrix}
            \bs O & \bs O \\
            \bs O & \bs S
        \end{bmatrix},\tilde{\bs B}_t\right\rangle\leq 4\sum_{i=1}^{d}\lambda_i\left(\bs w_i^*\right)^2.
    \end{equation}
    For $t>K$, we have
    \begin{equation}
        \left\langle\begin{bmatrix}
            \bs O & \bs O \\
            \bs O & \bs S
        \end{bmatrix},\tilde{\bs B}_t\right\rangle\leq\frac{36}{n^2\left(\log_2 n\right)^4}\sum_{i=1}^{k^*}\lambda_i\left(\bs w_i^*\right)^2+4\sum_{i=k^*+1}^d\lambda_i\left(\bs w_i^*\right)^2.
    \end{equation}
\end{lemma}
\begin{proof}
    Note that $\tilde{\bs B}_t$ is block-diagonal, we have
    \begin{equation}
        \left\langle\begin{bmatrix}
            \bs O & \bs O \\
            \bs O & \bs S
        \end{bmatrix},\tilde{\bs B}_t\right\rangle=\sum_{i=1}^d{\lambda_i\left(\bs w_i^*\right)^2\left(\prod_{s=1}^t\bs A_{s,i}\begin{bmatrix}
            1 \\ 1
        \end{bmatrix}\right)_1^2}.
    \end{equation}
    For $t\leq K$, $s\leq t$ implies $s$ belongs to the first stage. Thus, $\bs A_{s,i}=\bs A_{(\ell).i}=\bs A(\lambda_i)$. By Lemma~\ref{lemma:prod-A-bias-bound},
    \begin{equation}
        \left\vert\left(\prod_{s=1}^t\bs A_{s,i}\begin{bmatrix}
            1 \\ 1
        \end{bmatrix}\right)_1\right\vert=\left\vert\left(\bs A^t(\lambda_i)\begin{bmatrix}
            1 \\ 1
        \end{bmatrix}\right)_2\right\vert\leq2.
    \end{equation}
    Therefore,
    \begin{equation}
    \begin{aligned}        
        \left\langle\begin{bmatrix}
            \bs O & \bs O \\
            \bs O & \bs S
        \end{bmatrix},\tilde{\bs B}_t\right\rangle=&\sum_{i=1}^{d}{\lambda_i\left(\bs w_i^*\right)^2\left(\prod_{s=1}^t\bs A_{s,i}\begin{bmatrix}
            1 \\ 1
        \end{bmatrix}\right)_1^2} \\
        \leq&4\sum_{i=1}^{d}{\lambda_i\left(\bs w_i^*\right)^2}
    \end{aligned}
    \end{equation}
    
    For $t>K$, suppose $t$ belongs to the $\ell$-th stage, we have $\ell\geq2$. Since $i>k^*$ implies that $\lambda_i\in I_1$. Let $\ell_i^*=\max\left\{\ell:\frac{\lambda_i}{4^{\ell-1}}>\frac{16(1-c)\ln n}{(q-c\delta)K}\right\}$. If $\ell<\ell_i^*$, by applying Lemma~\ref{lemma:2-norm-uniform-bound} and Lemma~\ref{lemma:2-norm-bound}, we have
    \begin{equation}
    \begin{aligned}        
        \left(\prod_{s=1}^t\bs A_{s,i}\begin{bmatrix}
            1 \\ 1
        \end{bmatrix}\right)_1^2\leq&\underbrace{\left\Vert\bs A_{(\ell)}^{t-K(\ell-1)}\right\Vert^2}_{\text{Lemma~\ref{lemma:2-norm-uniform-bound}}}\underbrace{\left\Vert\bs A_{(\ell-1)}^K\right\Vert^2\cdots\left\Vert\bs A_{(1)}^K\right\Vert^2}_{\text{Lemma~\ref{lemma:2-norm-bound}}}\left\Vert\begin{bmatrix}
            1 \\ 1
        \end{bmatrix}\right\Vert^2 \\
        \leq& 6K^2\cdot\left(\frac{\sqrt 6}{n^2\log_2 n}\right)^{2(\ell-1)}\leq\frac{36}{n^2\left(\log_2 n\right)^4}.
    \end{aligned}
    \end{equation}
    If $\ell\geq\ell_i^*$, by applying Lemma~\ref{lemma:tf-prod-A-bound} and Lemma~\ref{lemma:2-norm-bound}, we have
    \begin{equation}
    \begin{aligned}
        \left(\prod_{s=1}^t\bs A_{s,i}\begin{bmatrix}
            1 \\ 1
        \end{bmatrix}\right)_1^2\leq&\underbrace{\left\Vert\prod_{t=K\ell_i^*+1}^n\bs A_{t}\right\Vert^2}_{\text{Lemma~\ref{lemma:tf-prod-A-bound}}}\underbrace{\left\Vert\bs A_{(\ell_i^*-1)}^K\right\Vert^2\cdots\left\Vert\bs A_{(1)}^K\right\Vert^2}_{\text{Lemma~\ref{lemma:2-norm-bound}}}\left\Vert\begin{bmatrix}
            1 \\ 1
        \end{bmatrix}\right\Vert^2 \\
        \leq&\frac{16}{(1-c)^2}\cdot\frac{6}{n^4(\log_2 n)^2}\cdot 2 \\
        \stackrel{a}{\leq}&\frac{1}{8n^2(\log_2 n)^4}.
    \end{aligned}
    \end{equation}    
    We apply the above bound of $\sum_{i=1}^{k^*}$, and use Lemma~\ref{lemma:tf-prod-A-bias-bound} to bound $\sum_{i=k^*+1}^{d}$:
    \begin{equation}
    \begin{aligned}        
        &\left\langle\begin{bmatrix}
            \bs O & \bs O \\
            \bs O & \bs S
        \end{bmatrix},\tilde{\bs B}_t\right\rangle \\
        =&\sum_{i=1}^{k^*}{\lambda_i\left(\bs w_i^*\right)^2\left(\prod_{s=1}^t\bs A_{s,i}\begin{bmatrix}
            1 \\ 1
        \end{bmatrix}\right)_1^2}+\sum_{i=k^*+1}^{d}{\lambda_i\left(\bs w_i^*\right)^2\left(\prod_{s=1}^t\bs A_{s,i}\begin{bmatrix}
            1 \\ 1
        \end{bmatrix}\right)_1^2} \\
        % \leq&\sum_{i=1}^{k^*}{\lambda_i\left(\bs w_i^*\right)^2\underbrace{\left\Vert\bs A_{(\ell)}^{t-K(\ell-1)}\right\Vert^2}_{\text{Lemma~\ref{lemma:2-norm-uniform-bound}}}\underbrace{\left\Vert\bs A_{(\ell-1)}^K\right\Vert^2\cdots\left\Vert\bs A_{(1)}^K\right\Vert^2}_{\text{Lemma~\ref{lemma:2-norm-bound}}}\left\Vert\begin{bmatrix}
        %     1 \\ 1
        % \end{bmatrix}\right\Vert^2}+4\sum_{i=k^*+1}^{d}{\lambda_i\left(\bs w_i^*\right)^2}. \\
        % \leq&\sum_{i=1}^{k^*}{\lambda_i\left(\bs w_i^*\right)^2\cdot 6K^2\cdot\left(\frac{\sqrt 6}{n^2\log_2 n}\right)^{2(\ell-1)}}+4\sum_{i=k^*+1}^{d}{\lambda_i\left(\bs w_i^*\right)^2}. \\
        \leq&\frac{36}{n^2\left(\log_2 n\right)^4}\sum_{i=1}^{k^*}{\lambda_i\left(\bs w_i^*\right)^2}+4\sum_{i=k^*+1}^{d}{\lambda_i\left(\bs w_i^*\right)^2}.
    \end{aligned}
    \end{equation}
    This completes the proof.
\end{proof}
\begin{proof}[Proof of Lemma~\ref{lemma:bias-1-ub}]
    From the recursive definition of $\tilde{\bs B}_{t}^{(1)}$ in \eqref{eq:def-bias-1} and Lemma~\ref{lemma:tilde-bias-uniform-ub}, we have:
    \begin{equation}
    \begin{aligned}
        \bs B_{t}^{(1)}=&\cB_{t}\circ\bs B_{t-1}^{(1)}+\bbE\left[\widehat{\bs G}_t\otimes\widehat{\bs G}_t\right]\circ\tilde{\bs B}_{t-1} \\
        \preceq&\cB_{t}\circ\bs B_{t-1}^{(1)}+4\left\Vert\bs w^*\right\Vert_{\bs S}^2\cdot\begin{bmatrix}
        \delta_t^2\bs S & \delta_t q_t\bs S \\
        \delta_t q_t\bs S & q_t^2\bs S
        \end{bmatrix}.
    \end{aligned}
    \end{equation}
    This form is identical to the recursion of $\tilde{\bs C}_t$ if we replace $4\left\Vert\bs w^*\right\Vert_{\bs S}^2$ by $\sigma^2$. Therefore, we apply Lemma~\ref{lemma:tilde-var-ub} to obtain
    \begin{equation}
        \left\langle\tilde{\bs T},\bs B_n^{(1)}\right\rangle\leq\left\Vert\bs w_0-\bs w^*\right\Vert_{\bs S}^2\cdot\left[\sum_{i=1}^{k^*}{\frac{2t_{ii}}{K\lambda_i}}+\frac{512}{15}K\left(\frac{q-c\delta}{1-c}\right)^2\sum_{i=k^*+1}^{d}\lambda_i t_{ii}\right].
    \end{equation}

\end{proof}

\subsection{Auxiliary Lemmas}
\begin{lemma}\label{lemma:aux-lemma-r}
    For $k\geq 0$ and $0\leq x\leq 1$, we have
    \begin{equation}
        x^k\left[1+k(1-x)\right]\leq 1.
    \end{equation}
\end{lemma}
\begin{proof}
    Let $f(x)=x^k\left[1+k(1-x)\right]$ and its derivative $f'(x)=k(k+1)x^{k-1}(1-x)\geq 0$. Thus, $f(x)\leq f(1)=1$.
\end{proof}

\begin{lemma}\label{lemma:aux-lemma-sin}
    For $k\in\bbN$ and $\sin \theta\neq 0$, we have
    \begin{equation}
        \left\vert\frac{\sin k\theta}{\sin\theta}\right\vert\leq k.
    \end{equation}
\end{lemma}
\begin{proof}
    By induction, for $k=0$, the conclusion is trivial. Assume \begin{equation}
        \left\vert\frac{\sin (k-1)\theta}{\sin\theta}\right\vert\leq k-1.
    \end{equation}
    Then we have
    \begin{equation}
    \begin{aligned}        
        \left\vert\frac{\sin k\theta}{\sin\theta}\right\vert=&\left\vert\frac{\sin(k-1)\theta\cos\theta+\cos(k-1)\theta\sin\theta}{\sin\theta}\right\vert \\
        \leq&\left\vert\cos\theta\right\vert\left\vert\frac{\sin (k-1)\theta}{\sin\theta}\right\vert+\left\vert\cos(k-1)\theta\right\vert\leq k.
    \end{aligned}
    \end{equation}
\end{proof}

\section{Proof of Optimality of ASGD in Section~\ref{sec: sub asgd optiamlity}}
\subsection{Convergence Rate of ASGD}
We first present a general Assumption~\ref{assumption:poly-decay-M} and establish the upper bound of ASGD in Lemma~\ref{lemma:poly-decay-asgd-upper-bound} and Lemma~\ref{lemma:poly-decay-sgd-upper-bound}. The two lemmas will be used in prove Theorems~\ref{thm:asgd-poly-class-opt} and~\ref{thm:asgd-poly-decay-opt-complete}.
\begin{assumption}\label{assumption:poly-decay-M}
    Let $a,b,\kappa,\nu\in\bbR$ be constants, where $b=sa$ and $\kappa=-ra$. Given target distribution $Q$, let $\bs T=\bbE_{\bs x\sim Q}\bs x\bs x^\top$.
    \begin{enumerate}
        \item Let $t_{ii}$ denote the $i$-th diagonal element of $\bs T$. We assume $\lambda_i\eqsim i^{-a}$, $\lambda_i/m_i\eqsim i^{-b}$ and $t_{ii}\lesssim i^{-a+\kappa}$. Therefore, $t_{ii}/m_i\lesssim i^{-b+
        \kappa}$
        \item We assume that $\left\Vert\bs M^{-1/2}\bs T_{i:\infty}\bs M^{-1/2}\right\Vert\lesssim i^{-b+\kappa+\nu}$, where $0\leq\nu\leq1$.
    \end{enumerate}
\end{assumption}
Note that distributions in $r$-smooth class satisfies the above assumption with $\nu=0$.

The following lemmas provide the convergence rate of ASGD.
\begin{lemma}\label{lemma:poly-decay-asgd-upper-bound}
    Suppose Assumptions~\ref{assumption:fourth-moment},~\ref{assumption:noise} and~\ref{assumption:poly-decay-M} hold, and $\kappa<b-\nu$. For the region $\frac{(a-1)^2}{2a-1}<b-\nu<a-1$, we choose the parameters as follows:
    \begin{equation}
        \tilde{\kappa}=\Theta\left(n^{\frac{a-b+\nu-1}{(a-1)(b+1)}}\right), \quad\delta=\Theta(1/\ln n), \quad\gamma=\Theta\left(n^{\frac{a-b+\nu-1}{b-\nu+1}}/\ln n\right).
    \end{equation}
    Then, the upper bound of ASGD is
    \begin{equation*}
        \sup_{\tilde{P}\in\mathcal{P}(W,\bs S,\bs T)}\bbE_{\tilde{P}^{\otimes n}}\left\Vert\bs w_n^\sgd-\bs w^*\right\Vert_{\bs T}^2\lesssim\begin{cases}
            (1+\sigma^2)\ln n/n,&\kappa<-1; \\
            (1+\sigma^2)(\ln n)^3/n,&\kappa=-1; \\
            (1+\sigma^2)(1/n)^\frac{b-\nu-\kappa}{b-\nu+1}(\ln n)^\frac{3(a-\kappa-1)}{a},&-1<\kappa<b-\nu. \\
            % (1+\sigma^2)(1/n)^\frac{b-\nu-\kappa}{b-\nu+1}(\ln n)^{\frac{3(a-\kappa-1)}{a}},&-1<\kappa<b-\nu. \\
        \end{cases}
    \end{equation*}
\end{lemma}
\begin{proof}
    Following the parameter choice procedure, we have 
    \begin{equation}
        \beta=\Theta\left(\frac{\delta}{\psi\tilde{\kappa}\gamma\ln n}\right)=\Theta\left(n^{-\frac{a(a-b+\nu-1)}{(a-1)(b-\nu+1)}}/\ln n\right), \quad\alpha=\frac{1}{1+\beta}.
    \end{equation}
    Note that from $\frac{(a-1)^2}{2a-1}<b$, we have
    \begin{equation}
        -\frac{a(a-b+\nu-1)}{(a-1)(b-\nu+1)}=-\frac{a}{a-1}\left(\frac{a}{b-\nu+1}-1\right)>-\frac{a}{a-1}\left(\frac{a}{\frac{(a-1)^2}{2a-1}+1}-1\right)=-1.
    \end{equation}
    Therefore,
    \begin{equation}
        \frac{n\left[1-\alpha(1-\beta)\right]}{\log_2 n\ln n}=\Theta\left(\frac{n^{1-\frac{a(a-b+\nu-1)}{(a-1)(b-\nu+1)}}}{\ln^3 n}\right)\stackrel{a}{\geq} 16,
    \end{equation}
    where $\stackrel{a}{\geq}$ follows from $1-\frac{a(a-b+\nu-1)}{(a-1)(b-\nu+1)}>0$. Recall that $k^*=\max\left\{k:\lambda_k>\frac{16\ln n}{(\gamma+\delta)K}\right\}$, we have
    \begin{equation}
        k^*=\Theta\left(n^\frac{1}{b-\nu+1}\ln^{-\frac{3}{a}}n\right).
    \end{equation}
    From Theorem~\ref{thm:asgd-upper-bound}, for the variance term, we have
    \begin{equation}
    \begin{aligned}        
        \mathrm{Variance}\lesssim&\sigma^2\left[\sum_{i=1}^{k^*}{\frac{t_{ii}}{K\lambda_i}}+K\left(\gamma+\delta\right)^2\sum_{i=k^*+1}^{d}\lambda_i t_{ii}\right] \\
        \stackrel{a}{\lesssim}&\sigma^2\left[\sum_{i=1}^{k^*}{\frac{i^\kappa}{K}}+K\left(\gamma+\delta\right)^2\sum_{i=k^*+1}^{d}i^{-2a+\kappa}\right]. \\
    \end{aligned}        
    \end{equation}
    where $\stackrel{a}{\lesssim}$ uses Lemma~\ref{lemma:decay-prop}. By direct calculation, we have
    \begin{equation}
        \sum_{i=1}^{k^*}{\frac{i^\kappa}{K}}=\begin{cases}
            \ln n/n,&\kappa<-1; \\
            (\ln n)^2/n,&\kappa=-1; \\
            (1/n)^\frac{b-\nu-\kappa}{b-\nu+1}(\ln n)^{-\frac{3(\kappa+1)}{a}},&\kappa>-1, \\
        \end{cases}
    \end{equation}
    \begin{equation}
        K\left(\gamma+\delta\right)^2\sum_{i=k^*+1}^{d}i^{-2a+\kappa}\stackrel{a}{\eqsim}(1/n)^\frac{b-\nu-\kappa}{b-\nu+1}(\ln n)^{\frac{3(a-\kappa-1)}{a}},
    \end{equation}
    where $\stackrel{a}{\eqsim}$ uses $-2a+\kappa\leq-a-(b-\nu+1)+\kappa\leq-b+\nu+\kappa-1<-1$. For the bias term, we have
    \begin{equation}
        \mathrm{Bias}\lesssim\left(k^*\right)^{-b+\kappa}+\frac{\left\Vert\bs S'\right\Vert\cdot\mathrm{Variance}}{\sigma^2}=\Theta\left(\left(1/n\right)^\frac{b-\kappa}{b+1}(\ln n)^{\frac{3(b-\kappa)}{a}}\right)+\Theta\left(\frac{\mathrm{Variance}}{\sigma^2}\right).
    \end{equation}
    Combining the above three results, we get
    \begin{equation}
        \sup_{\tilde{P}\in\mathcal{P}(W,\bs S,\bs T)}\bbE_{\tilde{P}^{\otimes n}}\left\Vert\bs w_n^\sgd-\bs w^*\right\Vert_{\bs T}^2\lesssim\begin{cases}
            (1+\sigma^2)\ln n/n,&\kappa<-1; \\
            (1+\sigma^2)(\ln n)^3/n,&\kappa=-1; \\
            (1+\sigma^2)(1/n)^\frac{b-\nu-\kappa}{b-\nu+1}(\ln n)^{\frac{3(a-\kappa-1)}{a}},&-1<\kappa<b-\nu. \\
        \end{cases}
    \end{equation}
    % We complete the proof by noting that $\frac{3(a-\kappa-1)}{a}<3$ if $\kappa>-1$.
    This completes the proof.
\end{proof}

\begin{lemma}\label{lemma:poly-decay-sgd-upper-bound}
    Suppose $\kappa<\min\{2a-1,b-\nu\}$. For the region $a-1\leq b-\nu$, we set
    \begin{equation}
        \tilde{\kappa}=\Theta(1), \quad\delta=\gamma=\Theta\left(n^{\frac{a-b+\nu-1}{b-\nu+1}}/\ln n\right).
    \end{equation}
    Then ASGD reduces to SGD with step-size $\gamma$, and the upper bound is
    \begin{equation}
    \begin{aligned}        
        &\sup_{\tilde{P}\in\mathcal{P}(W,\bs S,\bs T)}\bbE_{\tilde{P}^{\otimes n}}\left\Vert\bs w_n^\sgd-\bs w^*\right\Vert_{\bs T}^2 \\
        \lesssim&\begin{cases}
            (1+\sigma^2)\ln n/n,&\kappa<-1; \\
            (1+\sigma^2)(\ln n)^3/n,&\kappa=-1; \\
            (1+\sigma^2)(1/n)^\frac{b-\nu-\kappa}{b-\nu+1}(\ln n)^\frac{3(a-\kappa-1)}{a},&-1<\kappa<\min\{2a-1,b-\nu\}. \\
            % (1+\sigma^2)(1/n)^\frac{b-\nu-\kappa}{b-\nu+1}(\ln n)^{\frac{3(a-\kappa-1)}{a}},&-1<\kappa<\min\{2a-1,b-\nu\}. \\
        \end{cases}
    \end{aligned}
    \end{equation}
\end{lemma}
\begin{proof}
    Following the parameter choice, we have $\beta=\Theta(1/\ln n)$. Therefore,
    \begin{equation}        
        \frac{n\left[1-\alpha(1-\beta)\right]}{\log_2 n\ln n}=\Theta\left(\frac{n}{\ln^3 n}\right)\geq 16.
    \end{equation}
    Note that $\kappa<\min\{2a-1,b-\nu\}$ implies $-2a+\kappa<-1$. The rest of the proof is identical to the proof of Lemma~\ref{lemma:poly-decay-asgd-upper-bound}.
\end{proof}

\subsection{Proof of Theorem~\ref{thm:asgd-poly-class-opt}}
Suppose $\cQ$ is a $r$-smooth class. Let $Q_{\bs x_0}$ be defined in Definition~\ref{def:r-smooth-class}. Let $\bs T_0=\bbE_{Q_{\bs x_0}}\left[\bs x\bs x^\top\right]$, and $\bs T_0'=\bs M^{-1/2}\bs T_0\bs M^{-1/2}$, and recall that $\bs M$ is diagonal, so we have
\begin{equation}
    \left(\bs S'\right)^{-1/2}\bs T_0'\left(\bs S'\right)^{-1/2}=\bs S^{-1/2}\bs T_0\bs S^{-1/2}\succeq c\bs S^r.
\end{equation}
Thus, $\bs T_0\succeq c\bs S^{1+r}
\defeq\bs R$, and $\bs T_0'\succeq c\left(\bs S'\right)^{1/2}\bs S^r\left(\bs S'\right)^{1/2}=\bs M^{-1/2}\bs R\bs M^{-1/2}\triangleq\bs R'$. From Theorem~\ref{thm:lower-upper-match}, for $Q_{\bs x_0}$, we have the following lower bound:
\begin{equation}
\begin{aligned}    
    &\inf_{\hat{\bs w}}\sup_{\tilde{P}\in\mathcal{P}(W,\bs S,\bs T)}\bbE_{\tilde{P}^{\otimes n}\times P_{\xi}}\left\Vert\hat{\bs w}-\bs w^*\right\Vert_{\bs T_0}^2 \\
    \geq&\min_{\bs A\in\bbR^{d\times d}}\frac{1}{\pi^2}\left\Vert(\bs I-\bs A)^\top\bs R'(\bs I-\bs A)\right\Vert+\frac{\sigma^2}{n}\left\langle\bs R',\bs A\left(\bs S'\right)^{-1}\bs A^\top\right\rangle.
\end{aligned}
\end{equation}
Note that $\bs R$ and $\bs R'$ are diagonal, the following lemma simplifies the lower bound of $Q_{\bs x_0}$.
\begin{lemma}\label{lemma:lower-bound-diagonal}
    Let $\bs R=\diag\left\{r_i\right\}_{i=1}^d$ and $\bs M=\diag\left\{m_i\right\}_{i=1}^d$, we have
    \begin{equation}
    \begin{aligned}    
        &\min_{\bs A\in\bbR^{d\times d}}\frac{1}{\
        \pi^2}\left\Vert(\bs I-\bs A)^\top\bs R'(\bs I-\bs A)\right\Vert+\frac{\sigma^2}{n}\left\langle\bs R',\bs A\left(\bs S'\right)^{-1}\bs A^\top\right\rangle \\
        =&\min_{\tau\geq 0}\frac{\tau^2}{\pi^2}+\sum_{i\in\bbK_\tau}\left(1-\tau\sqrt{\frac{m_i}{r_i}}\right)^2\frac{\sigma^2r_i}{n\lambda_i}\geq\frac{1}{4}\min_{\tau\geq 0}\frac{\tau^2}{\pi^2}+\sum_{i\in\bbK_\tau}\frac{\sigma^2r_i}{n\lambda_i},\label{eq:lower-bound-diagonal}
    \end{aligned}
    \end{equation}
    where $\bbK_\tau=\left\{k:r_k/m_k>\tau^2\right\}$.
\end{lemma}
\begin{proof}
    A key observation is that when $\bs R'$ is diagonal, the minimum of the LHS of \eqref{eq:lower-bound-diagonal} is attained when $\bs A$ is diagonal. Note that the LHS of \eqref{eq:lower-bound-diagonal} is a convex optimization. Let $\bs A_0$ denote a minimizer. Consider $2^d$ reflection matrices $\bs U_i=\diag\left\{\pm 1,\pm 1,\ldots,\pm 1\right\}$, then for all $i\in \left[2^d\right]$, $\bs U_i\bs A_0\bs U_i$ is also a minimizer. From the convexity, we have that
    \begin{equation}
        \bs A^*=\frac{1}{2^d}\sum_{i=1}^{2^d}\bs U_i\bs A_0\bs U_i
    \end{equation}
    is also a minimizer, and $\bs A^*$ is diagonal. Thus, we can restrict $\bs A$ to be diagonal when minimizing the LHS of \eqref{eq:lower-bound-diagonal}. Therefore, let $\bs A=\diag\left\{a_i\right\}_{i=1}^d$ and note that $\bs R'=\diag\left\{r_i/m_i\right\}_{i=1}^d$, then the LHS of \eqref{eq:lower-bound-diagonal} is equivalent to
    \begin{equation}
        \min_{a_i}\max_{k\in[d]}\frac{(1-a_k)^2r_k}{\pi^2 m_k}+\sum_{i=1}^{d}\frac{\sigma^2a_i^2r_i}{n\lambda_i}.
    \end{equation}
    We can write out the following equivalent form:
    \begin{equation}
    \begin{aligned}        
        &\min_{a_i,\tau\geq 0}\frac{\tau^2}{\pi^2}+\sum_{i=1}^{d}\frac{\sigma^2a_i^2r_i}{n\lambda_i}, \\
        \text{s.t. } &\forall i\in[d],\frac{(1-a_i)^2r_i}{m_i}\leq\tau^2.
    \end{aligned}
    \end{equation}
    We first minimize the above program with respect to $a_i$ to get
    \begin{equation}
        a_i=\begin{cases}
            0,&r_i/m_i<\tau^2;\\
            1-\tau\sqrt{m_i/r_i},&r_i/m_i\geq\tau^2.
        \end{cases}
    \end{equation}
    Plugging the value of $a_i$ into left hand side of \eqref{eq:lower-bound-diagonal}, we obtain the first equality in \eqref{eq:lower-bound-diagonal}:
    \begin{equation}
        \min_{\tau\geq 0}\frac{\tau^2}{\pi^2}+\sum_{i\in\bbK_\tau}\left(1-\tau\sqrt{\frac{m_i}{r_i}}\right)^2\frac{\sigma^2r_i}{n\lambda_i}.\label{eq:lower-bound-diagonal-form}
    \end{equation}

    Let $\tau^*\geq 0$ denote the minimizer of \eqref{eq:lower-bound-diagonal-form}, we have
    \begin{equation}
    \begin{aligned}        
        &\left(\frac{\tau^*}{\pi}\right)^2+\sum_{i\in\bbK_{\tau^*}}\left(1-\tau^*\sqrt{\frac{m_i}{r_i}}\right)^2\frac{\sigma^2r_i}{n\lambda_i}\stackrel{a}{\geq}\left(\frac{\tau^*}{\pi}\right)^2+\sum_{i\in\bbK_{2\tau^*}}\left(1-\tau^*\sqrt{\frac{m_i}{r_i}}\right)^2\frac{\sigma^2r_i}{n\lambda_i} \\
        \stackrel{b}{\geq}&\frac{1}{4}\left(\frac{2\tau^*}{\pi}\right)^2+\sum_{i\in\bbK_{2\tau^*}}\frac{\sigma^2r_i}{4n\lambda_i}\stackrel{c}{\geq}\frac{1}{4}\min_{\tau\geq 0}\frac{\tau^2}{\pi^2}+\sum_{i\in\bbK_{\tau}}\frac{\sigma^2r_i}{n\lambda_i},
    \end{aligned}
    \end{equation}
    where $\stackrel{a}{\geq}$ is from $\bbK_{2\tau^*}\subset\bbK_{\tau^*}$, $\stackrel{b}{\geq}$ uses that $1-\tau^*\sqrt{\frac{m_i}{r_i}}\geq\frac{1}{2}$ for all $i\in\bbK_{2\tau^*}$, and $\stackrel{c}{\geq}$ replaces $2\tau^*$ by $\tau$ and minimizes with respect to $\tau$. This completes the proof of the inequality in \eqref{eq:lower-bound-diagonal}. 
\end{proof}

Given the polynomial decaying conditions, the following lemma provides some useful properties.
\begin{lemma}\label{lemma:decay-prop}
    Suppose $\lambda_i\eqsim i^{-a}$ and $\lambda_i/m_i\eqsim i^{-b}$, where $b=sa$. Let $\cQ$ be a $r$-smooth class and $\kappa=-ra$. We denote the hard instance in Definition~\ref{def:r-smooth-class} by $Q_0$, and let $\bs T_0=\bbE_{\bs x\sim Q}\left[\bs x\bs x^\top\right]$. Then we have
    \begin{enumerate}
        \item Recall that $\bs R\defeq c\bs S^{1/2}\bs D\bs S^{1/2}$ is diagonal. Then $r_i\eqsim i^{-a+\kappa}$ and $r_i/m_i\eqsim i^{-b+\kappa}$.
        \item For any $Q_{\bs x}\in\cQ$ and let $\bs T=\bbE_{Q_{\bs x}}\left[\bs x\bs x^\top\right]$. Then we have $t_{ii}\lesssim i^{-a+\kappa}$ and $t_{ii}/\lambda_i\lesssim i^{\kappa}$ for any integer $i\in[d]$, where $t_{ii}$ denotes the $i$-th diagonal element of $\bs T$.
        \item For any $Q_{\bs x}\in\cQ$ and let $\bs T'=\bs M^{-1/2}\bbE_{Q_{\bs x}}\left[\bs x\bs x^\top\right]\bs M^{-1/2}$. Then we have $\left\Vert\bs T'_{i:\infty,i:\infty}\right\Vert\lesssim i^{-b+\kappa}$ for any integer $i\in[d]$.
    \end{enumerate}
\end{lemma}
\begin{proof}
    \begin{enumerate}
        \item Note that $r_i=\lambda_i\cdot i^\kappa\eqsim i^{-a+\kappa}$ and $r_i/m_i\eqsim(\lambda_i/m_i)\cdot i^\kappa\eqsim i^{-b+\kappa}$.    
        \item From Definition~\ref{def:r-smooth-class}, we have $\bs T\preceq C\bs S^{1/2}\bs D\bs S^{1/2}\preceq C\diag\left\{\lambda_i\cdot i^\kappa\right\}_{i=1}^d$, Therefore,
        \begin{equation}
            t_{ii}\leq C\lambda_i\cdot i^\kappa\lesssim i^{-a+\kappa}.
        \end{equation}    
        \item Note that $\bs T'\preceq C\bs M^{-1/2}\bs S^{1/2}\bs D\bs S^{1/2}\bs M^{-1/2}=C\diag\left\{(\lambda_i/m_i)\cdot i^\kappa\right\}_{i=1}^d$. From Item 1, we have
        \begin{equation}
            \left\Vert\bs T'_{i:\infty,i:\infty}\right\Vert\leq C\left\Vert\diag\left\{(\lambda_{i+1}/m_{i+1})\cdot (i+1)^\kappa,\ldots,(\lambda_d/m_d)\cdot d^\kappa\right\}\right\Vert\leq i^{-b+\kappa}.
        \end{equation}
    \end{enumerate}
\end{proof}

The following lemma explicitly calculates the lower bound in terms of $\sigma^2$ and $n$ given polynomial decaying conditions.

\begin{lemma}\label{lemma:poly-class-lower-bound}
    Suppose $\lambda_i\eqsim i^{-a}$ and $\lambda_i/m_i\eqsim i^{-b}$, and $\kappa<b$. If we treat $a$, $b$, $\kappa$, and $c$ in the Definition~\ref{def:r-smooth-class} as constants, then the lower bound
    \begin{equation}
        \inf_{\hat{\bs w}}\sup_{\tilde{P}\in\mathcal{P}(W,\bs S,\bs T)}\bbE_{\tilde{P}^{\otimes n}\times P_{\xi}}\left\Vert\hat{\bs w}-\bs w^*\right\Vert_{\bs T_0}^2\gtrsim\begin{cases}
            \sigma^2/n,&\kappa<-1; \\
            (\sigma^2/n)\ln(n/\sigma^2),&\kappa=-1; \\
            \left(\sigma^2/n\right)^\frac{b-\kappa}{b+1},&-1<\kappa<b.
        \end{cases}
    \end{equation}
\end{lemma}
\begin{proof}
    From Lemma~\ref{lemma:decay-prop}, we have $r_i=\lambda_i\cdot i^\kappa\eqsim i^{-a+\kappa}$ and $r_i/m_i\eqsim(\lambda_i/m_i)\cdot i^\kappa\eqsim i^{-b+\kappa}$. Therefore, $\bbK_\tau=\left[k^*\right]$, where $k^*\eqsim \tau^{-2/(b-\kappa)}$. Therefore, we have
    \begin{equation}
        \frac{\tau^2}{\pi^2}+\sum_{i\in\bbK_\tau}\frac{\sigma^2r_i}{n\lambda_i}\eqsim\tau^2+\frac{\sigma^2}{n}\sum_{i=1}^{k^*}i^{\kappa}\eqsim\begin{cases}
            \tau^2+\sigma^2/n,&\kappa<-1; \\
            \tau^2-(\sigma^2/n)\ln\tau,&\kappa=-1; \\
            \tau^2+(\sigma^2/n)\tau^
            {-\frac{2(\kappa+1)}{b-\kappa}},&\kappa>-1.
        \end{cases}\label{eq:lower-bound-value}
    \end{equation}
    Minimize the RHS with respect to $\tau$. The optimal value of $\tau$ is
    \begin{equation}
        \tau=\begin{cases}
            0,&\kappa<-1; \\
            \sqrt{\sigma^2/(2n)},&\kappa=-1; \\
            \left(\frac{\kappa+1}{b-\kappa}\frac{\sigma^2}{n}\right)^\frac{b-\kappa}{2(b+1)},&-1<\kappa<b.
        \end{cases}
    \end{equation}
    Substitute the value of $\tau$ into \eqref{eq:lower-bound-value}, we have
    \begin{equation}
        \min_{\tau>0}\frac{\tau^2}{\pi^2}+\sum_{i\in\bbK_\tau}\frac{\sigma^2r_i}{n\lambda_i}\eqsim\begin{cases}
            \sigma^2/n,&\kappa<-1; \\
            (\sigma^2/n)\ln(n/\sigma^2),&\kappa=-1; \\
            \left(\sigma^2/n\right)^\frac{b-\kappa}{b+1},&-1<\kappa<b.
        \end{cases}
    \end{equation}
This completes the proof.
\end{proof}

We combine the above results to prove Theorem~\ref{thm:asgd-poly-class-opt}.
\begin{proof}[Proof of Theorem~\ref{thm:asgd-poly-class-opt}]
    We set $b=sa$ and $\kappa=-ra$. The lower bound is from Lemma~\ref{lemma:poly-class-lower-bound}. For the upper bound, note that any $Q\in\cQ$ satisfies Assumption~\ref{assumption:poly-decay-M} with $\nu=0$. For the region $s\geq1-\frac{1}{a}$, Lemma~\ref{lemma:poly-decay-sgd-upper-bound} shows that vanilla SGD achieves optimality up to logarithmic factors. For the region $s\geq1-\frac{1}{a}$, Lemma~\ref{lemma:poly-decay-asgd-upper-bound} for the region $1-\frac{1}{a}>s>\frac{(a-1)^2}{2a-1}$ shows that SGD with momentum achieves optimality up to logarithmic factors. If $\sigma^2$ is regarded as constant, then the upper bound matches the lower bound up to logarithmic factors.
\end{proof}

\subsection{Analysis of Optimality Beyond the Diagonal Dominated Condition}\label{sec:general-version-poly-decay}
We provide a more general analysis under Assumption~\ref{assumption:poly-decay} in the following Theorem. This theorem is not limited to cases where the target covariance matrix is diagonal dominant relative to the source covariance matrix; instead, it provides a general condition for directly calculating the ASGD upper bound and the minimax lower bound concerning the convergence rate with respect to the sample size $n$.
\begin{theorem}\label{thm:asgd-poly-decay-opt-complete}
    Suppose Assumptions~\ref{assumption:fourth-moment},~\ref{assumption:noise} and~\ref{assumption:poly-decay} hold.
    \begin{itemize}
        \item For the ASGD upper bound, we assume that $s>\frac{(a-1)^2}{a(2a-1)}+\frac{\nu}{a}$, $a>1$ and $r>\max\{1/a-2,-s+\nu/a\}$. With appropriately chosen parameters, ASGD has the following upper bound:
        \begin{equation}
        \begin{aligned}            
            &\sup_{\bs w^*\in W}\bbE_{P_{\bs w^*}}\left\Vert\bs w_n^\sgd-\bs w^*\right\Vert_{\bs T}^2\\
            \lesssim&\begin{cases}
            (1+\sigma^2)\ln n/n,&r>1/a; \\
            (1+\sigma^2)(\ln n)^3/n,&r=1/a; \\(1+\sigma^2)(1/n)^\frac{(r+s)a-\nu}{sa-\nu+1}(\ln n)^\frac{3(1+r)a-3}{a},&1/a>r>\max\{1/a-2,-s+\nu/a\}. \\
            % (1+\sigma^2)(1/n)^\frac{sa-\nu-\kappa}{sa-\nu+1}(\ln n)^{\frac{3(a-\kappa-1)}{a}},&-1<\kappa<\min\{2a-1,sa-\nu\}. \\
        \end{cases}
        \end{aligned}
        \end{equation}
        \item For the lower bound, if $r\geq 1/a$, we have        
        \begin{equation}
            \inf_{\hat{\bs w}}\sup_{\bs w^*\in W}\bbE_{P_{\bs w^*}}\left\Vert\hat{\bs w}-\bs w^*\right\Vert_{\bs T}^2\gtrsim\begin{cases}
            \sigma^2/n,&r>1/a; \\
            (\sigma^2/n)\ln (n/\sigma^2),&r=1/a,
            \end{cases}
        \end{equation}
        which matches the upper bound up to logarithmic factors. If $r<1/a$, let $k^*\eqsim(\sigma^2/n)^{-\frac{1}{sa-\nu+1}}$ and $\Delta k=\left\lfloor\frac{n}{\pi^2\sigma^2}\left\Vert\bs S'_{k^*:\infty}\right\Vert\right\rfloor$. We further assume that
        \begin{equation}
            \frac{1}{\Delta k}\sum_{i=1}^{\Delta k}\sigma_i\left(\bs M^{-1/2}\bs T_{k^*:\infty}\bs M^{-1/2}\right)\gtrsim\left\Vert\bs M^{-1/2}\bs T_{k^*:\infty}\bs M^{-1/2}\right\Vert.\label{eq:eigenvalue-mean-lower-bound}
            % \frac{1}{\Delta k}\sum_{i=1}^{\Delta k}\lambda_i\left(\bs M^{-1/2}\bs T_{k^*:\infty}\bs M^{-1/2}\right)\gtrsim\left\Vert\bs M^{-1/2}\bs T_{k^*:\infty}\bs M^{-1/2}\right\Vert\eqsim(\sigma^2/n)^{-\frac{(r+s)a-\nu}{sa-\nu+1}},
        \end{equation}
        where $\sigma_i(\bs A)$ denotes the $i$-singular value of $\bs A$. Then, the minimax lower bound is
        \begin{equation}
            \inf_{\hat{\bs w}}\sup_{\bs w^*\in W}\bbE_{P_{\bs w^*}}\left\Vert\hat{\bs w}-\bs w^*\right\Vert_{\bs T}^2\gtrsim\left(\sigma^2/n\right)^\frac{(r+s)a-\nu}{sa-\nu+1}.
        \end{equation}
        This result matches the upper bound up to logarithmic factors, indicating no other algorithm outperforms ASGD.
    \end{itemize}
\end{theorem}
\begin{remark}
    For the proof Theorem~\ref{thm:rank-one}, we have Assumption~\ref{assumption:poly-decay} holds for $\nu=1$. Therefore, we have $k^*\eqsim(\sigma^2/n)^{-\frac{1}{sa-\nu+1}}=(\sigma^2/n)^{-\frac{1}{sa}}$, and
    \begin{equation}
        \Delta k=\left\lfloor\frac{n}{\pi^2\sigma^2}\left\Vert\bs S'_{k^*:\infty}\right\Vert\right\rfloor\eqsim\frac{n}{\pi^2\sigma^2}\left(k^*\right)^{-sa}\eqsim\frac{1}{\sigma^2}.
    \end{equation}
    We can choose appropriate $k^*$ such that $\Delta k\geq 1$. Therefore, $\Delta k=\cO(1)$, and by Theorem~\ref{thm:asgd-poly-decay-opt-complete}, we have ASGD achieves optimality.
\end{remark}
\begin{proof}[Proof of Theorem~\ref{thm:asgd-poly-decay-opt-complete}]
    We set $b=sa$ and $\kappa=-ra$. Note that Assumption~\ref{assumption:poly-decay} implies Assumption~\ref{assumption:poly-decay-M}. Therefore, the upper bound part follows from Lemma~\ref{lemma:poly-decay-asgd-upper-bound} and Lemma~\ref{lemma:poly-decay-sgd-upper-bound}.

    For the lower bound, Lemma~\ref{lemma:poly-decay-lower-bound-easy} provides the case where $\kappa\leq -1$. For $\kappa>-1$, Lemma~\ref{lemma:poly-decay-lower-bound-hard} shows that
    \begin{equation}        
        \inf_{\hat{\bs w}}\sup_{\tilde{P}\in\mathcal{P}(W,\bs S,\bs T)}\bbE_{\tilde{P}^{\otimes n}\times P_{\xi}}\left\Vert\hat{\bs w}-\bs w^*\right\Vert_{\bs T}^2\gtrsim\frac{1}{\Delta k}\sum_{i=1}^{\Delta k}\sigma_i\left(\bs M^{-1/2}\bs T_{k^*:\infty}\bs M^{-1/2}\right).
    \end{equation}
    Note that from \eqref{eq:eigenvalue-mean-lower-bound} and $k^*\eqsim(\sigma^2/n)^{-\frac{1}{sa-\nu+1}}$, we obtain
    \begin{equation}        
            \frac{1}{\Delta k}\sum_{i=1}^{\Delta k}\lambda_i\left(\bs M^{-1/2}\bs T_{k^*:\infty}\bs M^{-1/2}\right)\gtrsim\left\Vert\bs M^{-1/2}\bs T_{k^*:\infty}\bs M^{-1/2}\right\Vert\eqsim(\sigma^2/n)^{-\frac{sa-\nu-\kappa}{sa-\nu+1}},
    \end{equation}
    then we have
    \begin{equation}
        \inf_{\hat{\bs w}}\sup_{\tilde{P}\in\mathcal{P}(W,\bs S,\bs T)}\bbE_{\tilde{P}^{\otimes n}\times P_{\xi}}\left\Vert\hat{\bs w}-\bs w^*\right\Vert_{\bs T}^2\gtrsim(\sigma^2/n)^{-\frac{sa-\nu-\kappa}{sa-\nu+1}},
    \end{equation}
    which matches the ASGD upper bound in Lemma~\ref{lemma:poly-decay-asgd-upper-bound} and Lemma~\ref{lemma:poly-decay-sgd-upper-bound}.
\end{proof}

The following lemma establishes the lower bound for the case $\kappa\leq -1$.
\begin{lemma}\label{lemma:poly-decay-lower-bound-easy}
    Suppose Assumption~\ref{assumption:poly-decay} holds. If $\kappa\leq -1$, we have
    \begin{equation}
        \inf_{\hat{\bs w}}\sup_{\tilde{P}\in\mathcal{P}(W,\bs S,\bs T)}\bbE_{\tilde{P}^{\otimes n}\times P_{\xi}}\left\Vert\hat{\bs w}-\bs w^*\right\Vert_{\bs T_0}^2\gtrsim\begin{cases}
            \sigma^2/n,&\kappa<-1; \\
            (\sigma^2/n)\ln(n/\sigma^2),&\kappa=-1.
        \end{cases}
    \end{equation}
\end{lemma}
\begin{proof}
    From Theorem~\ref{thm:lower-bound}, we have
    \begin{equation}
        \inf_{\hat{\bs w}}\sup_{\tilde{P}\in\mathcal{P}(W,\bs S,\bs T)}\bbE_{\tilde{P}^{\otimes n}\times P_{\xi}}\left\Vert\hat{\bs w}-\bs w^*\right\Vert_{\bs T}^2\geq\sup_{\substack{\bs F\succeq\bs O\\ \left\Vert\bs F\right\Vert_*\leq 1/\pi^2}}\left\langle\bs T',\left(\bs F^{-1}+\frac{n\bs S'}{\sigma^2}\right)^{-1}\right\rangle.
    \end{equation}
    % Denote the $i$-th diagonal entry of $\bs T'$ by $t'_{ii}$ , we first show that
    % \begin{equation}
    %     \sup_{\substack{\bs F\succeq\bs O\\ \left\Vert\bs F\right\Vert_*\leq 1/\pi^2}}\left\langle\bs T',\left(\bs F^{-1}+\frac{n\bs S'}{\sigma^2}\right)^{-1}\right\rangle\geq\sup_{\substack{\bs F\succeq\bs O\\ \left\Vert\bs F\right\Vert_*\leq 1/\pi^2}}\left\langle\diag\left\{t'_{ii}\right\}_{i=1}^d,\left(\bs F^{-1}+\frac{n\bs S'}{\sigma^2}\right)^{-1}\right\rangle.
    % \end{equation}
    Consider $2^d$ reflection matrices $\bs U_j=\diag\left\{\pm 1,\pm 1,\ldots,\pm 1\right\}$. Note that $\bs U_j\bs S'\bs U_j=\bs S'$ since $\bs S'$ is diagonal, we have
    \begin{equation}
    \begin{aligned}
        \sup_{\substack{\bs F\succeq\bs O\\ \left\Vert\bs F\right\Vert_*\leq 1/\pi^2}}\left\langle\bs T',\left(\bs F^{-1}+\frac{n\bs S'}{\sigma^2}\right)^{-1}\right\rangle=&
        \sup_{\substack{\bs F\succeq\bs O\\ \left\Vert\bs F\right\Vert_*\leq 1/\pi^2}}\left\langle\bs U_j\bs T'\bs U_j,\left(\left(\bs U_i\bs F\bs U_i\right)^{-1}+\frac{n\bs S'}{\sigma^2}\right)^{-1}\right\rangle \\
        =&\sup_{\substack{\bs F\succeq\bs O\\ \left\Vert\bs F\right\Vert_*\leq 1/\pi^2}}\left\langle\bs U_j\bs T'\bs U_j,\left(\bs F^{-1}+\frac{n\bs S'}{\sigma^2}\right)^{-1}\right\rangle.
    \end{aligned}
    \end{equation}
    Therefore,
    \begin{equation}
    \begin{aligned}        
        \sup_{\substack{\bs F\succeq\bs O\\ \left\Vert\bs F\right\Vert_*\leq 1/\pi^2}}\left\langle\bs T',\left(\bs F^{-1}+\frac{n\bs S'}{\sigma^2}\right)^{-1}\right\rangle=&\frac{1}{2^d}\sum_{j=1}^{2^d}\sup_{\substack{\bs F\succeq\bs O\\ \left\Vert\bs F\right\Vert_*\leq 1/\pi^2}}\left\langle\bs U_j\bs T'\bs U_j,\left(\bs F^{-1}+\frac{n\bs S'}{\sigma^2}\right)^{-1}\right\rangle \\
        \geq&\sup_{\substack{\bs F\succeq\bs O\\ \left\Vert\bs F\right\Vert_*\leq 1/\pi^2}}\frac{1}{2^d}\sum_{j=1}^{2^d}\left\langle\bs U_j\bs T'\bs U_j,\left(\bs F^{-1}+\frac{n\bs S'}{\sigma^2}\right)^{-1}\right\rangle \\
        \stackrel{a}{=}&\sup_{\substack{\bs F\succeq\bs O\\ \left\Vert\bs F\right\Vert_*\leq 1/\pi^2}}\left\langle\diag\left\{t'_{ii}\right\}_{i=1}^d,\left(\bs F^{-1}+\frac{n\bs S'}{\sigma^2}\right)^{-1}\right\rangle,
    \end{aligned}
    \end{equation}
    where $\stackrel{a}{=}$ holds because $\frac{1}{2^d}\sum_{j=1}^{2^d}\bs U_j\bs T'\bs U_j=\diag\left\{t'_{ii}\right\}_{i=1}^d$, and $t'_{ii}$ denotes the $i$-th diagonal entry of $\bs T'$. Since $\bs T'=\bs M^{-1/2}\bs T\bs M^{-1/2}$, we can apply Theorem $\ref{thm:lower-upper-match}$ and Lemma~\ref{lemma:lower-bound-diagonal}:
    \begin{equation}
    \begin{aligned}        
        &\sup_{\substack{\bs F\succeq\bs O\\ \left\Vert\bs F\right\Vert_*\leq 1/\pi^2}}\left\langle\diag\left\{t'_{ii}\right\}_{i=1}^d,\left(\bs F^{-1}+\frac{n\bs S'}{\sigma^2}\right)^{-1}\right\rangle \\
        =&\min_{\bs A\in\bbR^{d\times d}}\frac{1}{\pi^2}\left\Vert(\bs I-\bs A)^\top\bs T'(\bs I-\bs A)\right\Vert+\frac{\sigma^2}{n}\left\langle\bs T',\bs A\left(\bs S'\right)^{-1}\bs A^\top\right\rangle \\
        \geq&\frac{1}{4}\min_{\tau\geq 0}\frac{\tau^2}{\pi^2}+\sum_{i\in\bbK_\tau}\frac{\sigma^2r_i}{n\lambda_i}\eqsim\min_{\tau\geq 0}\begin{cases}
            \tau^2+\sigma^2/n,&\kappa<-1; \\
            \tau^2-(\sigma^2/n)\ln\tau,&\kappa=-1;
        \end{cases}
    \end{aligned}
    \end{equation}
    where $\bbK_\tau=\left\{k:t_{kk}/m_k\geq\tau^2\right\}$. Setting $\tau=0$ for $\kappa<-1$ or $\tau=\sqrt{\sigma^2/(2n)}$ for $\kappa=-1$ yields the desired result.
\end{proof}

The following lemma provides a lower bound for the case $\kappa>-1$.
\begin{lemma}\label{lemma:poly-decay-lower-bound-hard}
    Let $k^*\in\bbN$ and $\Delta k=\left\lfloor\frac{n}{\pi^2\sigma^2}\left\Vert\bs S'_{k^*:\infty}\right\Vert\right\rfloor$. If $\Delta k\geq 1$, then we have
    \begin{equation}
        \inf_{\hat{\bs w}}\sup_{\tilde{P}\in\mathcal{P}(W,\bs S,\bs T)}\bbE_{\tilde{P}^{\otimes n}\times P_{\xi}}\left\Vert\hat{\bs w}-\bs w^*\right\Vert_{\bs T}^2\geq\frac{1}{4\pi^2\Delta k}\sum_{i=1}^{\Delta k}\sigma_i\left(\bs T'_{k^*:\infty}\right),
    \end{equation}
    where $\bs T'=\bs M^{-1/2}\bs T\bs M^{-1/2}$.
\end{lemma}
\begin{proof}    
    Let $\epsilon>0$ be sufficiently small. Recall that $\bs T'_{k^*:\infty}\in\bbR^{d\times d}$ has the following structure:
    \begin{equation}
        \bs T_{k^*:\infty}=\begin{bmatrix}
            \bs O & \bs O \\
            \bs O & \bs T_{22}
        \end{bmatrix}, \quad\bs T'_{22}=\sum_{i=1}^{d-k^*}\sigma_i\left(\bs T'_{k^*:\infty}\right)\bs u_i\bs u_i^\top,
    \end{equation}
    where $\bs T_{22}\in\bbR^{(d-k^*)\times(d-k^*)}$, $\bs u_i\in\bbR^{d-k^*}$. We define $\bs F$ as follows:
    \begin{equation}
        \bs F=\frac{\sigma^2}{n\left\Vert\bs S'_{k^*:\infty}\right\Vert}\begin{bmatrix}
            \epsilon\bs I/\Delta k & \bs O \\
            \bs O & \bs C
        \end{bmatrix}, \quad\bs C=\sum_{i=1}^{\Delta k}(1-2\epsilon)\bs u_i\bs u_i^\top+\sum_{i=\Delta k+1}^{d}\frac{\epsilon}{2^{i-\Delta k}}\bs u_i\bs u_i^\top.
    \end{equation}
    Therefore,
    \begin{equation}
    \begin{aligned}
        \left\langle\bs T',\left(\bs F^{-1}+\frac{n\bs S'}{\sigma^2}\right)^{-1}\right\rangle\geq&\left\langle\bs T',\left(\bs F^{-1}+\frac{n
        \left(\bs S'_{0:k^*}+\left\Vert\bs S'_{k^*:\infty}\right\Vert\bs I_{k^*:\infty}\right)}{\sigma^2}\right)^{-1}\right\rangle \\        
        \geq&\left\langle\bs T_{22},\left(\frac{n\left\Vert\bs S'_{k^*:\infty}\right\Vert}{\sigma^2}\bs C^{-1}+\frac{n\left\Vert\bs S'_{k^*:\infty}\right\Vert\bs I}{\sigma^2}\right)^{-1}\right\rangle \\
        =&\frac{\sigma^2}{n\left\Vert\bs S'_{k^*:\infty}\right\Vert}\left[\sum_{i=1}^{\Delta k}\frac{\sigma_i\left(\bs T'_{k^*:\infty}\right)}{1/(1-2\epsilon)+1}+\sum_{i=\Delta k+1}^{d}\frac{\sigma_i\left(\bs T'_{k^*:\infty}\right)}{2^{i-\Delta k}/\epsilon+1}\right] \\
        \geq&\frac{\sigma^2}{n\left\Vert\bs S'_{k^*:\infty}\right\Vert}\sum_{i=1}^{\Delta k}\frac{\sigma_i\left(\bs T'_{k^*:\infty}\right)}{1/(1-2\epsilon)+1} \\
        \stackrel{a}{\geq}&\frac{1}{2\pi^2\Delta k}\sum_{i=1}^{\Delta k}\frac{\sigma_i\left(\bs T'_{k^*:\infty}\right)}{1/(1-2\epsilon)+1},
        % =&\left\langle\bs T'_{0:k^*},\left(\bs F_{0:k^*}^{-1}+\frac{n\bs S'_{0:k^*}}{\sigma^2}\right)^{-1}\right\rangle+\left\langle\bs T_{22},\left(\bs F_{22}^{-1}+\frac{\sigma^2\bs I}{n\lambda_{k^*}}\right)^{-1}\right\rangle \\
        % =&\sum_{i=1}^{k^*}\frac{t'_{ii}}{\pi^2\Delta k/\epsilon+n\lambda_i/\sigma^2}+\frac{\sigma^2}{n\lambda_{k^*}}\sum_{i=1}^{\Delta k}\frac{\mu_i}{(1-2\epsilon)/\pi^2+1}
    \end{aligned}
    \end{equation}
    where $\stackrel{a}{\geq}$ uses $\frac{n}{\pi^2\sigma^2}\left\Vert\bs S'_{k^*:\infty}\right\Vert\leq\Delta k+1\leq 2\Delta k$. We apply Theorem~\ref{thm:lower-bound}, and let $\epsilon\to 0^+$, we have
    \begin{equation}
    \begin{aligned}        
        \inf_{\hat{\bs w}}\sup_{\tilde{P}\in\mathcal{P}(W,\bs S,\bs T)}\bbE_{\tilde{P}^{\otimes n}\times P_{\xi}}\bbE_{P_{\bs w^*}}\left\Vert\hat{\bs w}-\bs w^*\right\Vert_{\bs T}^2\geq&\lim_{\epsilon\to 0^+}\frac{1}{2\pi^2\Delta k}\sum_{i=1}^{\Delta k}\frac{\sigma_i\left(\bs T'_{k^*:\infty}\right)}{1/(1-2\epsilon)+1} \\
        =&\frac{1}{4\pi^2\Delta k}\sum_{i=1}^{\Delta k}\sigma_i\left(\bs T'_{k^*:\infty}\right).
    \end{aligned}
    \end{equation}
    This completes the proof.
\end{proof}

\subsection{Proof of Examples in Section~\ref{sec:asgd-optimality}}
\begin{proof}[Proof of Example~\ref{example:gaussian}]
    Let $\bs S$ and $\bs T$ denote the source and target covariance matrix. Since
    \begin{equation}
        D_\mathrm{KL}(Q_{\bs x }\Vert P_{\bs x })=\frac{1}{2}\left(\tr(\bs S^{-1/2}\bs T\bs S^{-1/2})-d-\ln\det(\bs S^{-1/2}\bs T\bs S^{-1/2})\right)
    \end{equation}
    Let $\rho_i$ denote the eigenvalues of $\bs S^{-1/2}\bs T\bs S^{-1/2}$, 
    we have
    \begin{equation}
        D_\mathrm{KL}(Q_{\bs x }\Vert P_{\bs x })=\frac{1}{2}\sum_{i=1}^d\rho_i-1-\ln\rho_i<\epsilon.
    \end{equation}
    Since $x-1-\ln x\geq 0$ for any $x>0$, we have $\rho_i-1-\ln\rho_i<\epsilon$ for all $i\in[d]$. By solving the inequality, we obtain that $\rho_i$ are bounded by a constant depending on $\epsilon$
\end{proof}
\begin{proof}[Proof of Example~\ref{example:emergence}]
    For the lower bound, we adopt Lemma~\ref{lemma:lower-bound-diagonal} by replacing $\bs R$ by $\bs T$, then we have
    \begin{equation}
        \inf_{\hat{\bs w}}\sup_{\tilde{P} \in \mathcal{P}(W,\bs S,\bs T)}\bbE_{\tilde{P}^{\otimes n}\times P_{\xi }}\left\Vert\hat{\bs w}-\bs w^*\right\Vert_{\bs T}^2\geq\min_{\tau\geq 0}\frac{\tau^2}{\pi^2}+\sum_{i\in\bbK_\tau}\frac{\sigma^2t_{ii}}{n\lambda_i},
    \end{equation}
    where $\bbK_\tau=\left\{k:r_k/m_k>\tau^2\right\}$. Since $r<1/a$, we have
    \begin{equation}
        \frac{\tau^2}{\pi^2}+\sum_{i\in\bbK_\tau}\frac{\sigma^2t_{ii}}{n\lambda_i}\eqsim\begin{cases}
            \tau^2,&\tau\geq d_0^{-(1+r)a/2}; \\
            \tau^2+\left(\sigma^2/n\right)\tau^{\frac{2ra-2}{(1+r)a}},&\tau<d_0^{-(1+r)a/2}.
        \end{cases}
    \end{equation}
    For $n\leq\sigma^2 d_0^{-a-1}$, we set $\tau=d_0^{-(1+r)a/2}$, and for $n>\sigma^2 d_0^{-a-1}$, we set $\tau=\left(\sigma^2/n\right)^\frac{(1+r)a}{2(1+a)}$, so
    \begin{equation}
        \min_{\tau\geq 0}\frac{\tau^2}{\pi^2}+\sum_{i\in\bbK_\tau}\frac{\sigma^2t_{ii}}{n\lambda_i}\eqsim\begin{cases}
            d_0^{-(1+r)a},&n\leq\sigma^2 d_0^{a+1}; \\
            \left(\sigma^2/n\right)^{\frac{(1+r)a}{1+a}},&n>\sigma^2 d_0^{a+1}.
        \end{cases}
    \end{equation}

    For the upper bound, we set
    \begin{equation}
        \delta=\gamma=\Theta\left(n^{-\frac{1}{a+1}}/\ln n\right),
    \end{equation}
    which is equivalent to vanilla SGD with step size $\gamma$. Following the proof of Lemma~\ref{lemma:poly-decay-sgd-upper-bound}, we have $k^*=\Theta\left(n^\frac{1}{a+1}\ln^{-\frac{3}{a}}n\right)$. If $n\lesssim d_0^{a+1}$, we have $k^*\leq d_0$. Thus, for the bias term, we have
    \begin{equation}
    \begin{aligned}        
        \mathrm{Variance}\lesssim&d_0^{-(1+r)a}\left[\sum_{i=1}^{k^*}{\frac{1}{Ki^{-a}}}+K\left(\gamma+\delta\right)^2\sum_{i=k^*+1}^{d_0}i^{-a}\right] \\
        &+K\left(\gamma+\delta\right)^2\sum_{i=d_0+1}^{d}i^{-2a-ra} \\
        \lesssim&d_0^{-(1+r)a}\left(1+(\ln n)^{-\frac{3}{a}}\right)+d_0^{-(2+r)a+1}n^\frac{a-1}{a+1}(\ln n)^{-3}\lesssim d_0^{-(1+r)a},
    \end{aligned}
    \end{equation}
    and for the variance term, we have
    \begin{equation}
        \mathrm{Bias}\lesssim \left(1+\sigma^2\right)d_0^{-(1+r)a}.
    \end{equation}
    Combine the above two results, we obtain
    \begin{equation}
        \bbE_{\tilde{P}^{\otimes n}}\left\Vert\bs w_n^\sgd-\bs w^*\right\Vert_{\bs T}^2\lesssim
            (1+\sigma^2)n^{-\frac{(1+r)a}{a+1}}(\ln n)^{\frac{3(1+r)a-3}{a}}.\label{eq:emergence-n-small}
    \end{equation}
    If $n\gtrsim d_0^{a+1}$, we have
    \begin{equation}
        \mathrm{Variance}\lesssim\sum_{i=1}^{k^*}{\frac{i^{-ra}}{K}}+K\left(\gamma+\delta\right)^2\sum_{i=k^*+1}^{d}i^{-2a-ra}\eqsim(1/n)^\frac{(1+r)a}{a+1}(\ln n)^{\frac{3(1+r)a-3}{a}},
    \end{equation}
    \begin{equation}
        \mathrm{Bias}\lesssim\left(\max\{k^*,d_0\}\right)^{-(1+r)a}+\frac{\left\Vert\bs S'\right\Vert\cdot\mathrm{Variance}}{\sigma^2}.
    \end{equation}
    Thus, we have
    \begin{equation}
        \bbE_{\tilde{P}^{\otimes n}}\left\Vert\bs w_n^\sgd-\bs w^*\right\Vert_{\bs T}^2\lesssim
            (1+\sigma^2)n^{-\frac{(1+r)a}{a+1}}(\ln n)^{\frac{3(1+r)a-3}{a}}.\label{eq:emergence-n-large}
    \end{equation}
    Combining~\eqref{eq:emergence-n-small} and~\eqref{eq:emergence-n-large} yields the desired result.
\end{proof}

\end{document}